\documentclass[12pt,beng]{uom_eee_dissertation_casson} % course can be msc or beng
\usepackage[utf8]{inputenc} % allow utf-8 input
\usepackage[T1]{fontenc}    % use 8-bit T1 fonts
\usepackage[lf]{ebgaramond} 
\usepackage{amsmath}
\usepackage{amsthm}
\usepackage{siunitx}
\usepackage{graphicx,psfrag,color} % for postscript graphics files
  \graphicspath{ {./images/} }
\usepackage{amsmath}               % assumes amsmath package installed
  \allowdisplaybreaks[1]           % allow eqnarrays to break across pages
\usepackage{amssymb}               % assumes amsmath package installed 
\usepackage{url}                   % format hyperlinks correctly
\usepackage{rotating}              % allow portrait figures and tables
\usepackage{multirow}              % allows merging of rows in tables
\usepackage{lscape}                % allows pages to be typeset in landscape mode
\usepackage{tabularx}              % allows fixed width tables
\usepackage{verbatim}              % enhanced version of built-in verbatim environment
\usepackage{footnote}              % allows more control over footnote environments
\usepackage{float}                 % allows H option on floats to force here placement
\usepackage{booktabs}              % improve table line spacing
\usepackage{lipsum}                % for adding dummy text here
\usepackage[base]{babel}           % for proper hypthenation in lipsum sections
\usepackage{subcaption}            % for multiple sub-figures in a single float
\usepackage{enumitem,amssymb}
\newlist{todolist}{itemize}{9}
\setlist[todolist]{label=$\square$}

\usepackage{pifont}
\newcommand\norm[1]{\left\lVert#1\right\rVert}
\DeclareMathOperator{\E}{\mathbb{E}}
\DeclareMathOperator{\x}{\mathbf{x}}
\DeclareMathOperator{\z}{\mathbf{z}}
\DeclareMathOperator{\I}{\mathbf{I}}

\newtheorem{theorem}{Theorem}[section]
\newtheorem{corollary}{Corollary}[theorem]
\newtheorem{lemma}[theorem]{Lemma}

\usepackage[style=ieee,backend=biber,backref=true,hyperref=auto]{biblatex}
\DefineBibliographyStrings{english}{backrefpage = {cited on p\adddot},  backrefpages = {cited on pp\adddot}}
\begin{document}
\makeatletter
\title{\xmp@Title}
\studentid{\xmp@Author}
\studentname{Alexandru Buburuzan}
\supervisor{Professor Tim Cootes}
\makeatother

% Set the below yourself
\course{Artificial Intelligence} 
\faculty{Science and Engineering}                  % "Faculty of" is added automatically
\school{Department of Computer Science}
\submitdate{2025}                                  % regulations ask only for the year, not month
\wordcount{14038}		                           % use \wordcount{} to set the count, \thewordcount to print in the text
\maketitle

%%%%%%%%%%%%%%%%%% LISTS OF CONTENT %%%%%%%%%%%%%%%%%%
\uomtoc
\uomlof
\uomlot

\begin{uomterms}
\begin{description}[leftmargin=!]
  \item[MLP] Multi-Layer Perceptron
  \item[PbE] Paint-by-Example
  \item[SAM] Segment Anything Model
  \item[CLIP] Contrastive Language–Image Pretraining
  \item[BEV] Bird’s Eye View
  \item[MObI] Multimodal Object Inpainting
  \item[VAE] Variational Autoencoder
  \item[DDPM] Denoising Diffusion Probabilistic Model
  \item[DDIM] Denoising Diffusion Implicit Model
  \item[PLMS] Pseudo Linear Multistep Scheduler
  \item[LPIPS] Learned Perceptual Image Patch Similarity
  \item[D-LPIPS] Depth Learned Perceptual Image Patch Similarity
  \item[I-LPIPS] Intensity Learned Perceptual Image Patch Similarity
  \item[FID] Fréchet Inception Distance
  \item[AOE] Average Orientation Error
  \item[ASE] Average Scale Error
  \item[ATE] Average Translation Error
\end{description}
\end{uomterms}

%%%%%%%%%%%%%%%%%% DECLARATIONS %%%%%%%%%%%%%%%%%%
\begin{uomoriginality}
  \uomoriginalitydeclaration 
  % If the standard originality decalaration is sufficient, saying no portion of the work has been submitted in support of an application for another degree or qualification, the above command will automatically add the required text and nothing else is needed in this section.
  % If the standard statment isn't sufficient, then comment out the \uomoriginalitydeclaration command and type in your own text here explaining the authorship of any re-used portions. 
\end{uomoriginality}
\uomcopyrightstatement

%%%%%%%%%%%%%%%%%% ACKNOWLEDGEMENTS %%%%%%%%%%%%%%%%%%
\begin{uomacknowledgements}
I am thankful to my supervisor, Tim Cootes, my family and mentors for their unwavering support.
\end{uomacknowledgements}

%%%%%%%%%%%%%%%%%% ABSTRACT %%%%%%%%%%%%%%%%%%
\begin{abstract}
Safety-critical applications, such as autonomous driving and medical image analysis, require extensive multimodal data for rigorous testing. Synthetic data methods are gaining prominence due to the cost and complexity of gathering real-world data, but they demand a high degree of realism and controllability to be useful. This work introduces two novel methods for synthetic data generation in autonomous driving and medical image analysis, namely \textbf{MObI} and \textbf{AnydoorMed}, respectively.

\textbf{MObI} is a first-of-its-kind framework for \textbf{M}ultimodal \textbf{Ob}ject \textbf{I}npainting that leverages a diffusion model to produce realistic and controllable object inpaintings across perceptual modalities, demonstrated simultaneously for camera and lidar. Given a single reference RGB image, MObI enables seamless object insertion into existing multimodal scenes at a specified 3D location, guided by a bounding box, while maintaining semantic consistency and multimodal coherence. Unlike traditional inpainting methods that rely solely on edit masks, this approach uses 3D bounding box conditioning to ensure accurate spatial positioning and realistic scaling. Consequently, MObI provides significant advantages for flexibly inserting novel objects into multimodal scenes, offering a powerful tool for testing perception models under diverse conditions.

\textbf{AnydoorMed} extends this paradigm to the medical imaging domain, focusing on reference-guided inpainting for mammography scans. It leverages a diffusion-based model to inpaint anomalies with impressive detail preservation, maintaining the reference anomaly's structural integrity while semantically blending it with the surrounding tissue. AnydoorMed enables controlled and realistic synthesis of anomalies, offering a promising solution for augmenting datasets in the safety-critical medical domain.

Together, these methods demonstrate that foundation models for reference-guided inpainting in natural images can be readily adapted to diverse perceptual modalities, paving the way for the next generation of systems capable of constructing highly realistic, controllable and multimodal counterfactuals.

Code and model weights are being made available at: \url{https://github.com/alexbuburuzan/MObI} and \url{https://github.com/alexbuburuzan/AnydoorMed}.

\end{abstract}%
\clearpage

%%%%%%%%%%%%%%%%%% Content %%%%%%%%%%%%%%%%%%

\chapter{Background}

\begin{center}
    \textit{``Counterfactual reasoning is a hallmark of human thought, enabling the capacity to shift from perceiving the immediate environment to an alternative, imagined perspective.''}\\
    \hfill Van Hoeck et al., 2015~\cite{van2015cognitive}
\end{center}

If counterfactual reasoning lies at the core of human intelligence, then building systems capable of constructing and leveraging counterfactuals may represent the next defining step in the evolution of artificial intelligence. This work constitutes a small step towards that vision.

\newpage
\section{Motivation}

Decades of research in cognitive neuroscience have established counterfactual reasoning as a fundamental component of human perception~\cite{van2015cognitive, byrne2002mental}, defined as the capacity to consider alternative scenarios of ``what might have happened''. This process is achieved by constructing and manipulating mental representations of hypothetical realities, enabling learning and adaptation through internal simulation. Pioneering work by Costello and McCarthy~\cite{costello1999useful} underscored the significance of counterfactuals in enabling agents to reason. However, their early approach relied on formal methods, which encountered difficulties when faced with the high-dimensional, ambiguous, and dynamically changing nature of real-world perceptual inputs.

Recent advances in generative modelling offer a promising alternative for building counterfactual examples directly from perceptual data. In particular, latent diffusion models enable the controlled editing of complex, high-dimensional, and multimodal inputs by operating within a learnt latent space that captures their underlying semantic structure. Through this approach, realistic counterfactuals can be synthesised, offering a means to systematically explore alternative possibilities without reliance on handcrafted assets.

The ability to generate plausible counterfactuals from perceptual data holds significant promise for developing intelligent systems. By simulating alternative outcomes, perception models could be stress-tested under rare or hypothetical scenarios, improving their robustness and generalisation capabilities. Moreover, decision-making systems could benefit from counterfactual reasoning by evaluating the consequences of actions that might have been taken, thus enabling safer and more informed choices in safety-critical domains such as autonomous driving and medical diagnosis. Embedding counterfactual generation capabilities into artificial agents may represent a first step towards more adaptive, interpretable, and human-aligned intelligence.

\section{Introduction}

\begin{figure}
    \centering
    \includegraphics[width=0.99\textwidth]{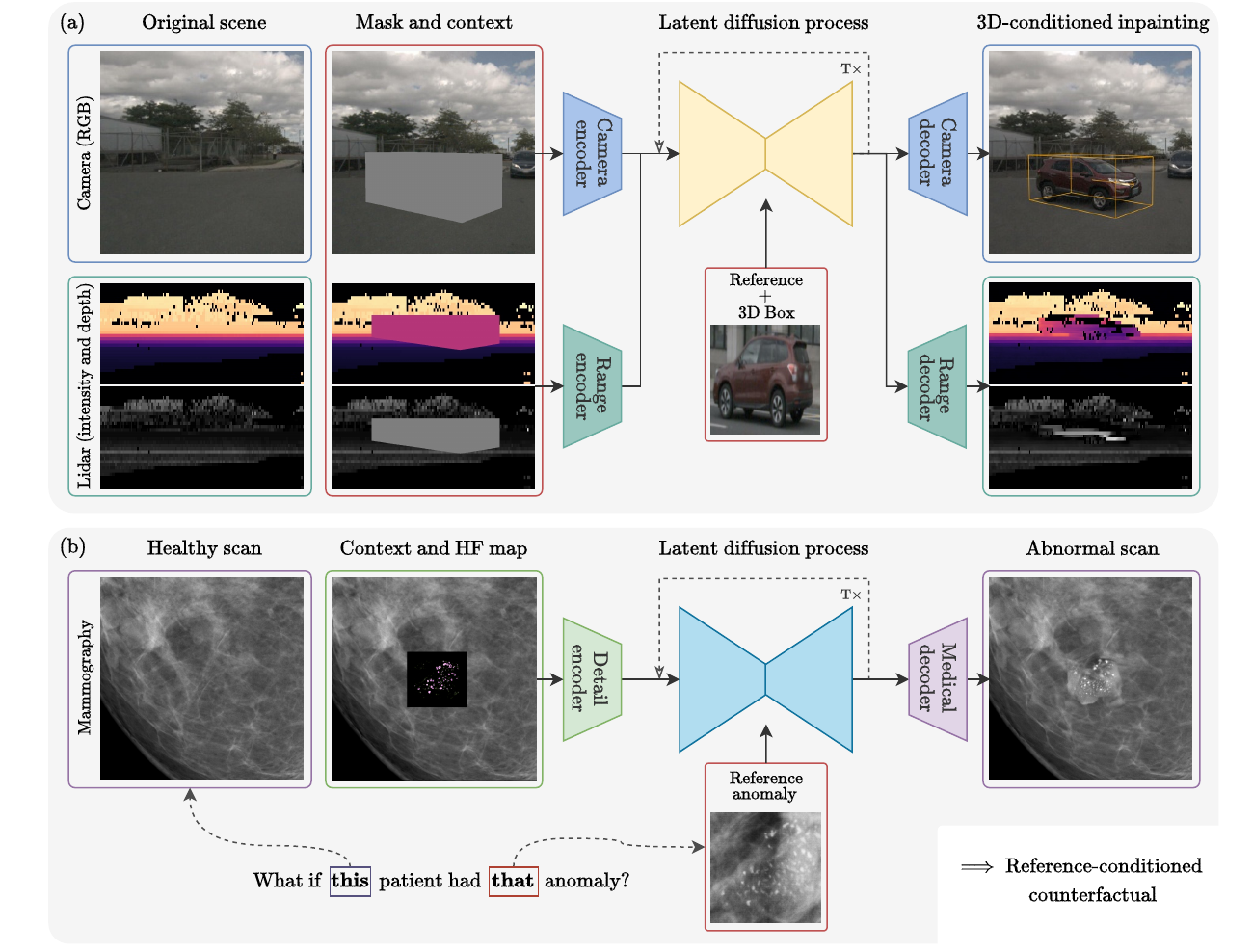}
    \caption[Teaser figure of MObI and AnydoorMed methods for editing of three perceptual modalities.]{(a) \textbf{MObI} enables the generation of multiple novel views from a single reference image while maintaining semantic consistency and multimodal coherence across camera and lidar modalities. The inserted object respects the geometric constraints imposed by an oriented 3D bounding box, with inpainting performed in a modality-agnostic latent space. (b) \textbf{AnydoorMed} inpaints an anomaly at a specific location within mammography scans with high fidelity, preserving fine details such as microcalcifications. This enables the construction of reference-guided counterfactuals, answering questions such as ``How would the scan look like if this patient had that anomaly?''}
    \label{fig:teaser}
\end{figure}

This report introduces two novel methods for reference-guided counterfactual generation across different domains: \textbf{MObI}, for camera-lidar object inpainting in autonomous driving scenes, and \textbf{AnydoorMed}, for anomaly inpainting in mammography scans. Both methods leverage the power of latent diffusion models to perform controlled, high-fidelity insertions while preserving semantic consistency, as illustrated in \autoref{fig:teaser}. MObI uniquely enables realistic, 3D-conditioned object insertion across camera and lidar modalities, while AnydoorMed can synthesise perceptually plausible anomalies at specific locations within a mammography scan.

This work's strengths lie in its ability to adapt foundation models for reference-guided inpainting to diverse perceptual modalities using a simple data-efficient adaptation mechanism. This achieves fine-grained control, multimodal coherence, and semantic consistency without reliance on handcrafted assets, as demonstrated by state-of-the-art results according to realism metrics, compared to their respective baselines. 

The remainder of this report is structured as follows: \autoref{section:theory} provides the necessary theoretical background to understand the fundamentals of latent diffusion models. \autoref{chaper:mobi} and \autoref{chapter:anydoor} present MObI and AnydoorMed, individually, detailing their architecture, training procedures, and experiments. Finally, \autoref{chapter:discussion} draws the key findings, presents some limitations and outlines potential future directions.

\section{Theory}
\label{section:theory}

This section establishes the theoretical foundations of generative modelling, beginning with the Variational Autoencoder (VAE) and subsequently presenting the principles underlying the recent state-of-the-art diffusion models for image generation.

\subsection{Variational Autoencoder (VAE)}
Firstly, directed graphical models, also known as Bayesian networks, are a way to represent joint probability distributions using directed acyclic graphs (DAGs). Each node in the graph represents a random variable, and directed edges encode conditional dependencies: a directed edge from $\z$ to $\x$ indicates that $\x$ is conditionally dependent on $\z$.

Mathematically, the joint distribution factorises according to the graph structure:
\begin{align}
    p(\x, \z) = p(\x|\z) p(\z).
\end{align}

VAEs~\cite{kingma2013auto} are a class of generative models that learn a probabilistic mapping from a latent space to the observed data space. They combine principles from variational inference and deep learning to generate new data samples that resemble those from the training distribution. The generation of an observation $\x$ given a latent variable $\z$ is modelled through a directed graphical model as presented in \autoref{fig:vae_figure}.

In this generative framework, the following components are defined:
\begin{itemize}
    \item \textbf{Prior} \( p_{\theta}(\z) \): A distribution over the latent variable, typically chosen as multivariate normal.
    \item \textbf{Likelihood} \( p_{\theta}(\x|\z) \): The conditional distribution describing how the data is generated from the latent variable.
    \item \textbf{Posterior} \( p_{\theta}(\z|\x) \): The distribution over latent variables given observed data.
\end{itemize}

\begin{figure}[h]
    \centering
    \includegraphics[scale=0.4]{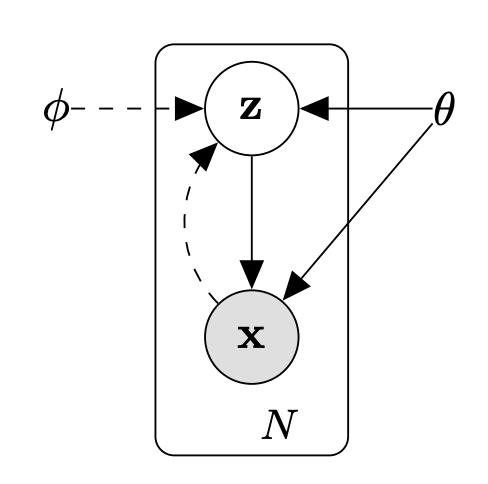}
    \caption[Directed graphical model of a Variational Autoencoder (VAE).]{Directed graphical model of a VAE~\cite{kingma2013auto} comprising the observable discrete random variable $\x$ and the latent continuous random variable $\z$. Solid lines represent the generative process $p_{\theta}(\z)p_{\theta}(\x|\z)$, while dashed lines represent the variational approximation $q_{\phi}(\z|\x)$ of the intractable true posterior $p_{\theta}(\z|\x)$.}
    \label{fig:vae_figure}
\end{figure}

\paragraph{Notation}
In the context of VAEs, the following notational conventions are adopted:
\begin{align}
    \begin{split}
        p(\cdot|\cdot,\theta) &\triangleq p_{\theta}(\cdot|\cdot) \\
        q(\cdot|\cdot,\phi) &\triangleq q_{\phi}(\cdot|\cdot)
    \end{split}
    \quad \quad
    \begin{split}
        D_{\text{KL}}(q\,\|\,p) &= \int q(x) \log \frac{q(x)}{p(x)} \mathrm{d}x \geq 0
    \end{split}
\end{align}
Here, $\theta$ and $\phi$ denote function parameters. $D_{\text{KL}}(q\,\|\,p)$ denotes the Kullback--Leibler (KL) divergence, which quantifies the difference between two probability distributions and is always non-negative.

\paragraph{The intractable posterior}
In the generative model, the marginal likelihood of an observation $\x$ is given by:
\begin{align}
    p_{\theta}(\x) = \int p_{\theta}(\x|\z) p_{\theta}(\z) \, \mathrm{d}\z.
\end{align}
However, this integral is generally intractable because it requires integrating across all possible configurations of the latent variable $\z$, which is high-dimensional and continuous. Thus, without a closed-form solution, evaluating the integral would require an infeasible amount of computation~\cite{blei2017variational}.

As a consequence, the true posterior distribution
\begin{align}
    p_{\theta}(\z|\x) = \frac{p_{\theta}(\x|\z)p_{\theta}(\z)}{p_{\theta}(\x)}
\end{align}
is also intractable. This motivates the need for an approximate inference method: \cite{kingma2013auto} introduces a variational distribution $q_{\phi}(\z|\x)$ to approximate the true posterior $p_{\theta}(\z|\x)$.

\paragraph{Evidence Lower Bound (ELBO)}
Let $\mathcal{X} = \{ \x^{(1)}, \x^{(2)}, \ldots, \x^{(N)} \}$ denote a set of $N$ independent and identically distributed (IID) observations drawn from the true data distribution. The data log-likelihood is given by:
\begin{align}
    \log p_{\theta}(\x^{(1)}, \x^{(2)}, \ldots, \x^{(N)}) = \sum_{i=1}^{N} \log p_{\theta}(\x^{(i)}).
\end{align}

For a single, discrete observation $\x^{(i)}$, the following decomposition can be derived:
\begin{align}
    \log p_{\theta}(\x^{(i)}) &= D_{\text{KL}}(q_{\phi}(\z|\x^{(i)})\,\|\,p_{\theta}(\z|\x^{(i)})) + \mathcal{L}(\theta,\phi;\x^{(i)})
\end{align}

\begin{proof}
    \begin{align*}
        D_{\text{KL}}(q_{\phi}(\z|\x^{(i)})~||~p_{\theta}(\z|\x^{(i)})) &= \int q_{\phi}(\z|\x^{(i)}) \log \frac{q_{\phi}(\z|\x^{(i)})}{p_{\theta}(\z|\x^{(i)})} \, \mathrm{d}\z \\
        &= \int q_{\phi}(\z|\x^{(i)}) \log \frac{q_{\phi}(\z|\x^{(i)})p_{\theta}(\x^{(i)})}{p_{\theta}(\x^{(i)},\z)} \, \mathrm{d}\z \\
        &= \int q_{\phi}(\z|\x^{(i)}) \log \frac{q_{\phi}(\z|\x^{(i)})}{p_{\theta}(\x^{(i)},\z)} \, \mathrm{d}\z~+~p_{\theta}(\x^{(i)}) \int q_{\phi}(\z|\x^{(i)}) \mathrm{d}\z \\
        &= D_{\text{KL}}(q_{\phi}(\z|\x^{(i)})~||~p_{\theta}(\x^{(i)},\z)) + \log p_{\theta}(\x^{(i)}) \quad \text{since } \int q_{\phi}(\z|\x^{(i)}) \mathrm{d}\z = 1 \\
        \implies \log p_{\theta}(\x^{(i)}) &= D_{\text{KL}}(q_{\phi}(\z|\x^{(i)})~||~p_{\theta}(\z|\x^{(i)})) - D_{\text{KL}}(q_{\phi}(\z|\x^{(i)})\,\|\,p_{\theta}(\x^{(i)},\z)) \\
        \implies \mathcal{L}(\theta,\phi;\x^{(i)}) &= -D_{\text{KL}}(q_{\phi}(\z|\x^{(i)})~||~p_{\theta}(\x^{(i)},\z)).
    \end{align*}
\end{proof}

\( \mathcal{L}(\theta,\phi;\x^{(i)}) \) is a lower bound to the evidence \(\log p_{\theta}(\x^{(i)})\):
\begin{align*}
    \log p_{\theta}(\x^{(i)}) &= D_{\text{KL}}(q_{\phi}(\z|\x^{(i)})\,\|\,p_{\theta}(\z|\x^{(i)})) + \mathcal{L}(\theta,\phi;\x^{(i)}) \\
    \implies \log p_{\theta}(\x^{(i)}) &\geq \mathcal{L}(\theta,\phi;\x^{(i)}) \quad \text{since } D_{\text{KL}}(\cdot||\cdot)\geq0
\end{align*}
which can be re-written as:
\begin{align}
    \mathcal{L}(\theta,\phi;\x^{(i)}) &= \mathbb{E}_{\z \sim q_{\phi}(\cdot|\x^{(i)})}[\log p_{\theta}(\x^{(i)}|\z)] - D_{\text{KL}}(q_{\phi}(\z|\x^{(i)})\,\|\,p_{\theta}(\z)).
\end{align}

\begin{proof}
    \begin{align*}
        \mathcal{L}(\theta,\phi;\x^{(i)}) &= -D_{\text{KL}}(q_{\phi}(\z|\x^{(i)})\,\|\,p_{\theta}(\z|\x^{(i)})) + \mathcal{L}(\theta,\phi;\x^{(i)}) \\
        &= -\int q_{\phi}(\z|\x^{(i)}) \log \frac{q_{\phi}(\z|\x^{(i)})}{p_{\theta}(\x^{(i)},\z)} \, \mathrm{d}\z \\
        &= -\int q_{\phi}(\z|\x^{(i)}) \log \frac{q_{\phi}(\z|\x^{(i)})}{p_{\theta}(\x^{(i)}|\z)p_{\theta}(\z)} \, \mathrm{d}\z \\
        &= \int q_{\phi}(\z|\x^{(i)}) \log {p_{\theta}(\x^{(i)},\z)} \, \mathrm{d}\z + \int q_{\phi}(\z|\x^{(i)}) \log {p_{\theta}(\x^{(i)},\z)} \, \mathrm{d}\z
    \end{align*}
\end{proof}

Since $\log p_{\theta}(\x^{(i)}) \geq \mathcal{L}(\theta,\phi;\x^{(i)})$, maximising the evidence lower bound (ELBO, right-hand term) will maximise the data log-likelihood.

The likelihood \( p_{\theta}(\x^{(i)}|\z) \) can be modelled by a \texttt{decoder}, a neural network with parameters \( \theta \), that reconstructs the input data \( \x^{(i)} \) given a latent representation \( \z \). Conversely, the approximate posterior \( q_{\phi}(\z|\x^{(i)}) \) can be modelled by an \texttt{encoder}, a separate neural network with parameters \( \phi \), which infers the distribution of the latent representation \( \z \) from the input \( \x^{(i)} \).

The evidence lower bound (ELBO) loss jointly optimises both the encoder and decoder. The latent space is parameterised as a multivariate normal distribution; specific details are omitted here for brevity.

\paragraph{Implementation details for image VAEs}
In the case of VAEs applied to image data, it is common to model the likelihood $p_{\theta}(\x^{(i)}|\z)$ as a factorised Gaussian distribution across the IID pixel intensities. Specifically, given a latent code $\z$, the decoder predicts the mean of the Gaussian for each pixel. Consequently, maximising the log-likelihood term $\mathbb{E}_{\z \sim q_{\phi}(\cdot|\x^{(i)})}[\log p_{\theta}(\x^{(i)}|\z)]$ corresponds to minimising a mean squared error (MSE) reconstruction loss between the predicted pixel intensities and the observed pixel values.

\paragraph{Summary}
Thus, by optimising the ELBO, VAEs are able to learn efficient representations of data in a continuous latent space, which can subsequently be sampled to generate novel instances resembling the training distribution. 

\subsection{Denoising Diffusion Probabilistic Model (DDPM)}
Denoising Diffusion Probabilistic Models (DDPMs)~\cite{ho2020denoising} are a class of generative models that learn to generate data by reversing a gradual noising process. During training, the model is optimised to predict and remove the noise added to data samples at various stages, conditioned explicitly on the timestep. This denoising is learnt so that the model can progressively reconstruct realistic data samples starting from pure Gaussian noise.

\begin{figure}[h]
    \centering
    \includegraphics[scale=0.4]{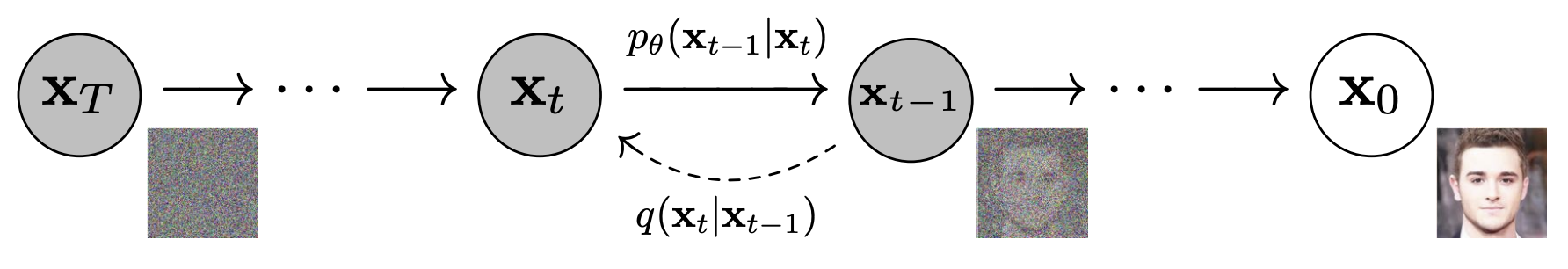}
    \caption[Directed graphical model of a Denoising Diffusion Probabilistic Model (DDPM).]{Directed graphical model of DDPM~\cite{ho2020denoising}. Dashed lines denote the forward diffusion process $q(\x_t|\x_{t-1})$. Solid lines denote the learnt denoising process $p_{\theta}(\x_{t-1}|\x_t)$.}
\end{figure}

\paragraph{Background}
The following notation and properties are used:
\begin{align*}
    &p(\x_i, \x_{i+1},...,\x_{j}) \triangleq p(\x_{i:j}) \text{ for } i<j \text{ and }  p(\x_i) \triangleq p(\x_{i:i})\\
    &\int \dots \int p(x_{i:j}) d\x_i \dots d\x_j \triangleq \int p(x_{i:j}) d\x_{i:j} \text{ for } i<j\\
    &p_{\theta}(\x_{t-1}|\x_{t:T}) = p_{\theta}(\x_{t-1}|\x_t) \implies p_{\theta}(\x_{0:T}) = p_{\theta}(\x_T)\prod_{t=1}^{T}p_{\theta}(\x_{t-1}|\x_t) && \text{Markov property}\\
    &q(\x_t|\x_{t-1:0}) = q(\x_t|\x_{t-1}) \implies q(\x_{0:T}) = q(\x_0)\prod_{t=1}^{T}q(\x_t|\x_{t-1}) && \text{Markov property}\\
    &f \text{ convex} \implies f(\E[x]) \leq \E[f(x)] && \text{Jensen's inequality}\\
    &f \text{ concave} \implies f(\E[x]) \geq \E[f(x)] && \text{Jensen's inequality}
\end{align*}

\paragraph{Diffusion}
The forward diffusion process gradually corrupts the data by adding Gaussian noise at each time step. This is controlled by a pre-defined noise schedule $\{\beta_1, \beta_2, ..., \beta_T\}$, where each $\beta_t \in (0,1)$ specifies the variance of the noise added at time $t$. In most practical implementations, the number of steps $T$ is set to a large value, typically $T = 1000$.

Formally, the forward process is given by:
\begin{align*}
    \text{Let } \x_t &= \sqrt{1 - \beta_t}\x_{t-1} + \sqrt{\beta_t}\epsilon_t, \quad \epsilon_t \sim \mathcal{N}(0, \I) \\
    q(\x_t|\x_{t-1}) &= \mathcal{N}(\x_t; \sqrt{1 - \beta_t}\x_{t-1}, \beta_t \I)
\end{align*}
where $\epsilon_t$ is drawn from a standard multivariate normal distribution with identity covariance matrix $\I$.

Typically, the noise schedule is chosen such that $\beta_1 < \beta_2 < ... < \beta_T$ and each $\beta_t$ is much smaller than 1, $\beta_t << 1$. This ensures that noise is added slowly and progressively over time.

\begin{lemma}
Each intermediate state $\x_t$ in the forward diffusion process is normally distributed, and its distribution can be expressed directly in terms of the original data $\x_0$:
\begin{align*}
    q(\x_t|\x_0) = \mathcal{N}(\x_t; \sqrt{\bar{\alpha}_t}\x_0, (1-\bar{\alpha}_t)\I)
\end{align*}
where $\bar{\alpha}_t = \prod_{i=1}^{t}\alpha_i$ and $\alpha_t = 1 - \beta_t$.
\end{lemma}

\begin{proof}
Recursive expansion of the diffusion process:
\begin{align*}
        \x_t &= \sqrt{1 - \beta_t}\x_{t - 1} + \sqrt{\beta_t}\epsilon_t \\
        &= \sqrt{\alpha_t}\x_{t - 1} + \sqrt{1 - \alpha_t}\epsilon_t \\
        &= \sqrt{\alpha_t}(\sqrt{\alpha_{t-1}}\x_{t - 2} + \sqrt{1 - \alpha_{t-1}}\epsilon_{t-1}) + \sqrt{1 - \alpha_t}\epsilon_t \\
        &= \sqrt{\alpha_t\alpha_{t-1}}\x_{t - 2} + \sqrt{\alpha_t(1 - \alpha_{t-1})}\epsilon_{t-1} + \sqrt{1 - \alpha_t}\epsilon_t \\
        &= \sqrt{\alpha_t\alpha_{t-1}}\x_{t - 2} + \sqrt{\alpha_t - \alpha_t \alpha_{t-1} + 1 - \alpha_t}\epsilon_{t-1:t} && \text{ variances add up and }\epsilon_{t-1:t} \sim \mathcal{N}(0, \I)\\
        &= \sqrt{\alpha_t\alpha_{t-1}}\x_{t - 2} + \sqrt{1 - \alpha_t \alpha_{t-1}}\epsilon_{t-1:t} \\
        &\dots \\
        &= \sqrt{\bar{\alpha}_t}\x_0 + \sqrt{1 - \bar{\alpha}_t}\epsilon_{0:t} && \epsilon_{0:t} \sim \mathcal{N}(0, \I) \\
        &\triangleq \sqrt{\bar{\alpha}_t}\x_0 + \sqrt{1 - \bar{\alpha}_t}\tilde{\epsilon}_t && \text{ where } \tilde{\epsilon}_t \sim \mathcal{N}(0, \I) \text{ is the full noise added to } \x_0\\
\end{align*}
Note, since $\epsilon_{t-1}$ and $\epsilon_t$ are independent standard Gaussians, the noise terms combine into another Gaussian noise term by adding the variance terms within the square root.
Thus, the probability distribution of $\x_t$ conditioned on $\x_0$ is a Gaussian with mean $\sqrt{\bar{\alpha}_t}\x_0$ and covariance $(1-\bar{\alpha}_t)\I$.
\end{proof}

\begin{corollary}
    The forward diffusion process converges to full Gaussian noise:
    \begin{align*}
        \lim_{T -> \infty}q(\x_T|\x_0) = \mathcal{N}(\x_T;\x_0, \I)
    \end{align*}
\end{corollary}

\begin{proof}
    \begin{align*}
        \lim_{T -> \infty}\bar{\alpha}_T &= \prod_{i=1}^{T}(1 - \beta_i) = 0 \text{ since } \beta_T \in (0, 1)\\
        \lim_{T -> \infty}\x_T &= \lim_{T -> \infty}\sqrt{\bar{\alpha}_T}\x_0 + \sqrt{1 - \bar{\alpha}_T}\tilde{\epsilon}_T = \tilde{\epsilon}_T \sim \mathcal{N}(0, \I)
    \end{align*}
\end{proof}

The probability distribution of the true denoising process $q(\x_{t-1}|\x_t)$ is intractable because the probability over the entire data space $q(\x_0)$ is unknown:

\begin{align*}
    q(\x_{t-1}) = \int q(\x_{t-1}|\x_0)q(\x_0)d\x_0  \text{ intractable }\implies q(\x_{t-1}|\x_t) = \frac{q(\x_t|\x_{t-1})q(\x_{t-1})}{q(\x_t)} \text{ intractable}
\end{align*}

An important observation is that if the original state $\x_0$ is known, it becomes easy to model the transition between $\x_t$ and $\x_{t-1}$, i.e. to infer and remove the added noise. The conditional distribution $q(\x_{t-1}|\x_t, \x_0)$ can be computed as follows:

\begin{align*}
    q(\x_{t-1}|\x_t, \x_0) &= \frac{q(\x_t|\x_{t-1}, \x_0)q(\x_{t-1}|\x_0)}{q(\x_t|\x_0)} \\
    &= \frac{q(\x_t|\x_{t-1})q(\x_{t-1}|\x_0)}{q(\x_t|\x_0)}&& \text{Markov property}\\
\end{align*}

Since $q(\x_t|\x_0)$ acts as a normalisation constant independent of $\x_{t-1}$, it can be omitted when considering the shape of the distribution. Thus, the following proportional relationship holds:
\begin{align*}
    q(\x_{t-1}|\x_t, \x_0) &\propto q(\x_t|\x_{t-1}) q(\x_{t-1}|\x_0). \\
    &= \mathcal{N}(\x_{t-1}|\sqrt{1 - \beta_t}\x_t, \beta_t\I) \mathcal{N}(\x_{t-1};\sqrt{\bar{\alpha}_{t-1}}\x_0, (1-\bar{\alpha}_{t-1})\I)\\
\end{align*}
This means that, conditioned on the original state $\x_0$, it is straightforward to model the denoising transition from $\x_t$ to $\x_{t-1}$. Given $q(\x_t|\x_{t-1})$ and $q(\x_{t-1}|\x_0)$ of the forward diffusion process as Gaussian distributions, their product is itself bell-shaped. As a result, the conditional denoising probability distribution $q(\x_{t-1}|\x_t, \x_0)$ can be explicitly computed as another Gaussian whose mean and variance can be derived analytically.

\begin{theorem}
\label{theorem:miu_theta}
    $q(\x_{t-1}|\x_t, \x_0)$ is a Gaussian distribution:
    \begin{gather*}
        q(\x_{t-1}|\x_t, \x_0) = \mathcal{N}(\x_{t-1};\tilde{\mu}_{t}(\x_t,\x_0),\tilde{\beta}_{t}{\I})\\
        \text{where } \tilde{\mu}_{t}(\x_{t},\x_{0}):=\frac{\sqrt{\bar{\alpha}_{t-1}}\beta_{t}}{1-\bar{\alpha}_{t}}\x_{0}+\frac{\sqrt{\alpha_t}(1 - \bar{\alpha}_{t-1})}{1-\bar{\alpha}_{t}}\x_{t} \text{ and } \tilde{\beta}_{t}:=\frac{1-\bar{\alpha}_{t-1}}{1-\bar{\alpha}_{t}}\beta_{t}
    \end{gather*}
\end{theorem}

The proof was excluded for brevity.

\paragraph{Learnt denoising process}

Recall $q(\x_{t-1}|\x_t) = \frac{q(\x_t|\x_{t-1})q(\x_{t-1})}{q(\x_t)}$. The posterior (i.e. true denoising process) $q(\x_{t-1}|\x_t)$ is approximately Gaussian, a key insight. This is because $q(\x_t|\x_{t-1})$ is a normal distribution with variance $\beta_t \ll 1$ and $q(\x_{t-1})$ does not vary a lot over the density of $q(\x_t|\x_{t-1})$. As such, the product of their probability density functions will be bell-shaped. Additionally, $q(\x_t)$ is a normalisation factor. See \autoref{fig:theory:bishop} for an illustrative example.

\begin{figure}[h]
    \centering
    \includegraphics[scale=0.3]{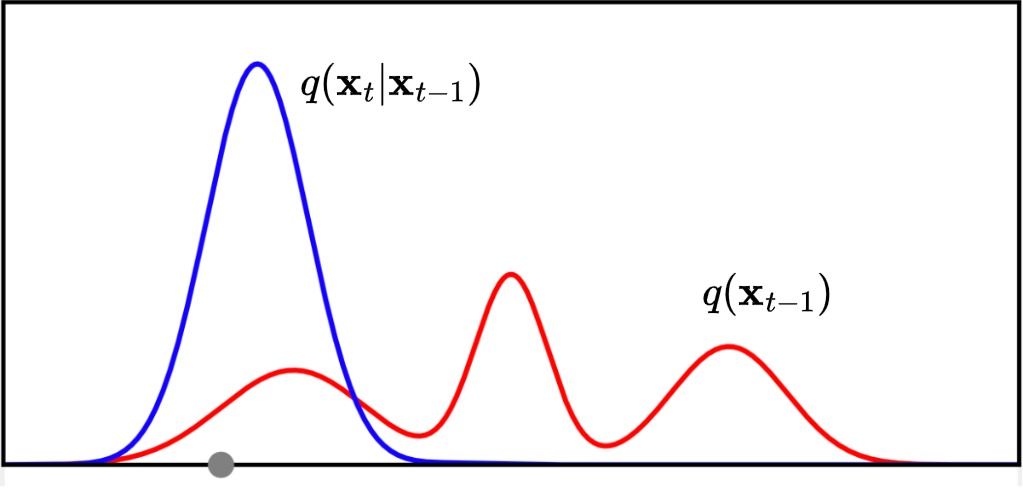}
    \caption[Intuition for why the denoising process of a diffusion model is approximately Gaussian.]{Intuition for why the true denoising process $q(\x_{t-1}|\x_t)$ is approximately Gaussian, since the product of $q(\x_t|\x_{t-1})$ and $q(\x_{t-1}|\x_t)$ will be bell-shaped. This figure and the idea behind the explanation were adapted from~\cite{Bishop2024}.}
    \label{fig:theory:bishop}
\end{figure}

The true posterior can be approximated using a normal distribution whose mean $\mu_{\theta}(\x_t, t)$ and covariance $\Sigma_{\theta}(\x_t, t)$ are predicted by a neural network with parameters $\theta$:
\begin{align}
    p_{\theta}(\x_{t-1}|\x_t) = \mathcal{N}(\x_{t-1};\mu_{\theta}(\x_t, t), \Sigma_{\theta}(\x_t, t))
\end{align}

The covariance matrix of~\cite{ho2020denoising} is assumed to be diagonal $\Sigma_{\theta}(\x_t, t)=\sigma_t^2I$, where $\sigma_t^2=\tilde{\beta}_t$, a known quantity. As such, the neural network does not have to predict a separate variance term. Experimentally, \cite{ho2020denoising} shows that selecting $\sigma_t^2=\beta_t$ produces similar results.

\paragraph{Evidence Lower Bound (ELBO)}

% \begin{lemma}
    The \text{log}-likelihood is lower bounded by:
    \begin{align*}
        \text{log}~p_{\theta}(\x_0) \geq \E_{\x_{1:T}\sim q(\cdot|\x_0)}[\text{log}~p_{\theta}(\x_{0:T}) - \text{log}~q(\x_{1:T}|\x_0)]
    \end{align*}
% \end{lemma}

\begin{proof}
    \begin{align*}
        \text{log}~p_{\theta}(\x_0) &= \text{log} \int p_{\theta}(\x_0, \x_{1:T})d\x_{1:T} \\
        &= \text{log} \int p_{\theta}(\x_0, \x_{1:T})\frac{q(\x_{1:T}|\x_0)}{q(\x_{1:T}|\x_0)}d\x_{1:T} \\
        &= \text{log} \E_{\x_{1:T}\sim q(\cdot|\x_0)} \left[ \frac{p_{\theta}(\x_{0:T})}{q(\x_{1:T}|\x_0)} \right] \\
        &\geq \E_{\x_{1:T}\sim q(\cdot|\x_0)}[\text{log}~p_{\theta}(\x_{0:T}) - \text{log}~q(\x_{1:T}|\x_0)] && \text{Jensen's inequality for log, concave}
    \end{align*}
\end{proof}
In an optimisation procedure, maximising the evidence lower bound would maximise the data log-likelihood. However, in practice, \(E_{\x_{1:T}\sim q(\cdot|\x_0)}[\text{log}~p_{\theta}(\x_{0:T}) - \text{log}~q(\x_{1:T}|\x_0)]\) remains hard to compute as it relies on the complete Markov process.

The negative log-likelihood is bounded by the negative evidence lower bound, which can be rewritten as follows:
\begin{align*}
    & -log~p_{\theta}(\x_0) \\
    &\leq -\E_{\x_{1:T}\sim q(\cdot|\x_0)}[\text{log}~p_{\theta}(\x_{0:T}) - \text{log}~q(\x_{1:T}|\x_0)] \\
    &= \E_{\x_{1:T}\sim q(\cdot|\x_0)} \left[ -\text{log}~\frac{p_{\theta}(\x_{0:T})}{q(\x_{1:T}|\x_0)} \right] \\
    &= \E_{\x_{1:T}\sim q(\cdot|\x_0)} \left[ -\text{log}~p_{\theta}(\x_{T}) - \sum_{t \geq 1} \text{log}~\frac{p_{\theta}(\x_{t-1}|\x_t)}{q(\x_t|\x_{t-1})} \right] 
 \text{   \qquad from the Markov property}\\
    &= \E_{\x_{1:T}\sim q(\cdot|\x_0)} \left[ -\text{log}~p_{\theta}(\x_{T}) - \sum_{t > 1} \text{log}~\frac{p_{\theta}(\x_{t-1}|\x_t)}{q(\x_t|\x_{t-1})} - \text{log}~\frac{p_{\theta}(\x_0|\x_1)}{q(\x_1|\x_0)} \right] \\
    &= \E_{\x_{1:T}\sim q(\cdot|\x_0)} \left[ -\text{log}~p_{\theta}(\x_{T}) - \sum_{t > 1} \text{log}~\frac{p_{\theta}(\x_{t-1}|\x_t)}{q(\x_{t-1}|\x_{t}, \x_0)} - \sum_{t>1} \text{log}~\frac{q(\x_{t-1}|\x_0)}{q(\x_{t}|\x_0)} - \text{log}~\frac{p_{\theta}(\x_0|\x_1)}{q(\x_1|\x_0)} \right] \\
    &= \E_{\x_{1:T}\sim q(\cdot|\x_0)} \left[ \text{log}~\frac{q(\x_T|\x_0)}{p_{\theta}(\x_{T})} + \sum_{t > 1} \text{log}~\frac{q(\x_{t-1}|\x_{t}, \x_0)}{p_{\theta}(\x_{t-1}|\x_t)} - \text{log}~p_{\theta}(\x_0|\x_1) \right] \text{ sum terms cancelled out}\\
    &\triangleq \E_{\x_{1:T}\sim q(\cdot|\x_0)} \left[ L_T + \sum_{t>1} L_{t-1} - L_0 \right]
\end{align*}

The expected negative log-likelihood is:
\begin{align*}
    \E_{\x_0 \sim q(\x_0)} \left[ -log~p_{\theta}(\x_0) \right]
    \leq \E_{\x_0 \sim q(\x_0),\x_{1:T}\sim q(\cdot|\x_0)} \left[ L_T + \sum_{t>1} L_{t-1} - L_0 \right]
\end{align*}

\begin{lemma}
    \begin{align*}
        L_{t-1} = \frac{1}{2\sigma_t^2}\norm{\tilde{\mu}_{t}(\x_t,\x_0) - \mu_{\theta}(\x_t, t)}^2 + C
    \end{align*}
\end{lemma}

\begin{proof}
    \begin{align*}
        L_{t-1} &= \text{log}~q(\x_{t-1}|\x_t, \x_0) - \text{log}~p_{\theta}(\x_{t-1}|\x_t) \\
        &= \text{log}~\mathcal{N}(\x_{t-1};\tilde{\mu}_{t}(\x_t,\x_0),\tilde{\beta}_{t}{\I}) - \text{log}~\mathcal{N}(\x_{t-1};\mu_{\theta}(\x_t, t), \sigma_t^2\I) \\
        &= \frac{1}{2\sigma_t^2}\norm{\x_{t-1} - \mu_{\theta}(\x_t, t)}^2 - \frac{1}{2\tilde{\beta}_{t}}\norm{\x_{t-1} - \tilde{\mu}_{t}(\x_t,\x_0)}^2 + C \\
        &= \frac{1}{{2\sigma_t^2}}\norm{\tilde{\mu}_{t}(\x_t,\x_0) - \mu_{\theta}(\x_t, t)}^2 + C
    \end{align*}
\end{proof}

In practical implementations, $L_0$ and $L_T$, which correspond to an initial reconstruction and terminal KL terms, are ignored for the purpose of optimisation, together with the constant $C$ of $L_{t-1}$. The simplified loss function is:
\begin{align*}
    & \E_{\x_0 \sim q(\x_0),\x_{1:T}\sim q(\cdot|\x_0)} \left[ L_T + \sum_{t>1} L_{t-1} - L_0 \right] \\
    & \propto \E_{\x_0 \sim q(\x_0),\x_{1:T}\sim q(\cdot|\x_0)} \left[ \sum_{t>1} L_{t-1} \right] \\
    &= \E_{\x_0 \sim q(\x_0),\x_{1:T}\sim q(\cdot|\x_0)} \left[ \sum_{t>1} \left( \frac{1}{{2\sigma_t^2}}\norm{\tilde{\mu}_{t}(\x_t,\x_0) - \mu_{\theta}(\x_t, t)}^2 + C \right) \right] \\
    & \propto \E_{\x_0 \sim q(\x_0),\x_{1:T}\sim q(\cdot|\x_0)} \left[ \sum_{t>1} \left( \frac{1}{{2\sigma_t^2}}\norm{\tilde{\mu}_{t}(\x_t,\x_0) - \mu_{\theta}(\x_t, t)}^2 \right) \right] \\
    &\triangleq \mathcal{L(\theta)}
\end{align*}

Recall \autoref{theorem:miu_theta}, where $ \tilde{\mu}_{t}(\x_{t},\x_{0}):=\frac{\sqrt{\bar{\alpha}_{t-1}}\beta_{t}}{1-\bar{\alpha}_{t}}\x_{0}+\frac{\sqrt{\alpha_t}(1 - \bar{\alpha}_{t-1})}{1-\bar{\alpha}_{t}}\x_{t} \text{ and } \tilde{\beta}_{t}:=\frac{1-\bar{\alpha}_{t-1}}{1-\bar{\alpha}_{t}}\beta_{t}$

By optimising $\mathcal{L(\theta)}$, the model $\mu_{\theta}$ can be trained to predict the mean of the probability distribution of $\x_{t-1}$, given $\x_t$ and the timestep $t$, thus approximating $q(\x_{t-1}|\x_t, \x_0)$ through $p_{\theta}(\x_{t-1}|\x_t)$. At generation time, iteratively sampling from $p_{\theta}(\x_{t-1}|\x_t)$, gradually removes noise, ultimately producing a new sample. 

\paragraph{Reparameterisation} \cite{ho2020denoising} proposes a reparameterisation trick that improves results.

Recall each state $\x_t$ can be written in terms of the original state $\x_0$ as:
\begin{align*}
    \x_t=\sqrt{\bar{\alpha}_t}\x_0 + \sqrt{1 - \bar{\alpha}_t}\tilde{\epsilon}_t \implies \x_0 = \frac{1}{\sqrt{\bar{\alpha}_t}}(\x_t - \sqrt{1 - \bar{\alpha}_t}\tilde{\epsilon}_t) \\
\end{align*}
Where $\tilde{\epsilon}_t \sim \mathcal{N}(0, \I)$ is the full noise added to the original state $\x_0$

\begin{corollary}
\label{coro:miu}
The mean of the true posterior can be written as:
    \begin{align*}
        \tilde{\mu}_{t}(\x_t,\x_0) = \tilde{\mu}_{t} \left( \x_t, \frac{1}{\sqrt{\alpha_t}}(\x_t - \sqrt{1 - \bar{\alpha}_t}\tilde{\epsilon}_t) \right) = \frac{1}{\sqrt{\alpha_t}} \left( \x_t - \frac{\beta_t}{\sqrt{1 - \bar{\alpha}_t}}\tilde{\epsilon}_t \right)
    \end{align*}
\end{corollary}

\begin{proof}
    \begin{align*}
        \tilde{\mu}_{t}(\x_t,\x_0) &= \frac{\sqrt{\bar{\alpha}_{t-1}}\beta_{t}}{1-\bar{\alpha}_{t}}\x_{0}+\frac{\sqrt{\alpha_t}(1 - \bar{\alpha}_{t-1})}{1-\bar{\alpha}_{t}}\x_{t} \\
        &= \frac{\sqrt{\bar{\alpha}_{t-1}}\beta_{t}}{1-\bar{\alpha}_{t}}\frac{1}{\sqrt{\bar{\alpha}_t}}(\x_t - \sqrt{1 - \bar{\alpha}_t}\tilde{\epsilon}_t) + \frac{\sqrt{\alpha_t}(1 - \bar{\alpha}_{t-1})}{1-\bar{\alpha}_{t}}\x_{t} \\
        &= \frac{\beta_t + \alpha_t(1 - \bar{\alpha}_{t-1})}{(1-\bar{\alpha}_t)\sqrt{\alpha_t}} \x_t - \frac{\beta_t}{\sqrt{1 - \bar{\alpha}_t}\sqrt{\alpha_t}}\tilde{\epsilon}_t \\ 
        &= \frac{1 - \alpha_t + \alpha_t - \bar{\alpha}_t}{(1-\bar{\alpha}_t)\sqrt{\alpha_t}} \x_t - \frac{\beta_t}{\sqrt{1 - \bar{\alpha}_t}\sqrt{\alpha_t}}\tilde{\epsilon}_t \\ 
        &= \frac{1}{\sqrt{\alpha_t}} \left( \x_t - \frac{\beta_t}{\sqrt{1 - \bar{\alpha}_t}}\tilde{\epsilon}_t \right)
    \end{align*}
\end{proof}

\begin{theorem}
    In the reparameterised DDPM framework, the underlying model is trained to predict the noise $\tilde{\epsilon}_t$ added to the original state $\x_0$, instead of predicting the denoised representation $\x_{t-1}$, from $\x_t$ and $t$:
    \begin{align}
        \mu_\theta(\x_t, t) &= \frac{1}{\sqrt{\alpha_t}} \left( \x_t - \frac{\beta_t}{\sqrt{1 - \bar{\alpha}_t}} \epsilon_{\theta}(\x_t, t) \right) \\
        \implies L_{t-1} &= \frac{\beta_t^2}{2\sigma_t^2\alpha_t(1 - \bar{\alpha}_t)}\norm{\tilde{\epsilon}_t - \epsilon_{\theta}(\x_t, t)}^2 + C
    \end{align}
\end{theorem}

\begin{proof}
    \begin{align*}
        L_{t-1} &= \frac{1}{2\sigma_t^2}\norm{\tilde{\mu}_{t}(\x_t,\x_0) - \mu_{\theta}(\x_t, t)}^2 + C \\
        &= \frac{1}{2\sigma_t^2}\norm{\frac{1}{\sqrt{\alpha_t}} \left( \x_t - \frac{\beta_t}{\sqrt{1 - \bar{\alpha}_t}}\tilde{\epsilon}_t \right) - \frac{1}{\sqrt{\alpha_t}} \left( \x_t - \frac{\beta_t}{\sqrt{1 - \bar{\alpha}_t}} \epsilon_{\theta}(\x_t, t) \right)}^2 + C \\
        &= \frac{\beta_t^2}{2\sigma_t^2\alpha_t(1 - \bar{\alpha}_t)}\norm{\tilde{\epsilon}_t - \epsilon_{\theta}(\x_t, t)}^2 + C
    \end{align*}
\end{proof}

% \begin{corollary}
The simplified loss of the reparameterised model is:
\begin{align*}
    \mathcal{L'}(\theta) = \E_{t \sim (1, T],\x_0 \sim q(\x_0),\tilde{\epsilon}_t \sim \mathcal{N}(0, \I)} \left[ \norm{\tilde{\epsilon}_t - \epsilon_{\theta}(\sqrt{\bar{\alpha}_t}\x_0 + \sqrt{1 - \bar{\alpha}_t}\tilde{\epsilon}_t, t)}^2 \right]
\end{align*}
% \end{corollary}

Derivation
\begin{align*}
    &\E_{\x_0 \sim q(\x_0)} \left[ -log~p_{\theta}(\x_0) \right] \\
    &\leq \E_{\x_0 \sim q(\x_0),\x_{1:T}\sim q(\cdot|\x_0)} \left[ L_T + \sum_{t>1} L_{t-1} - L_0 \right] \\
    & \propto \E_{\x_0 \sim q(\x_0),\x_{1:T}\sim q(\cdot|\x_0)} \left[ \sum_{t>1} L_{t-1} \right] \\
    &= \E_{\x_0 \sim q(\x_0),\x_{1:T}\sim q(\cdot|\x_0),\tilde{\epsilon}_t \sim \mathcal{N}(0, \I)} \left[ \sum_{t>1} \left( \frac{\beta_t^2}{2\sigma_t^2\alpha_t(1 - \bar{\alpha}_t)}\norm{\tilde{\epsilon}_t - \epsilon_{\theta}(\x_t, t)}^2 + C \right) \right] \\
    & \propto \E_{\x_0 \sim q(\x_0),\x_{1:T}\sim q(\cdot|\x_0),\tilde{\epsilon}_t \sim \mathcal{N}(0, \I)} \left[ \sum_{t>1} \norm{\tilde{\epsilon}_t - \epsilon_{\theta}(\x_t, t)}^2 \right] \\
    &= \E_{t \sim (1, T],\x_0 \sim q(\x_0),\tilde{\epsilon}_t \sim \mathcal{N}(0, \I)} \left[ \norm{\tilde{\epsilon}_t - \epsilon_{\theta}(\x_t, t)}^2 \right] \text{   recall } \x_t=\sqrt{\bar{\alpha}_t}\x_0 + \sqrt{1 - \bar{\alpha}_t}\tilde{\epsilon}_t \\
    &= \E_{t \sim (1, T],\x_0 \sim q(\x_0),\tilde{\epsilon}_t \sim \mathcal{N}(0, \I)} \left[ \norm{\tilde{\epsilon}_t - \epsilon_{\theta}(\sqrt{\bar{\alpha}_t}\x_0 + \sqrt{1 - \bar{\alpha}_t}\tilde{\epsilon}_t, t)}^2 \right] \\
    &\triangleq \mathcal{L'}(\theta)
\end{align*}

 In this formulation, $\x_t$ is generated by corrupting the clean sample $\x_0$ with noise $\tilde{\epsilon}_t$ according to the known forward diffusion process. By minimising $\mathcal{L'}(\theta)$, the network $\epsilon_{\theta}$ learns to undo the full noise added to $\x_0$, which resulted in $\x_t$, by conditioning on this noisy sample and the timestep $t$. Notably, each training step samples a single timestep $t$ randomly, unlike at generation time, where all timesteps must be sequentially traversed.

 \paragraph{Generation at test time}
Once the model $\epsilon_{\theta}$ has been trained, sampling a new data point proceeds by iteratively applying the learnt reverse denoising process, starting from pure Gaussian noise $\x_T \sim \mathcal{N}(0, \I)$. At each timestep $t$, the model predicts the full noise component $\epsilon_{\theta}(\x_t, t)$, and a sample $x_{t-1}$ from the learnt posterior $p_{\theta}(\x_{t-1}|\x_t)$ is obtained using the following update rule:
\begin{align*}
    \x_{t-1} = \frac{1}{\sqrt{\alpha_t}} \left( \x_t - \frac{\beta_t}{\sqrt{1-\bar{\alpha}_t}} \, \epsilon_{\theta}(\x_t, t) \right) + \sigma_t \z,
\end{align*}
where $\mathbf{e} \sim \mathcal{N}(0, {\I})$ is standard Gaussian noise and $\sigma_t$ is the standard deviation corresponding to timestep $t$. This is because $p_{\theta}(\x_{t-1}|\x_t)$ is a Gaussian distribution, trained to approximate a true posterior with $\tilde{\mu}_{t}(\x_t,\x_0) = \frac{1}{\sqrt{\alpha_t}} \left( \x_t - \frac{\beta_t}{\sqrt{1 - \bar{\alpha}_t}}\tilde{\epsilon}_t \right)$, as demonstrated in Corollary~\ref{coro:miu}. This procedure is repeated sequentially for $t = T, T-1, \dots, 1$ (i.e. 1000 steps).

\subsection{Denoising Diffusion Implicit Model (DDIM)}

\begin{figure}
    \centering
    \includegraphics[width=0.6\linewidth]{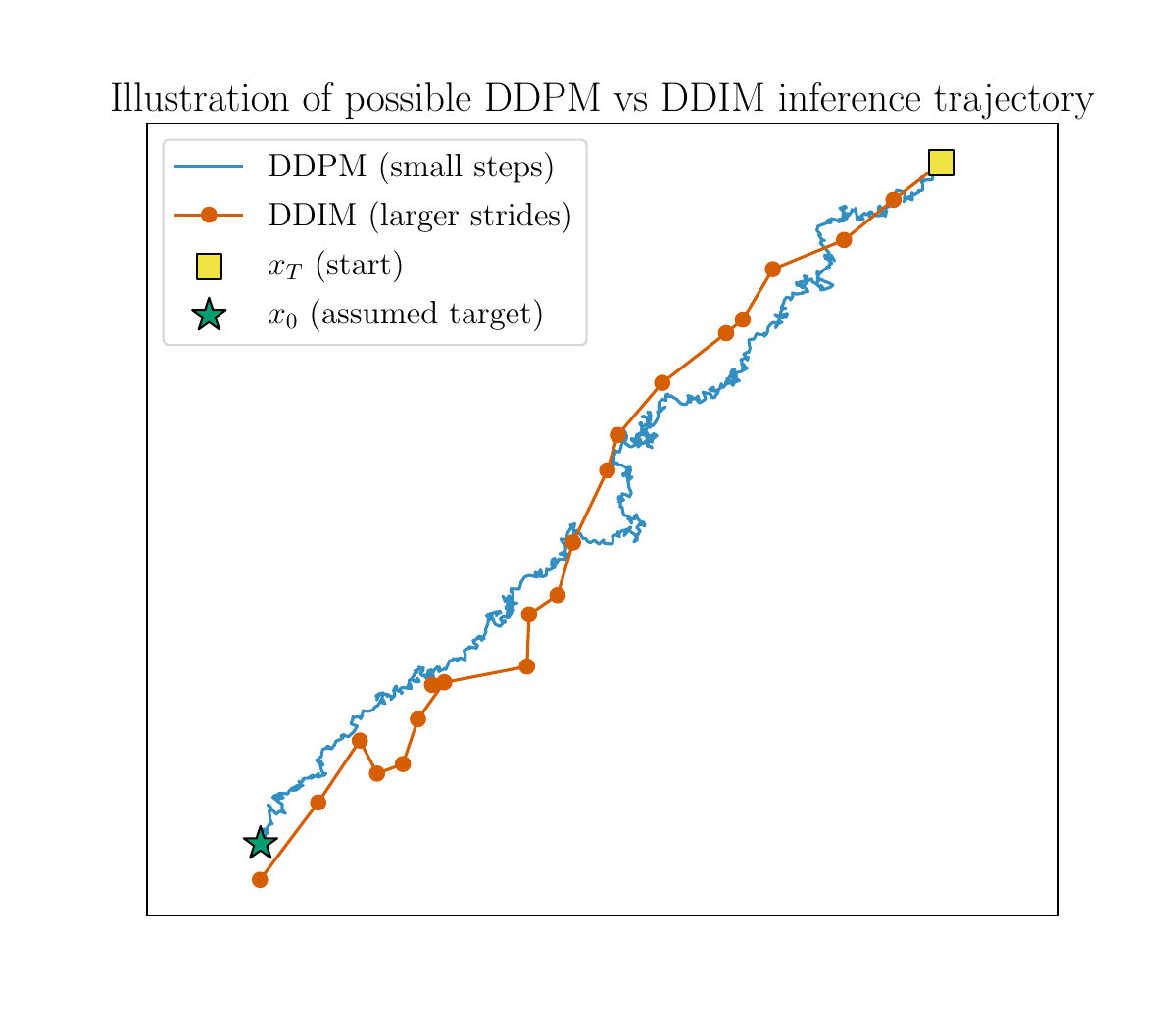}
    \caption[Illustration of the denoising trajectories of DDIM and DDPM]{Comparison between the denoising trajectories of DDPM and DDIM. DDPM follows a noisy path with small incremental steps, eventually reaching \( x_0 \). In contrast, DDIM takes larger strides, which can lead to divergence from the true data distribution. This figure is only meant to illustrate the differences.}
    \label{fig:enter-label}
\end{figure}

While DDPMs achieve impressive generative performance, their generation procedure is extremely slow. To obtain a high-quality sample with no noise, a large number of sequential denoising steps (typically $T = 1000$) is required. Each step introduces only a small amount of denoising, meaning that hundreds of iterative updates are necessary to progressively remove noise from the initial Gaussian input.

Denoising Diffusion Implicit Models (DDIMs)~\cite{song2020denoising} address this inefficiency by reformulating the sampling process. The key idea is to traverse the learnt denoising trajectory in larger strides, reducing the number of steps needed during generation without significantly sacrificing sample quality.

Intuitively, at each timestep $t$, the model $\epsilon_{\theta}$ predicts the full noise component in $\x_t$. Rather than taking a small stochastic step from $\x_t$ to $\x_{t-1}$ (as in DDPM), DDIMs interpret $\epsilon_{\theta}(\x_t, t)$ as providing the direction towards the true data subspace\footnote{This is also referred to as the data manifold, where samples from the real distribution lie on a surface within the high-dimensional space of all possible data points.}, and then deterministically take a larger step in that direction, yet smaller than the full denoising that could be inferred from the neural network output.

\paragraph{Sampling trade-offs}
By choosing a suitable number of sampling steps (often $\approx50$ instead of $1000$), DDIM significantly accelerates the generation process. However, larger stride steps imply that the denoising trajectory might deviate from the true data and introduce artefacts.

\subsection{Latent Diffusion Model (LDM)}

Training denoising diffusion probabilistic models (DDPMs)~\cite{ho2020denoising} directly on high-resolution images, such as those of size $512 \times 512 \times 3$, is prohibitively expensive in terms of computational resources. Latent Diffusion Models (LDMs)~\cite{rombach2022high} offer a practical solution by performing the generative process within a lower-dimensional latent space. Typically, images are encoded into a latent representation of size $64 \times 64 \times 4$, significantly reducing memory and compute requirements while retaining essential semantic and structural information.

This dimensionality reduction is achieved through a pre-trained Variational Autoencoder (VAE)~\cite{kingma2013auto}, where the encoder compresses the input image into a latent vector $\z_0$, and the decoder reconstructs it back into pixel space. During the diffusion model's training, the VAE parameters are kept fixed, ensuring the latent space remains stable and semantically meaningful. The diffusion model is trained to denoise latent representations rather than raw pixel data, enabling faster and more scalable training. At the same time, decoded samples remain visually realistic and consistent with the data distribution.

\paragraph{Base architecture}

The core architecture of the latent diffusion model is a U-Net~\cite{ronneberger2015u}, introduced initially for biomedical image segmentation. The U-Net comprises a contracting path, which captures contextual features at various spatial resolutions, and a symmetric expanding path, which supports precise spatial reconstruction via skip connections. This hierarchical structure makes U-Nets especially effective for modelling the complex dependencies between features in natural images. In modern latent diffusion models, this U-Net is further improved with attention layers~\cite{vaswani2017attention}, which enable the model to selectively focus on relevant spatial and semantic regions within the latent input, encouraging a global receptive field.

\paragraph{Controllable generation}

To enable conditional generation, LDMs incorporate additional information, such as class labels, textual descriptions, or visual cues, during training and inference. This is achieved by extending the diffusion model's loss function to depend on the noisy latent and a conditioning signal $C$. Formally, the training loss becomes:

\begin{corollary}
    The simplified loss of the conditional latent diffusion model is given by
    \begin{align*}
        \mathcal{L''}(\theta) = \mathbb{E}_{t \sim [1, T], \z_0 \sim q(\z_0), \tilde{\epsilon}_t \sim \mathcal{N}(0, \I)} \left[ \left\| \tilde{\epsilon}_t - \epsilon_{\theta}\left(\sqrt{\bar{\alpha}_t} \z_0 + \sqrt{1 - \bar{\alpha}_t} \tilde{\epsilon}_t, t, C \right) \right\|^2 \right],
    \end{align*}
\end{corollary}

\noindent
where $\z_t = \sqrt{\bar{\alpha}_t} \z_0 + \sqrt{1 - \bar{\alpha}_t} \tilde{\epsilon}_t$ represents the noisy latent representation at time step $t$. The model learns to predict the noise $\tilde{\epsilon}_t$ added to the clean latent $\z_0$, conditioned on $C$.

More advanced conditioning techniques, such as ControlNet~\cite{zhang2023controlnet}, have been proposed to enhance controllability. These models introduce conditioning signals, such as edge maps, keypoints, or segmentation masks, at multiple levels in the decoder of the U-Net architecture. By injecting information at various spatial resolutions, ControlNet allows for fine-grained manipulation of specific visual attributes, such as object pose, layout, or structural detail, during the generation process.

Overall, latent diffusion offers a scalable and flexible framework for high-resolution image synthesis, with support for structured conditioning and fine-grained control through simple mechanisms.

\subsection{Diffusion Inpainting}

\begin{figure}[ht]
    \centering
    \includegraphics[width=0.8\linewidth]{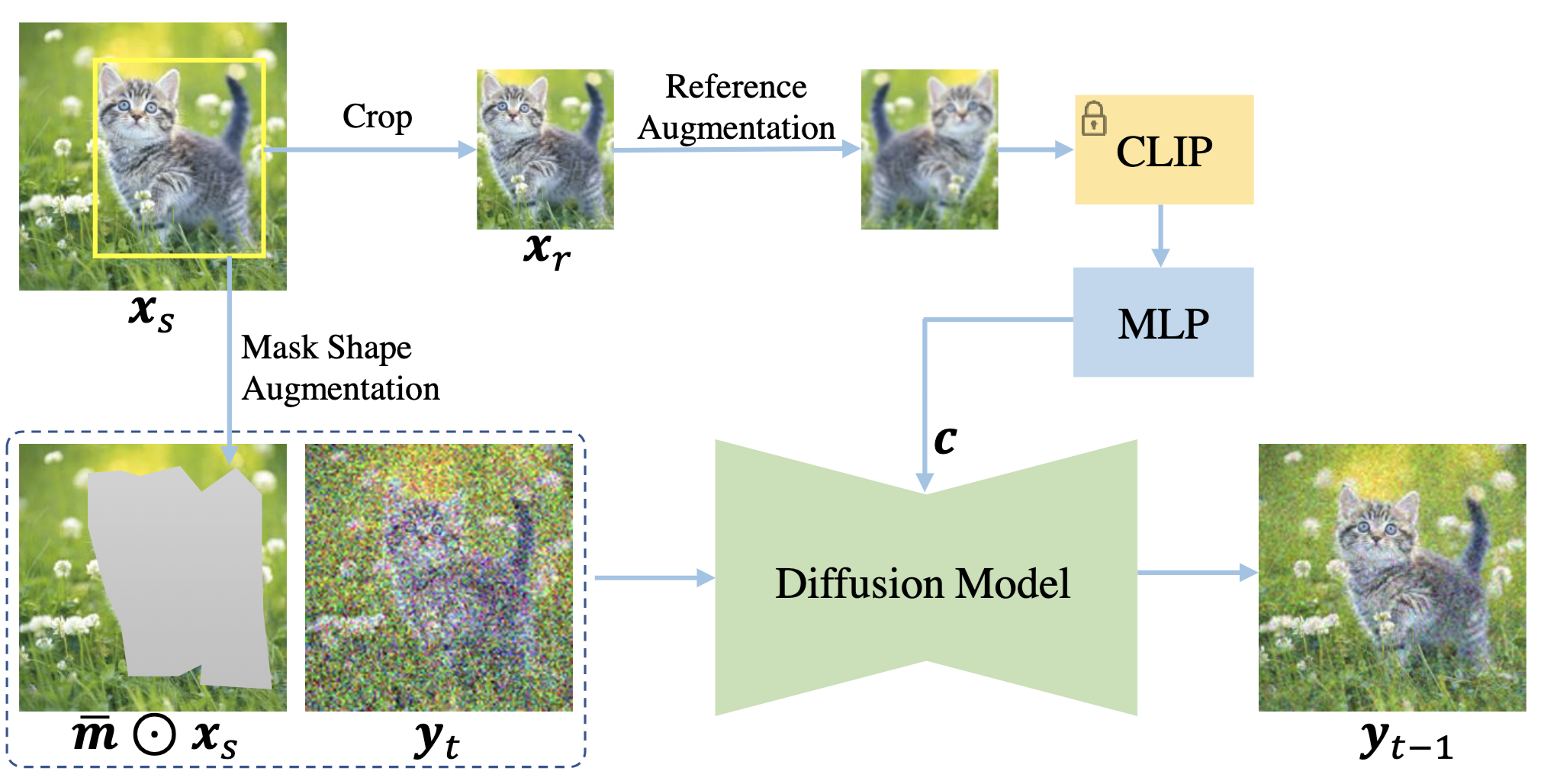}
    \caption[Architecture and training pipeline of Paint-by-Example.]{Architecture and training pipeline of Paint-by-Example~\cite{yang2023paint}.}
    \label{fig:pbe}
\end{figure}

\paragraph{Inpainting}
Inpainting is the task of reconstructing missing or occluded regions of an image in a semantically coherent and visually plausible manner. Within the diffusion framework, this is achieved by predicting suitable content for masked regions through iterative denoising of a noisy latent representation, typically conditioned on the unmasked context. The image with a binary mask applied is first projected into a latent space using a pre-trained VAE. The diffusion model is then trained to synthesise the complete latent representation, which is decoded into the image space.

\paragraph{Compositing}
While traditional inpainting involves the restoration of missing parts without external guidance, compositing—referred to here as \emph{reference-guided inpainting} in this work—involves the integration of visual content from an external source. In this setting, the goal is to fill in masked regions and insert an object or concept derived from a separate reference image. This requires generating content contextually consistent with the destination scene and visually aligned with the provided reference.

\paragraph{Paint-by-Example}

The Paint-by-Example (PBE)~\cite{yang2023paint} framework extends image inpainting by conditioning the diffusion process on a semantic encoding of a reference image. During training, an object is removed from the destination scene, and a masked latent is constructed. The diffusion model is trained to reinsert the missing object by denoising the latent representation of the composited image while being conditioned on features extracted from the reference image and the latent representation of the image context.

An informational bottleneck is imposed on the reference encoding to promote generalisation and semantic transfer. This prevents the model from memorising low-level details and encourages learning abstract, transferable representations. As a result, the model learns to synthesise context-aware insertions that blend naturally into the scene. An overview of the Paint-by-Example training setup, including the conditioning mechanism, is provided in~\autoref{fig:pbe}.

\chapter{MObI: Multimodal Object Inpainting Using Diffusion Models}
\label{chaper:mobi}

This chapter is based on the following first-authored peer-reviewed publication~\cite{buburuzan2025mobi};

\begin{quote} \textbf{Buburuzan, A.}, Sharma, A., Redford, J., Dokania, P.K., and Mueller, R. (2025). \textit{MObI: Multimodal Object Inpainting Using Diffusion Models}. In Proceedings of the Computer Vision and Pattern Recognition Conference Workshops (CVPRW) (pp. 1974-1984). \end{quote}

\newpage

\section{Introduction}
Extensive multimodal data, including camera and lidar, is crucial for the safe testing and deployment of autonomous driving systems. 
However, collecting large amounts of multimodal data in the real world can be prohibitively expensive because rare but high-severity failures have an outstripped impact on the overall safety of such systems~\cite{koopman2016challenges}.
Synthetic data offers a way to address this problem by allowing the generation of diverse safety-critical situations before deployment. Still, existing methods often fall short either by lacking controllability or realism.

For example, reference-based image inpainting methods~\cite{yang2023paint, chen2023anydoor, ruiz2024magicinsertstyleawaredraganddrop,kulal2023puttingpeopleplaceaffordanceaware} can produce realistic samples that seamlessly blend into the scene using a single reference, but they often lack precise control over the 3D positioning and orientation of the inserted objects.
In contrast, methods based on actor insertion using 3D assets~\cite{wang2023cadsim, chang2024just, zhou2023scene, wei2024editable, chen2021geosim, lin2024drive, multitest, li2023lift3d} provide a high degree of control---enabling precise object placement in the scene---but often struggle to achieve realistic blending and require high-quality 3D assets, which can be challenging to produce. Similarly, reconstruction methods~\cite{prism1bywayve,tonderski2024neurad, yang2023unisim} are also highly controllable but require almost full coverage of the inserted actor.
Some of these shortcomings are illustrated in \autoref{fig:failure modes of pbe}.
More recent methods have explored 3D geometric control for image editing~\cite{wang2025diffusion, wu2024neural, yenphraphai2024image, pandey2024diffusion, michel2024object, yuan2023customnet}. However, none consider multimodal generation, which is crucial in autonomous driving.

Recent advancements in controllable full-scene generation in autonomous driving for multiple cameras~\cite{gao2023magicdrive, li2023drivingdiffusion, wen2023panacea, su2024text2street, huang2024subjectdrivescalinggenerativedata, drivescape}, and lidar~\cite{lidardiffusion, lidargen, hu2024rangeldmfastrealisticlidar, xiong2023ultralidarlearningcompactrepresentations, bian2024dynamiccitylargescalelidargeneration,xie2024x} have led to impressive results. 
However, generating full scenes can create a large domain gap, especially for downstream tasks such as object detection, making it challenging to generate realistic counterfactual examples.
For this reason, works such as GenMM~\cite{singh2024genmm} have focused instead on camera-lidar object inpainting using a multi-stage pipeline.
This work takes a similar approach, but proposes an end-to-end method that generates camera and lidar jointly.

\noindent The contributions of this work are:
\begin{itemize}
    \item A multimodal inpainting approach for joint camera-lidar editing using a single reference image.
    \item Conditioning the object inpainting process on a 3D bounding box to ensure accurate spatial placement.
    \item Demonstrating the generation of realistic and controllable multimodal counterfactuals of driving scenes.
\end{itemize}

\begin{figure*}[htbp]
    \centering
    \footnotesize
    \begin{minipage}{0.57\textwidth}
    \begin{tabularx}{0.96\columnwidth}{@{}>{\centering\arraybackslash}X>{\centering\arraybackslash}X>{\centering\arraybackslash}X>{\centering\arraybackslash}X@{}}
        Reference & Edit mask & PbE~\cite{yang2023paint} & MObI \\
    \end{tabularx}
    \includegraphics[width=\columnwidth]{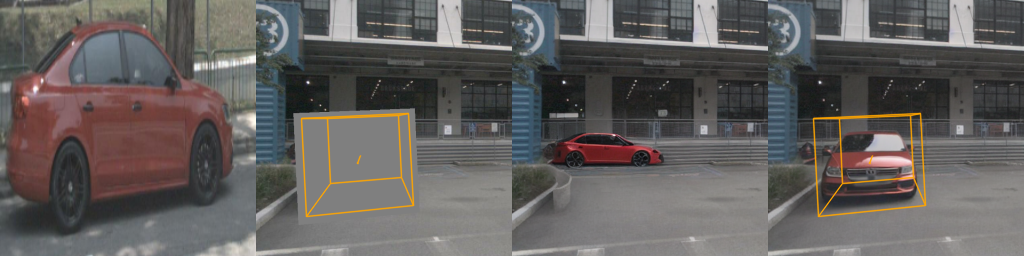}
    \includegraphics[width=\columnwidth]{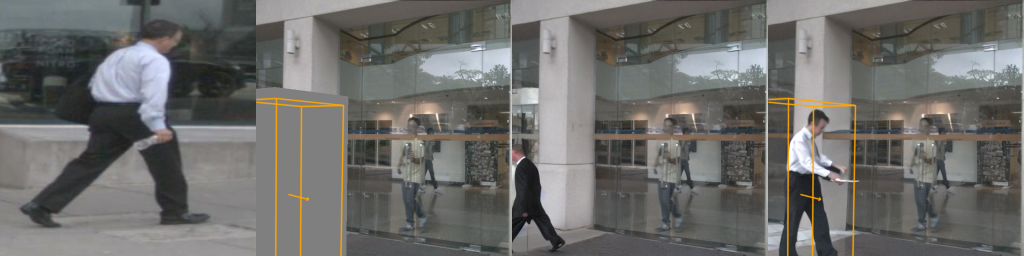} \\
    \centering Object inpainting
    \end{minipage}
    \footnotesize
    \begin{minipage}{0.4\textwidth}
    \centering
    \vspace{11pt}
    \begin{tabular}{r@{\hspace{4pt}}c@{\hspace{2pt}}c@{\hspace{4pt}}l}
        % First row: Label, Original, PbE, Label
        \rotatebox{90}{\qquad{Original}} & 
        \includegraphics[width=0.35\columnwidth]{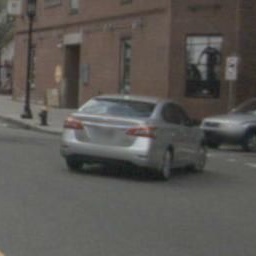} &
        \includegraphics[width=0.35\columnwidth]{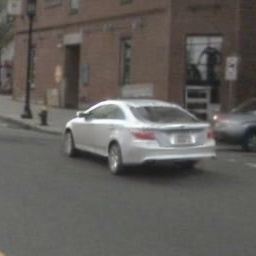} & 
        \rotatebox{90}{\qquad{PbE~\cite{yang2023paint}}} \\
        % Second row: Label, NeuRAD, Ours, Label
        \rotatebox{90}{{\quad NeuRAD~\cite{tonderski2024neurad}}} & 
        \includegraphics[width=0.35\columnwidth]{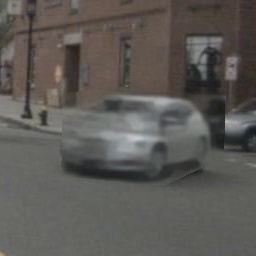} &
        \includegraphics[width=0.35\columnwidth]{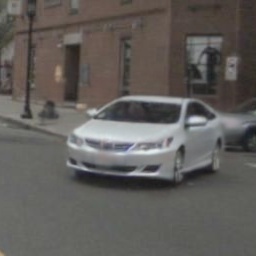} & 
        \rotatebox{90}{\qquad\quad{MObI}} \\
    \end{tabular}
    \centering Object 180$^\circ$ flip\\
    \end{minipage}
    \caption[]{The proposed method can inpaint objects with a high degree of realism and controllability. Left: object inpainting methods based on edit masks alone such as Paint-by-Example~\cite{yang2023paint} (PbE) achieve high realism but can lead to surprising results because there are often multiple semantically consistent ways to inpaint an object within a scene.
    Right: methods based on 3D reconstruction such as NeuRAD~\cite{tonderski2024neurad} have strong controllability but sometimes lead to low realism, especially for object viewpoints that have not been observed.
    The proposed method achieves both high semantic consistency and controllability of the generation.
    }
    \label{fig:failure modes of pbe}
\end{figure*}

\section{Related work}

Multimodal data is crucial for ensuring safety in autonomous driving, and most state-of-the-art perception systems employ a sensor fusion approach, particularly for tasks like 3D object detection~\cite{liang2022bevfusionsimplerobustlidarcamera, liu2023bevfusion, gunn2024liftattendsplatbirdseyeviewcameralidarfusion}. However, testing and developing such safety-critical systems require vast amounts of data, which is costly and time-consuming to obtain in the real world. Consequently, there is a growing need for simulated data, enabling models to be tested efficiently without requiring on-road vehicle testing.

\paragraph{Copy-and-paste} Early efforts in synthetic data generation relied on copy-and-paste methods. For example, \cite{georgakis2017synthesizing} used depth maps for accurate scaling and positioning when inserting objects, while later approaches like~\cite{dwibedi2017cut} focused on achieving patch-level realism through blending, improving 2D object detection. A more straightforward approach, presented by~\cite{ghiasi2021simple}, naïvely pastes objects into images without blending and demonstrates its efficacy in improving image segmentation.
In autonomous driving, PointAugmenting~\cite{wang2021pointaugmenting} extends this copy-and-paste approach to both camera and lidar data to enhance 3D object detection. Building on the lidar GT-Paste method~\cite{yan2018second}, it incorporates ideas from CutMix augmentation~\cite{yun2019cutmix} while ensuring multimodal consistency. This method addresses scale mismatches and occlusions by utilising the lidar point cloud for guidance during the insertion process. Similarly, MoCa~\cite{zhang2020exploring} employs a segmentation network to extract source objects before insertion, instead of directly pasting entire patches. Geometric consistency in monocular 3D object detection has also been explored in~\cite{lian2022exploring}. While these methods improve object detection and mitigate class imbalance, their compositing strategy leads to unrealistic blending, especially in image space. Furthermore, they lack controllability, such as the ability to adjust the position and orientation of inserted objects, limiting their utility for testing.

\paragraph{Full scene generation} Recent advancements in conditional full-scene generation have yielded impressive results. BEVControl~\cite{yang2023bevcontrol} uses a two-stage method (controller and coordinator) to generate scenes conditioned on sketches, ensuring accurate foreground and background content. Text2Street~\cite{su2024text2street} combines bounding box encoding with text conditions, employing a ControlNet-like~\cite{zhang2023controlnet} architecture for guidance. DrivingDiffusion~\cite{li2023drivingdiffusion} represents bounding boxes as layout images passed as an extra channel in the U-Net~\cite{ronneberger2015u}. MagicDrive~\cite{gao2023magicdrive} incorporates bounding boxes and camera parameters alongside text conditions for full-scene generation, with a cross-view attention module leveraging BEV layouts. SubjectDrive~\cite{huang2024subjectdrivescalinggenerativedata} generates camera videos conditioned on the appearance of foreground objects. LiDM~\cite{lidardiffusion} focuses on lidar scene generation conditioned on semantic maps, text, and bounding boxes. DriveScape~\cite{drivescape} introduces a method to generate multi-view camera videos conditioned on 3D bounding boxes and maps using a bi-directional modulated transformer for spatial and temporal consistency.

Synthetic lidar data generation has also advanced significantly. LidarGen~\cite{lidargen} and LiDM~\cite{lidardiffusion} employ diffusion for lidar generation, with the latter also incorporating semantic maps, bounding boxes, and text. UltraLidar~\cite{xiong2023ultralidarlearningcompactrepresentations} densifies sparse lidar point clouds, while RangeLDM~\cite{hu2024rangeldmfastrealisticlidar} accelerates lidar data generation by converting point clouds into range images using Hough sampling and enhancing reconstruction through a range-guided discriminator. DynamicCity~\cite{bian2024dynamiccitylargescalelidargeneration} generates lidar sequences conditioned on dynamic scene layouts, and~\cite{XIANG2024105207} generates object-level lidar data, demonstrating its benefits for object detection. However, these works do not jointly generate camera and lidar data, and full-scene generation can result in a large domain gap, particularly for downstream tasks like object detection, making it challenging to create realistic counterfactuals.

\paragraph{Multimodal object inpainting} GenMM~\cite{singh2024genmm} represents a new direction in multimodal object inpainting using a multi-stage pipeline that ensures temporal consistency. However, it remains limited in controllability, requiring the reference to closely align with the insertion angle. Furthermore, it does not generate lidar and camera modalities jointly; instead, it focuses on geometric alignment while excluding lidar intensity values. This work takes a similar approach, but proposes an end-to-end method that jointly generates camera and lidar data for reference-guided multimodal object inpainting. The proposed method achieves realistic and consistent multimodal outputs across diverse object angles.

\section{Method}
\label{sec:method}

\begin{figure*}[t]
  \centering
  \includegraphics[width=\linewidth]{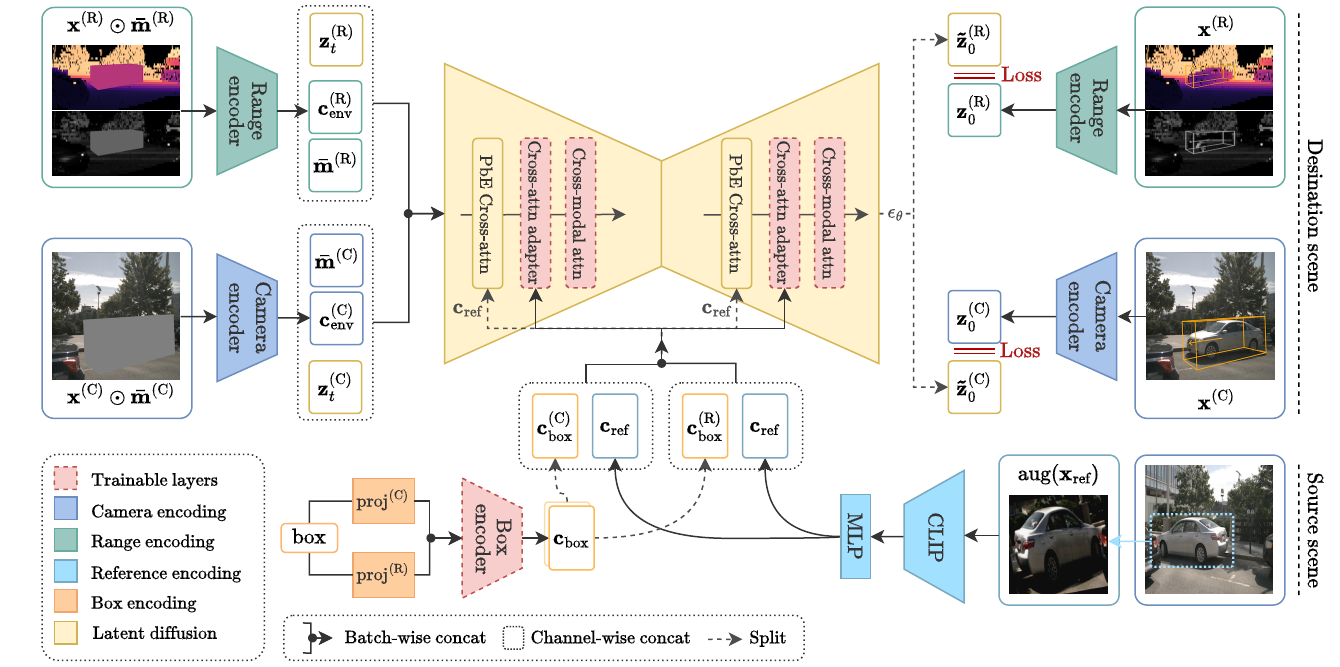}
  \caption{MObI architecture and training procedure.}
  \label{fig:architecture}
\end{figure*}

This work extends Paint-by-Example~\cite{yang2023paint} (PbE), a reference-based image inpainting method, to include bounding box conditioning and to jointly generate camera and lidar perception inputs. A diffusion model~\cite{rombach2022high, ho2020denoising, sohl2015deep} is trained using the architecture illustrated in~\autoref{fig:architecture}, where the denoising process is conditioned on the latent representations of the camera and lidar range view contexts ($\mathbf{c}^{\text{(R)}}_{\text{env}}$ and $\mathbf{c}^{\text{(C)}}_{\text{env}}$), the RGB object reference $\mathbf{c}_{\text{ref}}$, a per-modality projected 3D bounding box conditioning ($\mathbf{c}_{\text{box}}^{\text{(R)}}$ and $\mathbf{c}_{\text{box}}^{\text{(C)}}$) and the complement of the edit mask targets ($\bar{\mathbf{m}}^{\text{(C)}}$ and $\bar{\mathbf{m}}^{\text{(R)}}$). 
The diffusion model \( \epsilon_\theta \) is trained in a self-supervised manner as in~\cite{yang2023paint} to predict the full scene based on the masked-out inputs.
More formally, the model predicts the total noise added to the latent representation of the scene \(\{ \mathbf{z}_0^{\text{(R)}}, \mathbf{z}_0^{\text{(C)}} \} \) using the loss
\begin{align*}
  \mathcal{L} = \mathbb{E}_{\mathbf{z}^{\text{(R)}}_0, \mathbf{z}^{\text{(C)}}_0, t, \mathbf{c}, \epsilon \sim \mathcal{N}(0, 1)} 
  \left[ \left\| \epsilon - \epsilon_{\theta}(\mathbf{z}^{\text{(R)}}_t, \mathbf{z}^{\text{(C)}}_t, \mathbf{c}, t) \right\|^2 \right],
\end{align*}
where $\mathbf{c} = \{ \mathbf{c}^{\text{(R)}}_{\text{env}}, \mathbf{c}^{\text{(C)}}_{\text{env}}, \mathbf{c}_{\text{ref}}, \mathbf{c}_{\text{box}}^{\text{(R)}}, \mathbf{c}_{\text{box}}^{\text{(C)}}, \bar{\mathbf{m}}^{\text{(R)}}, \bar{\mathbf{m}}^{\text{(C)}}\}$.
The input of the UNet-style network~\cite{ronneberger2015u} is the noised sample ($\mathbf{z}_t^{\text{(R)}}$ and $\mathbf{z}_t^{\text{(C)}}$) at step \( t \), concatenated channel-wise with the latent representation of the scene context and its corresponding edit mask, resized to the latent dimension.

\subsection{Image processing and encoding}
\label{sec:method:multimodal encoding}

The model is trained to insert an object from a source scene with image $ I_s \in \mathbb{R}^{H \times W \times 3}$ and bonding box $ \text{box}_s \in \mathbb{R}^{8 \times 3}$, into a destination scene with corresponding camera image $ I_d \in \mathbb{R}^{H \times W \times 3}$ and annotation bounding box $ \text{box}_d \in \mathbb{R}^{8 \times 3}$. During training, these bounding boxes correspond to the same object at different timestamps, while at inference, they can be chosen arbitrarily. The bounding boxes from the source and destination scenes, $ \text{box}_s, \text{box}_d \in \mathbb{R}^{8 \times 3}$, are projected onto the image space using the respective camera transformations:
\[
\text{box}_s^{\text{(C)}} = \mathbf{T}_s^{\text{(C)}} \cdot \text{box}_s \in \mathbb{R}^{8 \times 2}, \quad 
\text{box}_d^{\text{(C)}} = \mathbf{T}_d^{\text{(C)}} \cdot \text{box}_d \in \mathbb{R}^{8 \times 2}.
\] 

Following the zoom-in strategy of AnyDoor~\cite{chen2023anydoor},  \( I_d \) is cropped and resized to \( \mathbf{x}^{\text{(C)}} \in \mathbb{R}^{D \times D \times 3} \), centering it around \( \text{box}_d^{\text{(C)}} \), in such a way that the projected bounding box covers at least 20\% of the area. The same viewport transformation is applied to \( \text{box}_d^{\text{(C)}} \).
Following PbE~\cite{yang2023paint}, the image \( \mathbf{x}^{\text{(C)}} \) is encoded using the pre-trained VAE~\cite{kingma2013auto} from StableDiffusion~\cite{rombach2022high}, obtaining the latent \( \mathbf{z}_0^{\text{(C)}} = \mathcal{E}^{\text{(C)}}(\mathbf{x}^{\text{(C)}}) \). Similarly, the latent representation of the camera context is computed as \( \mathbf{c}^{\text{(C)}}_{\text{env}} = \mathcal{E}^{\text{(C)}}(\mathbf{x}^{\text{(C)}} \odot \bar{\mathbf{m}}^{\text{(C)}}) \), where \( \odot \) denotes element-wise multiplication. The edit region is defined by a binary mask $ \mathbf{m}^{\text{(C)}} \in \{0, 1\}^{D \times D} $, created by inpainting $ \text{box}_d^{\text{(C)}} $ onto an initially all-zero matrix, where the inpainted region is assigned values of 1. The complement of this mask is defined as:
\[
\bar{\mathbf{m}}^{\text{(C)}} = \mathbf{J} - \mathbf{m}^{\text{(C)}}, \quad \mathbf{J} \in \{1\}^{D \times D}.
\]

\subsection{Lidar processing and encoding}
Lidar (Light Detection and Ranging) is a sensing technology that uses laser beams to measure distances to surrounding objects. A lidar sensor performs a rapid 360-degree sweep of its environment, emitting laser pulses and recording the time it takes for each pulse to return. This process generates a point cloud, a collection of 3D points that capture the scene's geometry. Each point typically includes spatial coordinates $(x, y, z)$ and an intensity value corresponding to the reflected laser signal strength.

This work considers the lidar point cloud of the destination scene, $ P_d \in \mathbb{R}^{N \times 4} $, where $N$ represents the number of points and the four channels correspond to the $x, y, z$ coordinates and intensity values. The lidar points are projected onto a cylindrical view, so-called range view, $ R_d \in \mathbb{R}^{32 \times 1096 \times 2} $ using the transformation described below. This projection is essentially lossless, except for a small set of points at the boundary of the sweep. During the lidar capture, the point cloud forms a slightly twisted, helical structure rather than a perfect cylinder, for each beam, in the x and y axes. Due to motion compensation, the sensor attempts to correct for its motion during the sweep, effectively "morphing" the helical structure into a more cylindrical shape. However, this process causes points near the end of the sweep to drift and overlap with points from the beginning. When projecting onto a cylindrical range view, points collected at the end of the sweep may spatially overlap with points from the beginning, introducing minor occlusions in the projected view.

\paragraph{Point cloud to range view transformation}

For each point in \( P_d \), the depth (Euclidean distance from the sensor) is calculated as:
\[
d_i = \sqrt{x_i^2 + y_i^2 + z_i^2}.
\]
Points with depths outside the predefined range \( [1.4\text{m}, 54\text{m}] \) are filtered out. The yaw and pitch angles are then computed as:
\[
\text{yaw}_i = -\arctan2(y_i, x_i), \quad \text{pitch}_i = \arcsin\left(\frac{z_i}{d_i}\right).
\]

The beam pitch angles \( \{\theta_k\}_{k=1}^{H} \) are chosen as \( \theta_k = 0.0232 \cdot x_k \), where \( x_k \in \{-23, -22, \ldots, 8\} \), to best match the binning of the nuScenes~\cite{caesar2020nuscenes} lidar sensor's vertical beams and its field of view. Each point is assigned to the closest vertical beam based on its pitch angle, determining its $y_i$ vertical coordinate, an integer in the range $ [0, 31] $. 

The yaw angle is mapped to the horizontal coordinate \( x \) of the range view grid as:
\[
x_i = \left\lfloor \frac{\text{yaw}_i}{\pi} \cdot \frac{W}{2} + \frac{W}{2} \right\rfloor,
\]

The final range view representation \( R_d \) of the destination scene encodes depth and intensity for each point projected onto the \( H \times W \) grid, where \( H = 32 \) denotes the number of vertical beams, and \( W = 1096 \) represents the horizontal resolution. Unassigned pixels in the range view are set to a default value. Each point is mapped to a specific pixel coordinate in the range view. 

Again, note that the transformation is not injective, as some points overlap at the start and end of the lidar sweep due to motion compensation; however, this overlap has minimal impact. Additionally, the proposed processing technique store the original pitch and yaw values for each point assigned to a range view pixel in matrices \( R_d^{\text{yaw}} \in \mathbb{R}^{H \times W} \) and \( R_d^{\text{pitch}} \in \mathbb{R}^{H \times W} \), respectively. These matrices enhance the inverse transformation from range view to point cloud by preserving the unrasterised angular information.

\paragraph{Range view to range image processing}
The bounding box $ \text{box}_d $ is projected onto $ R_d $ using the coordinate-to-range transformation, resulting in $ \text{box}_d^{\text{(R)}} \in \mathbb{R}^{8 \times 3} $, while preserving the depth of each bounding box point.
To enhance the region of interest, a zoom-in strategy is employed, analogous to that used in the image processing, by cropping the range view width-wise around $ \text{box}_d^{\text{(R)}} $, resulting in a $ {32 \times W^{\text{(R)}} \times 2} $ object-centric range view, and resizing it to obtain the range image $ \mathbf{x}^{\text{(R)}} \in \mathbb{R}^{D \times D \times 2} $. The same viewport transformation is applied to the bounding box $\text{box}_d^{\text{(R)}}$.
The edit region is defined by a mask $ \mathbf{m}^{\text{(R)}} \in \{0, 1\}^{D \times D} $, which is created by inpainting the bounding box $\text{box}_d$ onto an initially all-zero matrix, where the inpainted region has values of 1. The complement of this mask is: 
\[
\bar{\mathbf{m}}^{\text{(R)}} = \left(\mathbf{J} - \mathbf{m}^{\text{(R)}}\right).
\]

\paragraph{Range image encoding}
This work adapts the pre-trained image VAE~\cite{kingma2013auto} of StableDiffusion~\cite{rombach2022high} to the lidar modality through a series of training-free adaptations and a fine-tuning step, ablated in \autoref{tab:lidar reconstruction}.

As a naïve solution to encode the lidar modality, the preprocessed range view $\mathbf{x}^{\text{(R)}} \in \mathbb{R}^{D \times D \times 2}$ is considered, duplicates the depth channel, and passes the resulting 3-channel representation through the image VAE~\cite{kingma2013auto}. After discarding one depth channel and resizing back to $32 \times W^{\text{(R)}} \times 2$ using nearest neighbour interpolation, it computes reconstruction errors using the lidar reconstruction metrics described in~\autoref{sec:suppl:method:lidar processing}. This approach results in unsatisfactory reconstruction errors.

To address this, this work proposes three cumulative adaptations that improve depth and intensity reconstruction for object points and the extended edit mask. First, it leverages the higher resolution of $\mathbf{x}^{\text{(R)}}$ by applying average pooling when downsizing, which serves as an error correction mechanism.

Next, it is observed that the reconstruction error of range pixel values is proportional to the interval size of their distribution. Since intensity values follow an exponential distribution, intensity values $i \in [0, 255]$ are normalised using the cumulative distribution function (CDF) of the exponential distribution, choosing $\lambda = 4$ experimentally:
\[
i' = 2e^{-\lambda\frac{i}{255}} - 1 \in [-1, 1]
\]

To enhance object-level depth reconstruction, depth normalisation is applied based on the minimum and maximum depth of $\text{box}_d^{\text{(R)}}$, which extends the interval in which the object depth values are distributed and, in turn, improves object reconstruction error:

\[
d' = 
\begin{cases} 
    -\alpha + 2\alpha \cdot \frac{d - \text{min}_d}{\text{max}_d - \text{min}_d} & \text{if } \text{min}_d \leq d \leq \text{max}_d \\
    -1 + (-(\alpha - 1)) \cdot \frac{d + 1}{\text{min}_d + 1} & \text{if } -1 \leq d < \text{min}_d \\
    \alpha + (1 - \alpha) \cdot \frac{d - \text{max}_d}{1 - \text{max}_d} & \text{if } \text{max}_d < d \leq 1 
\end{cases}
\]

where $d$ is the depth value, $\alpha$ controls range scaling, and $\text{min}_d$, $\text{max}_d$ define normalisation boundaries within $[-1, 1]$, see \autoref{fig:range_scaling}. Depth values are originally between $[1.4, 54]$, but are linearly normalised to $[-1, 1]$.

\begin{figure}
    \centering
    \includegraphics[width=0.7\linewidth]{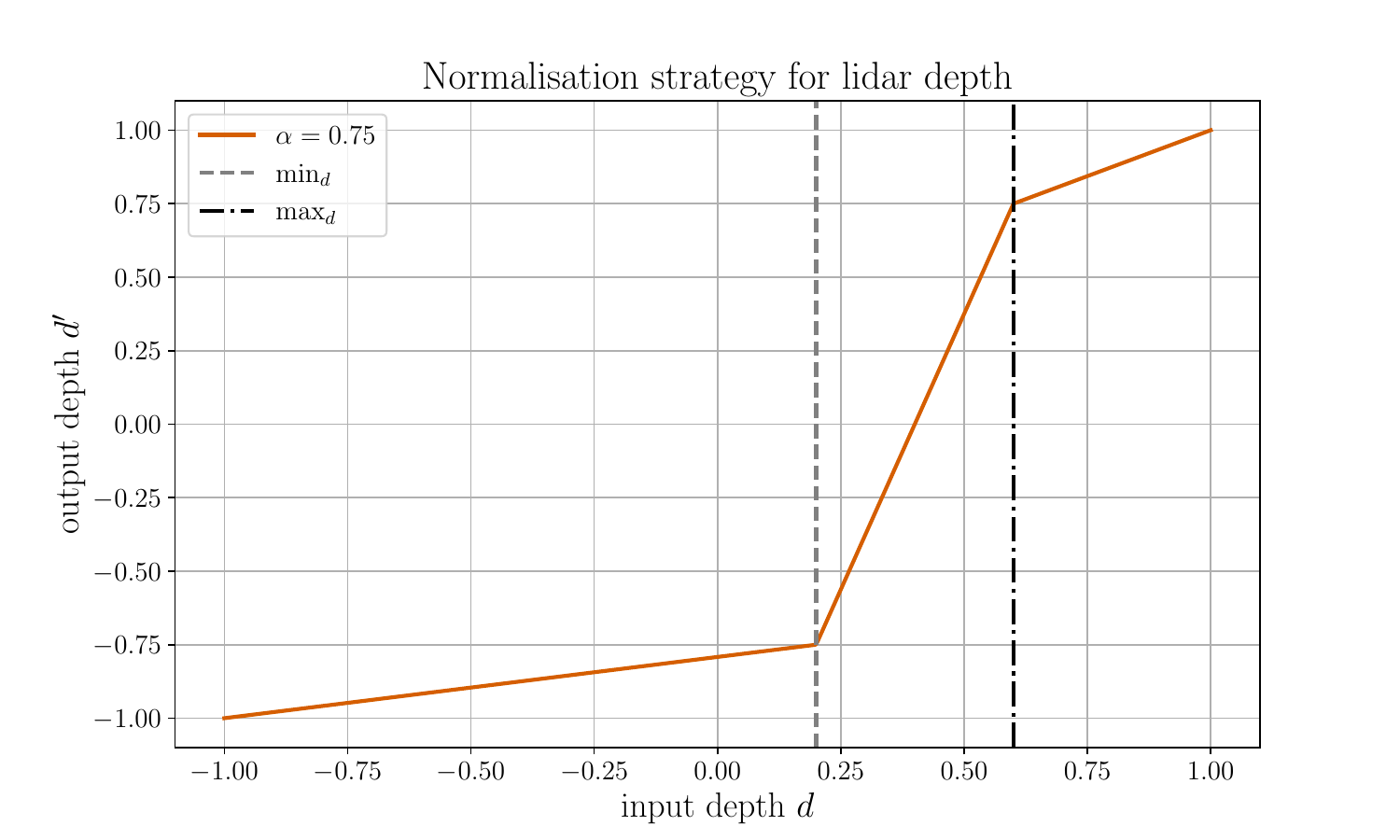}
    \caption[Normalisation strategy of the lidar depth.]{Normalisation strategy of the lidar depth, which influences the interval size allocated to the depth values of the object bounding box.}
    \label{fig:range_scaling}
\end{figure}

Thirdly, the input and output convolutions of the pre-trained image encoder and decoder are replaced with two residual blocks~\cite{he2016deep}, respectively. There are now two input and output channels. This work fine-tunes the VAE~\cite{kingma2013auto} with an additional discriminant~\cite{esser2021taming}. The same normalisation and resizing strategies are applied, yielding the best reconstruction metrics for $\mathbf{\tilde{x}}^{(R)} = \text{resize}(\mathcal{D}^{\text{(R)}}(\mathcal{E}^{\text{(R)}}(\text{norm}({\mathbf{x}^{\text{(R)}}}))))$.

Thus, this work adapts the pre-trained image VAE~\cite{kingma2013auto} of StableDiffusion~\cite{rombach2022high} to the lidar modality through a series of adaptations—improved downsampling, intensity and depth normalisation, and fine-tuning of input and output adaptation layers—to achieve better object reconstruction. These findings are demonstrated in~\autoref{tab:lidar reconstruction}.

Finally, the range image $ \mathbf{x}^{\text{(R)}} $ is encoded to obtain a latent representation $ \mathbf{z}_0^{\text{(R)}} = \mathcal{E}^{\text{(R)}}(\text{norm}({\mathbf{x}^{\text{(R)}}})) $.
Similarly, the range context $ \mathbf{x}^{\text{(R)}} \odot \bar{\mathbf{m}}^{\text{(R)}} $ is encoded to obtain a latent conditioning representation $ \mathbf{c}^{\text{(R)}}_{\text{env}} = \mathcal{E}^{\text{(R)}}(\text{norm}(\mathbf{x}^{\text{(R)}} \odot \bar{\mathbf{m}}^{\text{(R)}})) $.

\subsection{Conditioning encoding}

\paragraph{Reference extraction and encoding}
This work extracts the reference image \( \mathbf{x}_{\text{ref}} \) from the source image \( I_s \) by cropping the minimal 2D bounding box that encompasses \( \text{box}_s^{\text{(C)}} \), capturing the object's features. During inference, the reference image can be obtained from external sources. 
Following PbE~\cite{yang2023paint}, the reference image \( \mathbf{x}_{\text{ref}} \) is encoded using CLIP~\cite{radford2021learning}, selecting the classification token and passing it through a Multi-Layer Perceptron (MLP). These components initialised from PbE~\cite{yang2023paint}, are kept frozen during the training of the proposed method. While CLIP effectively preserves high-level details such as gestures or car models, it lacks fine-detail preservation. For applications requiring finer details, self-supervised pretrained encoders like DINOv2~\cite{oquab2023dinov2} may be preferable, as demonstrated in~\cite{chen2023anydoor}. This is further illustrated in~\autoref{chapter:anydoor}, where encoding references of medical anomalies requires fine detail preservation and granularity in feature extraction.

\paragraph*{Bounding box encoding}
This work considers the projected bounding boxes $ \text{box}_d^{\text{(C)}} \in \mathbb{R}^{8 \times 2} $ and $ \text{box}_d^{\text{(R)}} \in \mathbb{R}^{8 \times 3} $. The box $\text{box}_d^{\text{(C)}} $ captures the $ (x, y) $ coordinates in the camera view, scaled by the image dimensions; note some points may lie outside the image. The depth dimension from $ \text{box}_d^{\text{(R)}} $ is incorporated into $ \text{box}_d^{\text{(C)}} $ to aid with spatial consistency across modalities, resulting in $ \widetilde{\text{box}}_d^{\text{(C)}} \in \mathbb{R}^{8 \times 3} $.
These bounding boxes are encoded into conditioning tokens $ \mathbf{c}_{\text{box}}^{\text{(C)}} $ and $ \mathbf{c}_{\text{box}}^{\text{(R)}} $ using Fourier embeddings, similar to MagicDrive~\cite{gao2023magicdrive}, and modality-agnostic trainable linear layers:
\begin{align*}
  \mathbf{c}_{\text{box}}^{\text{(M)}} = \text{MLP}_{\text{box}}(\text{Fourier}(\widetilde{\text{box}}_d^{\text{(M)}})), \quad \text{for~} \text{M} \in \{\text{C}, \text{R}\}.
\end{align*}

Fourier embeddings map each coordinate value into a higher-dimensional space using sinusoidal functions (sine and cosine) at multiple frequencies. Specifically, for an input \( x \), the embedding includes terms of the form \( \sin(\omega_k x) \) and \( \cos(\omega_k x) \) for different frequencies \( \omega_k \). This allows the model to capture fine and coarse spatial patterns, facilitating the encoding of coordinates through the multilayer perception.

\subsection{Multimodal generation}
\label{sec:method:multimodal generation}
This work fine-tunes a single latent diffusion model for both modalities, leveraging the pre-trained weights of PbE~\cite{yang2023paint}.
Similar to the adaptation strategy of Flamingo~\cite{alayrac2022flamingo}, separate gated cross-attention layers are interwoven: a modality-agnostic bounding box adapter and modality-dependent cross-modal attention. The use of such layers is a commonly used strategy for methods in scene generation~\cite{gao2023magicdrive, xie2024x}, coupled with zero-initialised gating such as in ControlNet~\cite{zhang2023controlnet}.

\paragraph*{Cross-modal attention}
This method introduces a modality-dependent cross-modal attention mechanism which attends to the tokens of the other modality from the same scene in the batch.
The query, key, and value representations are derived from the input camera and lidar features for the cross-attention mechanism, from camera to lidar. Using learnable transformations \( W_Q^{\text{(C)}}, W_K^{\text{(R)}}, W_V^{\text{(R)}} \), the cross-attention is computed as:
\[
\text{Attn}^{\text{(C)}} = \text{softmax}\left(\frac{Q^{\text{(C)}} (K^{\text{(R)}})^T}{\sqrt{d_{\text{head}}}}\right) V^{\text{(R)}},
\]
where \( Q^{\text{(C)}} = W_Q^{\text{(C)}} \mathbf{h}^{\text{(C)}} \), \( K^{\text{(R)}} = W_K^{\text{(R)}} \mathbf{h}^{\text{(R)}} \), and \( V^{\text{(R)}} = W_V^{\text{(R)}} \mathbf{h}^{\text{(R)}} \). The camera features are updated by adding a residual connection through a zero-initialised gating module:
\(
\mathbf{h}^{\text{(C)}} \leftarrow \mathbf{h}^{\text{(C)}} + \text{Gate}^{\text{(C)}}(\text{Attn}^{\text{(C)}}).
\)

The zero-initialised gating module plays a crucial role in how the network is trained, particularly in the early stages of fine-tuning. At the start, when the gating module is initialised with zeros and the rest of the cross-attention matrices, randomly, the module acts as an identity function, meaning that it does not influence the camera features \( \mathbf{h}^{\text{(C)}} \) during the initial phase. This identity property allows the pretrained model (before fine-tuning) to maintain its learned knowledge without interference from new, randomly initialised parameters. The pretrained weights, which were trained on different tasks or datasets, are preserved, and no significant changes are made during the initial forward pass.

During fine-tuning, however, the gradients propagate through the gating module, and through gradient descent, the module gradually steers the model toward focusing on the new task-specific information as it adjusts its weights. The gated cross-attention thus enables the network to progressively learn task-specific features without sacrificing the performance of the pretrained model, facilitating efficient fine-tuning.

Finally, the computation for lidar-to-camera cross-attention is analogous, with lidar features attending to the camera modality. The cross-modal attention is not restricted and lets the network learn an implicit correspondence, which is facilitated by the respective projected bounding boxes. Lastly, the camera and lidar tokens are concatenated within the batch.

\paragraph*{Bounding box adapter}
The bounding box adapter is a modality-agnostic layer designed to provide bounding box conditioning while preserving reference features encoded in $\mathbf{c}_{\text{ref}}$. This adapter employs the same gating mechanism as the cross-attention module. Still, instead, it is conditioned on one of the bounding box tokens \( \mathbf{c}_{\text{box}}^{\text{(R)}} \) or \( \mathbf{c}_{\text{box}}^{\text{(C)}} \), depending on the modality, and the reference token \( \mathbf{c}_{\text{ref}} \). This enables flexible conditioning across modalities, ensuring that spatial information from the bounding box is effectively integrated alongside the reference features. Classifier-free guidance~\cite{ho2022classifier} with a scale of 5 is employed as in PbE~\cite{yang2023paint}, extending it to both reference and bounding box conditioning.

\subsection{Inference and compositing}
\label{sec:method:spatial compositing}
\paragraph*{Inference process}
At inference, the method starts from random noise \( \mathbf{\epsilon} \sim \mathcal{N}(0, \mathbf{I}) \) combined with the latent scene context and resized edit mask, and iteratively denoises this input for \( T = 50 \) steps using the DDIM sampler~\cite{song2020denoising}, conditioned on the reference \( \mathbf{c}_{\text{ref}} \) and 3D bounding box token \( \mathbf{c}_{\text{box}} \), to yield the final latent representations \( \{ \tilde{\mathbf{z}}_0^{(\text{C})}, \tilde{\mathbf{z}}_0^{(\text{R})}\} \). These latent representations are then decoded by the image and range decoders to produce the edited camera and range images \( \tilde{\mathbf{x}}^{(\text{C})} = \mathcal{D}^{(\text{C})}(\tilde{\mathbf{z}}_0^{(\text{C})}) \) and \( \tilde{\mathbf{x}}^{(\text{R})} = \mathcal{D}^{(\text{R})}(\tilde{\mathbf{z}}_0^{(\text{R})}) \).  Inference throughput is about 8 camera+lidar samples per minute on a single 80GB NVIDIA A100 GPU.

\paragraph{Range view to point cloud transformation}

To reconstruct the point cloud from the range view, the stored unrasterised pitch and yaw matrices, \( R_d^{\text{pitch}} \in \mathbb{R}^{H \times W} \) and \( R_d^{\text{yaw}} \in \mathbb{R}^{H \times W} \), are used, which preserve the original angular information for each pixel.

The depth values \( R_d^{\text{depth}} \in \mathbb{R}^{H \times W} \) are flattened to the vector \( \mathbf{d} \in \mathbb{R}^N \), where \( N = H \times W \). Similarly, the pitch and yaw matrices are flattened to the vectors \( \boldsymbol{\theta} \in \mathbb{R}^N \) and \( \boldsymbol{\phi} \in \mathbb{R}^N \), representing the pitch and yaw angles for each pixel in the range view. Using these angular and depth values, the point cloud \( P_d \in \mathbb{R}^{N \times 3} \) is reconstructed as:
\begin{align*}
\mathbf{p}_x &= \mathbf{d} \cdot \cos(\boldsymbol{\phi}) \cdot \cos(\boldsymbol{\theta}) \\
\mathbf{p}_y &= -\mathbf{d} \cdot \sin(\boldsymbol{\phi}) \cdot \cos(\boldsymbol{\theta}) \\
\mathbf{p}_z &= \mathbf{d} \cdot \sin(\boldsymbol{\theta}),
\end{align*}
where \( \mathbf{p}_x, \mathbf{p}_y, \mathbf{p}_z \in \mathbb{R}^N \) are the vectors of reconstructed \( x \), \( y \), and \( z \) coordinates, respectively. The reconstructed point cloud \( P_d \) is then given by stacking these coordinate vectors as \( P_d = [\mathbf{p}_x, \mathbf{p}_y, \mathbf{p}_z] \).

By leveraging the stored pitch and yaw matrices, the process accurately restores the point cloud while avoiding misalignments introduced by motion compensation. This ensures that the reconstructed point cloud aligns perfectly with the original input, except for the overlapping points previously mentioned, which are not reconstructed.

\paragraph*{Spatial compositing}
Final results are obtained by compositing the edited camera and range images back into the original scene.
For images, the region within the projected bounding box from the edited image $\tilde{\mathbf{x}}^{(\text{C})}$ is extracted and inserted back into the destination image \( I_d \). Following the approach of POC~\cite{de2024placing}, a Gaussian kernel is applied to improve blending, resulting in the final composited image.
For lidar, a 2D mask \( \mathbf{m}_{\text{points}} \) is created by selecting points from the original lidar point cloud \( P_d \) that fall within the destination 3D bounding box. The edited range image \( \mathbf{\tilde{x}}^{\text{(R)}} \) is then resized to an object-centric range view using average pooling and denormalised before computing coordinate and intensity values using the range view to point cloud transformation described above. Pixels in the original range view \( R_d \) are then replaced with the corresponding pixels from the edited range image if either (i) they fall within \( \mathbf{m}_{\text{points}} \) or (ii) its corresponding 3D point in the edited range image is contained by the bounding box of the object.

\subsection{Training details}
\label{sec:method:training details}

\paragraph{Sample selection}
This work considers objects from the nuScenes dataset~\cite{caesar2020nuscenes} train split with at least 64 lidar points, whose 2D bounding box is at least $100 \times 100$ pixels, with a 2D IoU overlap not exceeding 50\%, and current camera visibility of at least 70\%.
Unless stated otherwise, the proposed model is trained on ``car'' and ``pedestrian'' categories, dynamically sampling 4096 new actors per class each epoch.
During training, once an object is selected, the current scene is used as the destination, from which the 3D bounding box, environmental context, and ground truth insertion are extracted.

\paragraph{Reference selection}
Object references are taken from the same object at a different timestamp picked randomly as follows.
References for the current object are collected across all frames that meet the previous criteria to ensure good visibility and arranged by normalised temporal distance $\Delta t$, where $1$ represents the furthest reference in time and $0$ represents the current one. 
References are randomly sampled based on a beta distribution $\Delta t \sim \text{Beta}(4, 1)$, which ensures a preference for instances of the object that are far away from the current timestamp.
Thus, rather than reinserting objects into the scene using the same reference, this work utilises the temporal structure of the nuScenes dataset~\cite{caesar2020nuscenes} for augmentation. Thus, references for the current object are sampled from a different timestamp following the distribution shown in \autoref{fig:supply:beta_distribution}.

\paragraph{Augmentation}
During training, the reference image undergoes augmentations similar to those described in PbE~\cite{yang2023paint}, such as random flip, rotation, blurring and brightness and contrast transformations.
Additionally, empty bounding boxes are randomly sampled (i.e., containing no objects), overriding both the reference image and bounding box with zero values.
This encourages the model to infer and reconstruct missing details based on the surrounding context alone. Further details are provided in \autoref{sec:suppl:method}.

\paragraph{Range image reconstruction metrics}
\label{sec:suppl:method:lidar processing}
An important step towards achieving realistic lidar inpainting is ensuring the range autoencoder can reconstruct the input point cloud with high fidelity. Since the point cloud to range view transformation is lossless, the evaluation focused the quality of reconstructed range views.
The evaluation is restricted to the region within the edit mask $ \mathbf{m}^{\text{(R)}} $ and the object points from the target range view, selected using the 3D bounding box.
For each input range view $ \mathbf{X}^{\text{(R)}} $ 
and its reconstruction, $ \mathcal{D}^{\text{(R)}}(\mathcal{E}^{\text{(R)}}(\mathbf{X}^{\text{(R)}})) $, the median depth error and the mean squared error (MSE) of the intensity values are computed, restricted on the object points and the edit mask.

\paragraph{Fine-tuning procedure}
This work proceeds by training the newly added input and output adapters of the range autoencoder while keeping the rest of the image VAE~\cite{kingma2013auto} from Stable Diffusion~\cite{rombach2022high} frozen. This training phase spans 8 epochs (15k steps) with a learning rate of \(4.5 \times 10^{-5}\), selecting the checkpoint with the lowest reconstruction loss. Note, the image VAE~\cite{kingma2013auto} from Stable Diffusion~\cite{rombach2022high} is used as the camera encoder, with no fine-tuning, due to good reconstruction performance of the RGB camera input.

During fine-tuning of the latent diffusion model, the camera autoencoder, range autoencoder and all other layers from the PbE~\cite{yang2023paint} framework remain frozen, while only the bounding box encoder, bounding box adaptation layer, and cross-modal attention layers are trained. This method uses an input dimension of \( D = 512 \) and a latent dimension of \( D_h = 64 \), training for 30 epochs (approximately 90k steps), with a constant learning rate of \(8 \times 10^{-5}\) and a batch size of 2 multimodal samples. The top five models with the lowest validation loss are retained. The final model is selected based on the best Fréchet Inception Distance (FID)~\cite{heusel2017gans} achieved on a test set of 200 pre-selected images, where objects are reinserted into scenes using the previously-described filters. Fine-tuning of the latent diffusion model takes approximately 20 hours on 8x 24GB NVIDIA A10G or 2x 80GB NVIDIA A100 GPUs.

\paragraph{Sampling empty boxes for augmentation} For augmentation purposes, empty bounding boxes are sampled to train the model to reconstruct missing details. A dedicated database of 10,000 such boxes is created. For a given scene, an object from a different scene is selected, ensuring that teleporting the bounding box into the current scene does not result in 3D overlap or a total 2D IoU overlap exceeding 50\% with other objects. During training, 30\% of the samples are drawn from this database. All-zero reference images and boxes with zero coordinates are used for these samples, enabling the model to learn how to fill in background details, as shown in \autoref{fig:suppl:erase_training_examples}.

\begin{figure*}[t]
    \centering
    % First figure: Beta distribution
    \begin{minipage}{0.48\textwidth}
        \centering
        \includegraphics[width=\textwidth]{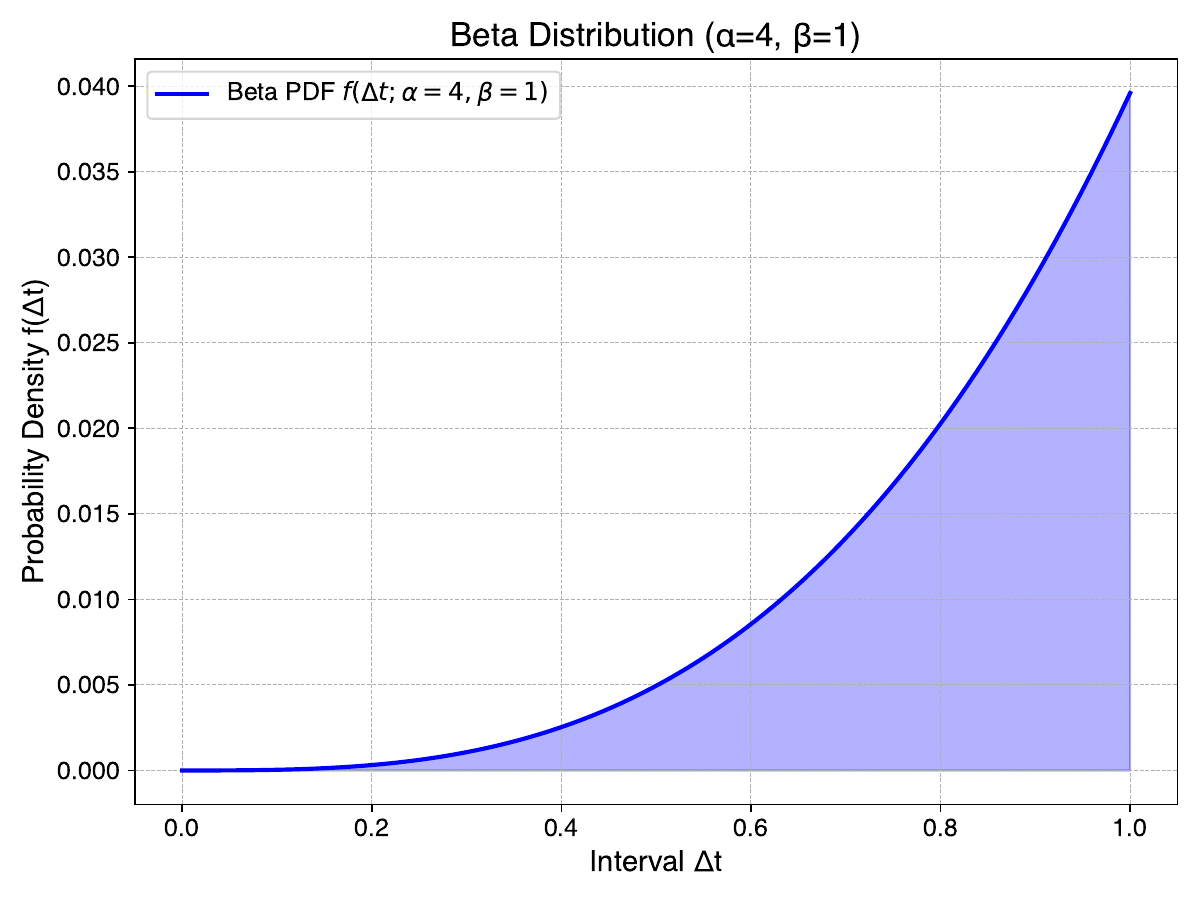}
        \caption[Beta distribution used to sample reference patches given temporal information.]{The probability density function of the Beta distribution with parameters \(\alpha=4\) and \(\beta=1\), used to sample reference patches of an object based on the normalised timestamp difference \(\Delta t\) between tracked instances. Patches from further time points are sampled with higher frequency.}
        \label{fig:supply:beta_distribution}
    \end{minipage}
    \hfill
    % Second figure: Training examples
    \begin{minipage}{0.48\textwidth}
        \centering
        \setlength{\tabcolsep}{1pt} % Adjust column spacing
        \footnotesize
        \begin{tabular}{ccc}
            {Training input} & {Empty projection} & {Training output} \\
            \includegraphics[width=0.3\columnwidth]{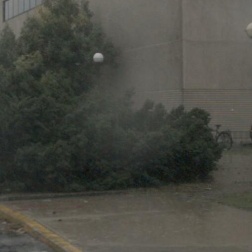} &
            \includegraphics[width=0.3\columnwidth]{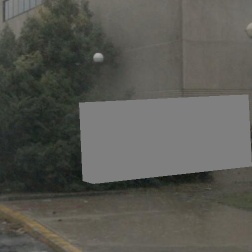} &
            \includegraphics[width=0.3\columnwidth]{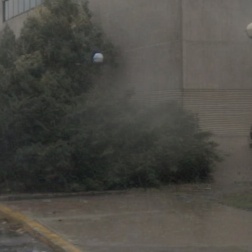} \\
            
            \includegraphics[width=0.3\columnwidth]{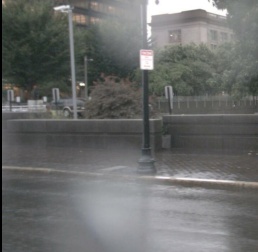} &
            \includegraphics[width=0.3\columnwidth]{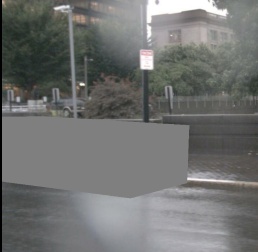} &
            \includegraphics[width=0.3\columnwidth]{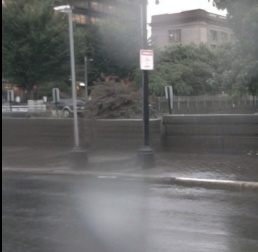} \\
            
            \includegraphics[width=0.3\columnwidth]{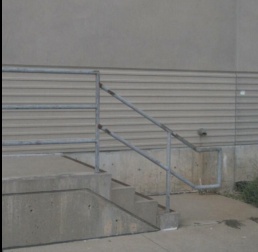} &
            \includegraphics[width=0.3\columnwidth]{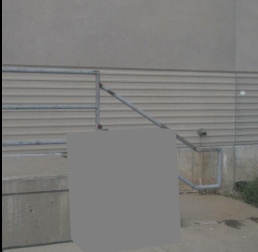} &
            \includegraphics[width=0.3\columnwidth]{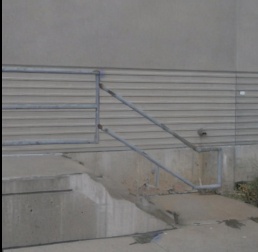} \\
            
            \includegraphics[width=0.3\columnwidth]{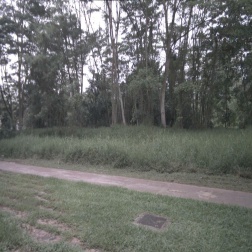} &
            \includegraphics[width=0.3\columnwidth]{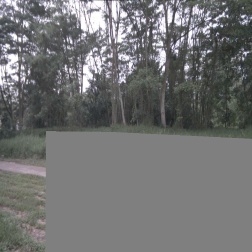} &
            \includegraphics[width=0.3\columnwidth]{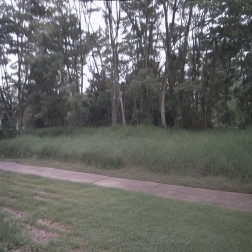} \\
        \end{tabular}
        \caption[Data augmentation with no reference.]{Empty boxes are sampled during training for data augmentation, with the reference conditioning set to a black image and the bounding box coordinates set to zero.}
        \label{fig:suppl:erase_training_examples}
    \end{minipage}
\end{figure*}

\section{Experiments and results}
\label{sec:experiments}

\subsection{Object insertion and replacement}
\label{sec:experiments:implementation}

\paragraph{Setup}
In order to avoid situations where inpainted objects are placed at locations incompatible with the scene (e.g. a car on the pavement), the position of existing objects is used and either object reinsertion or replacement is performed, which differ by the choice of the inpainting reference. By doing so, the model's ability to generate realistic objects conditioned on a 3D bounding box while being semantically consistent with the scene is tested. A total of 200 original objects are sampled from the nuScenes validation set as in~\autoref{sec:method:training details}, balanced across the ``car'' and ``pedestrian'' classes.

\paragraph{Reinsertion}
Two types of references are defined: \textit{same reference}, where the source and destination images and bounding boxes are identical, meaning the object is reinserted in the exact same scene and position; and \textit{tracked reference}, where the object is reinserted given its reference from a different timestamp, using the same sampling strategy described in \autoref{sec:method:training details}. This setting tests whether the object’s appearance can be preserved by the model, and whether novel view synthesis can be realistically performed (for \textit{tracked reference}).

\paragraph{Replacement}
Two different domains are defined based on the weather conditions ($\text{rainy}(I_s), \text{rainy}(I_d) \in \{0, 1\}$) and time of day ($\text{night}(I_s), \text{night}(I_d) \in \{0, 1\}$), and the following reference types are considered: \textit{in-domain reference}, where the source and destination bounding boxes correspond to different objects that are of the same class and same domain $(\text{rainy}(I_d) = \text{rainy}(I'_d) ~\&~ \text{night}(I_s) = \text{night}(I'_d))$, and \textit{cross-domain reference}, where the bounding boxes correspond to different objects of the same class, yet are drawn from at least a different domain $(\text{rainy}(I_d) \neq \text{rainy}(I_d) \text{ or } \text{night}(I_s) \neq \text{night}(I_d))$. Replacements are selected within the same class only to ensure that object placement and dimensions are meaningful and coherent.

\paragraph{Qualitative results}
Results are presented on \autoref{fig:results} both for replacement (rows 1--4) and insertion (row 5).
It can be seen that inpainted objects correspond tightly to their conditioning 3D bounding boxes while having a high degree of realism, both for camera (RGB) and lidar (depth and intensity), and show a strong coherence (lightning, weather conditions, occlusions, etc.) with the rest of the scene.
The last row showcases object deletion, which can be achieved by using an empty reference image (note that empty references are used during training, as described in~\autoref{sec:method:training details}).
Even though references in the replacement setting are from a different domain (time of day/weather), the model is able to inpaint such objects realistically. See \autoref{fig:inpainting-hard-suppl} for more examples, including failure cases.
Finally, the flexibility of the proposed bounding box conditioning is illustrated, and it is shown to generate multiple views with a high degree of consistency, as demonstrated in~\autoref{fig:suppl:rotation_results} and \autoref{fig:suppl:controllability_main_full}.

\begin{figure*}[ht!]
\centering
\setlength{\tabcolsep}{2pt}
\footnotesize
\begin{tabular}{ccccccc}
& \multicolumn{3}{c}{\textbf{Original}} & \multicolumn{3}{c}{\textbf{Edited}}\\
\cmidrule(r){2-4} \cmidrule(r){5-7}
Reference & Camera & Depth & Intensity & Camera & Depth & Intensity \\
\cmidrule(r){1-1} \cmidrule(r){2-7}\\[-10pt]
\includegraphics[width=0.135\linewidth]{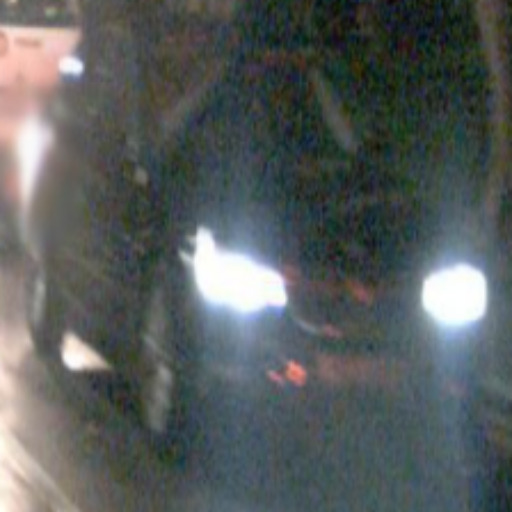} &
\includegraphics[width=0.135\linewidth]{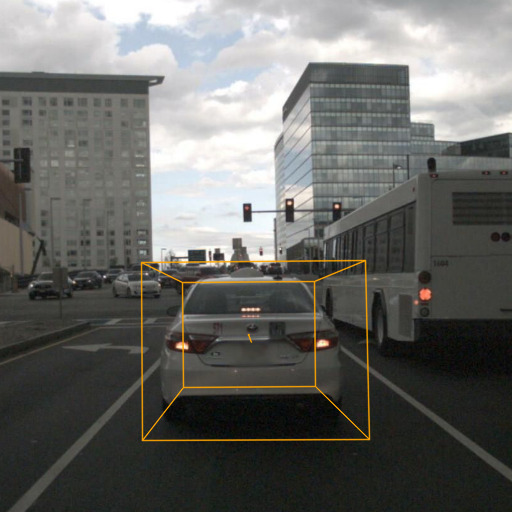} &
\includegraphics[width=0.135\linewidth]{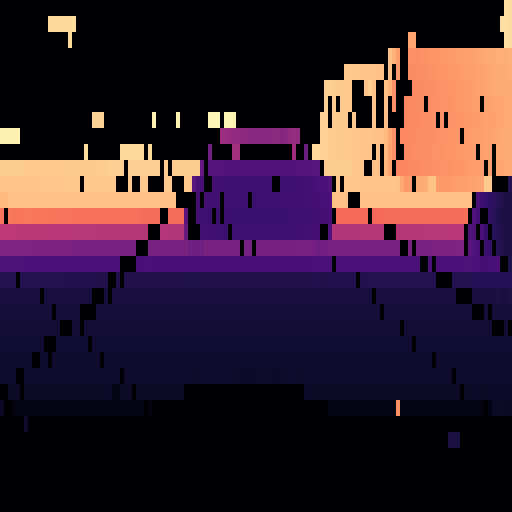} &
\includegraphics[width=0.135\linewidth]{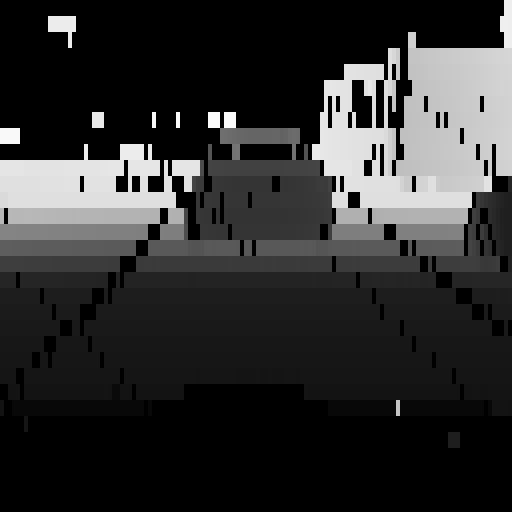} &
\includegraphics[width=0.135\linewidth]{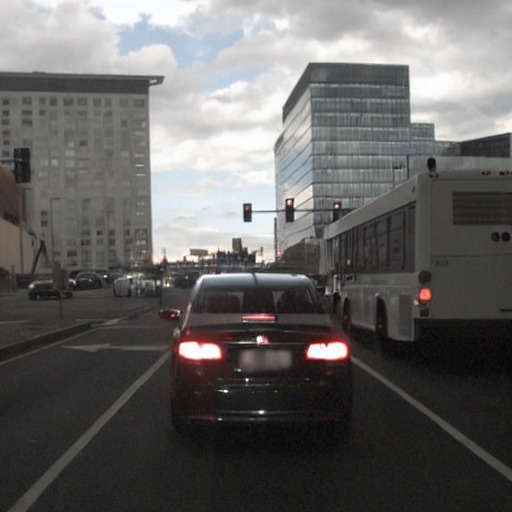} &
\includegraphics[width=0.135\linewidth]{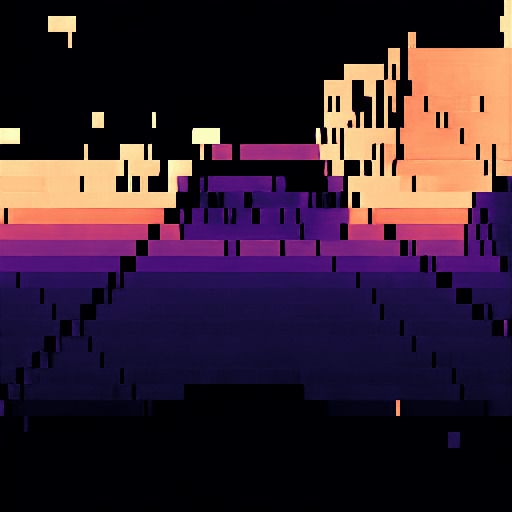} &
\includegraphics[width=0.135\linewidth]{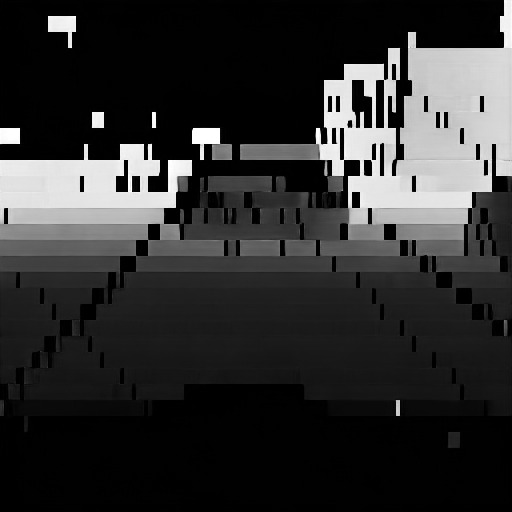} \\
\includegraphics[width=0.135\linewidth]{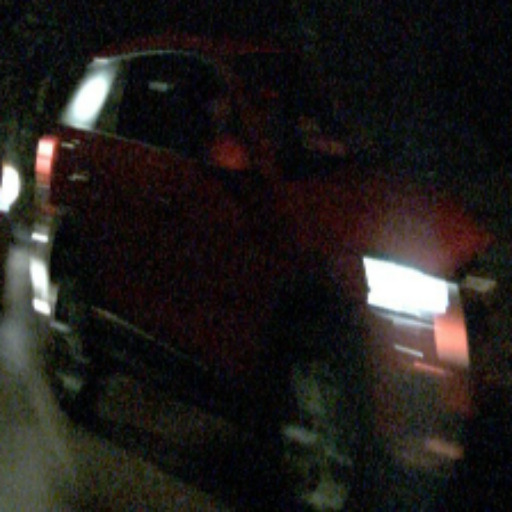} &
\includegraphics[width=0.135\linewidth]{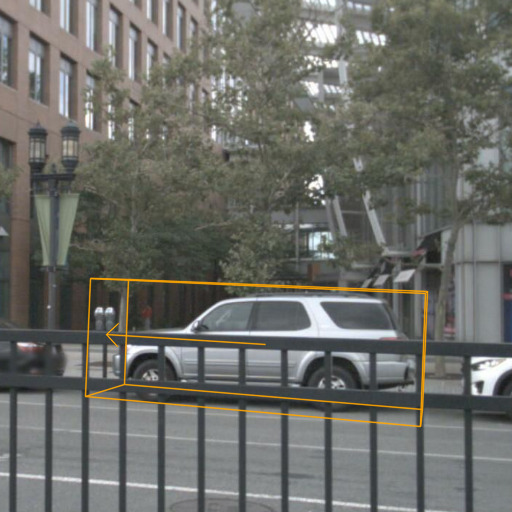} &
\includegraphics[width=0.135\linewidth]{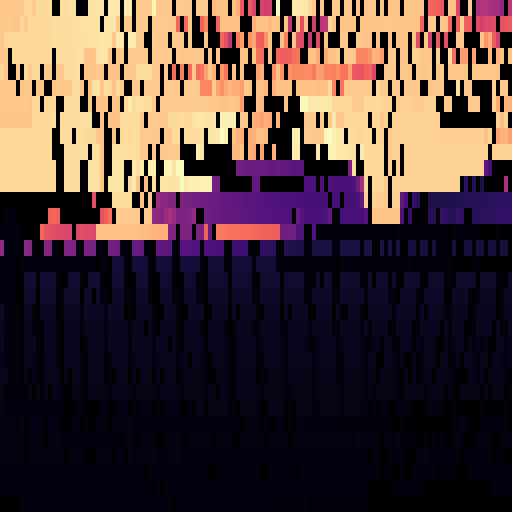} &
\includegraphics[width=0.135\linewidth]{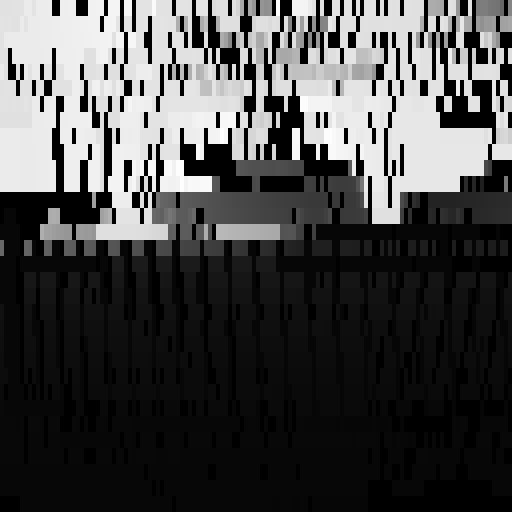} &
\includegraphics[width=0.135\linewidth]{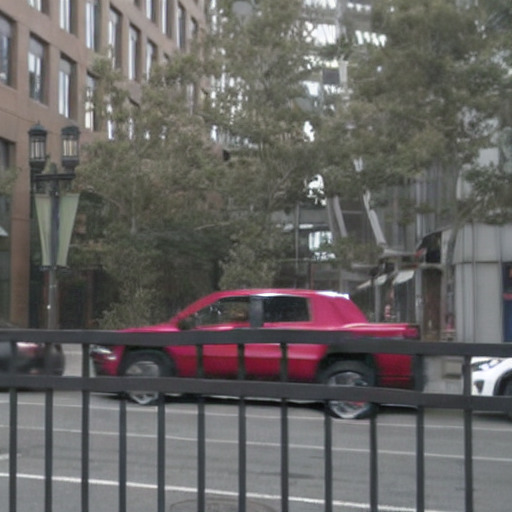} &
\includegraphics[width=0.135\linewidth]{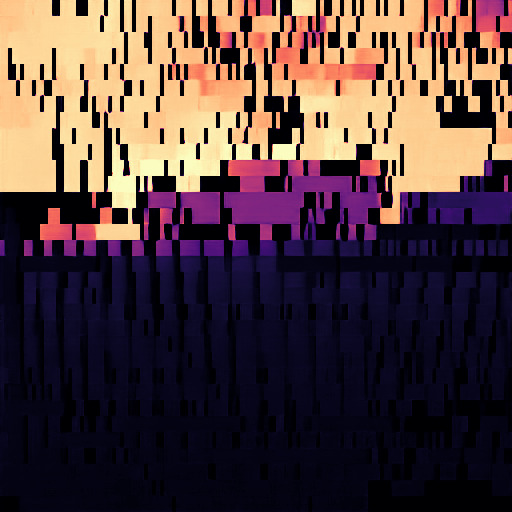} &
\includegraphics[width=0.135\linewidth]{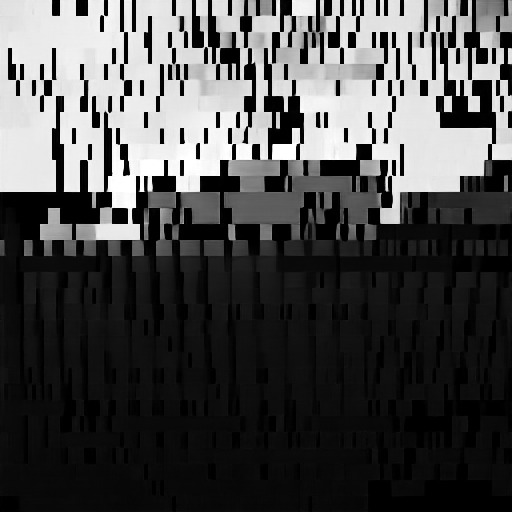} \\
\includegraphics[width=0.135\linewidth]{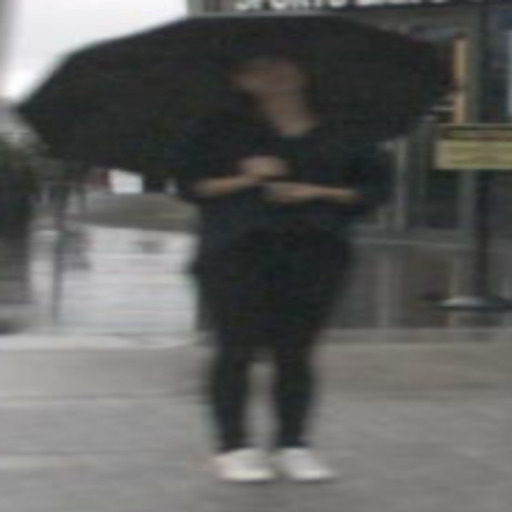} &
\includegraphics[width=0.135\linewidth]{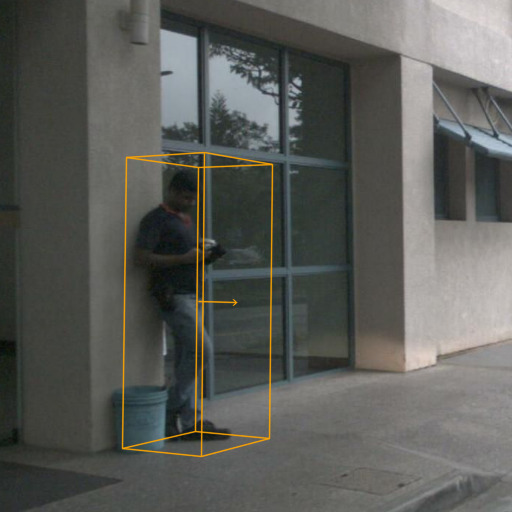} &
\includegraphics[width=0.135\linewidth]{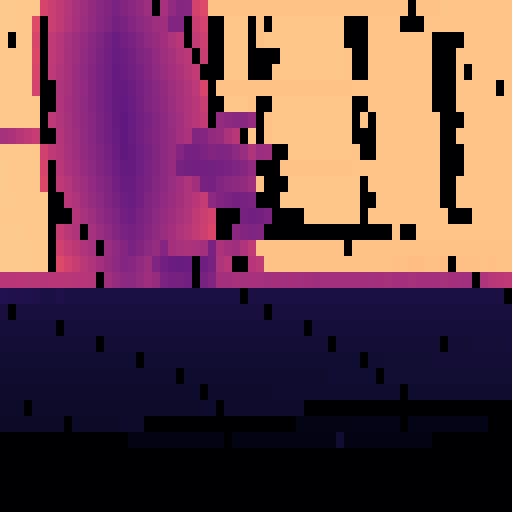} &
\includegraphics[width=0.135\linewidth]{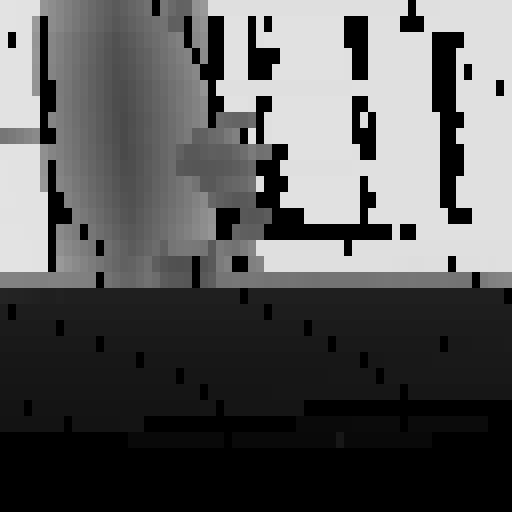} &
\includegraphics[width=0.135\linewidth]{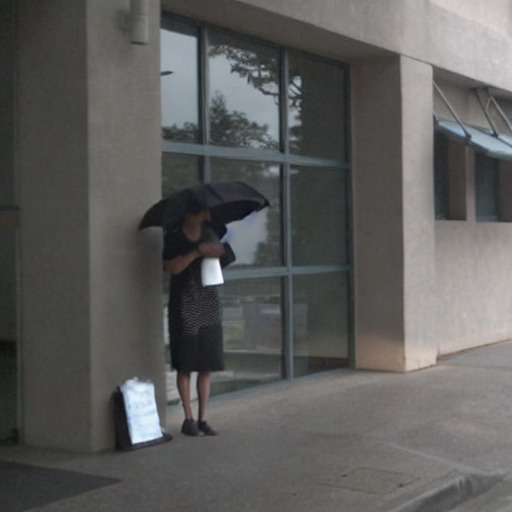} &
\includegraphics[width=0.135\linewidth]{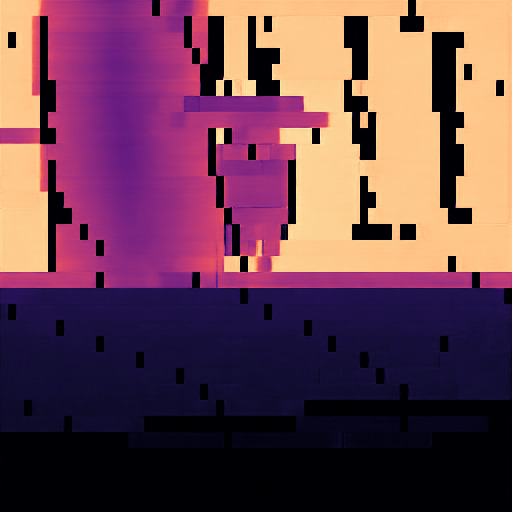} &
\includegraphics[width=0.135\linewidth]{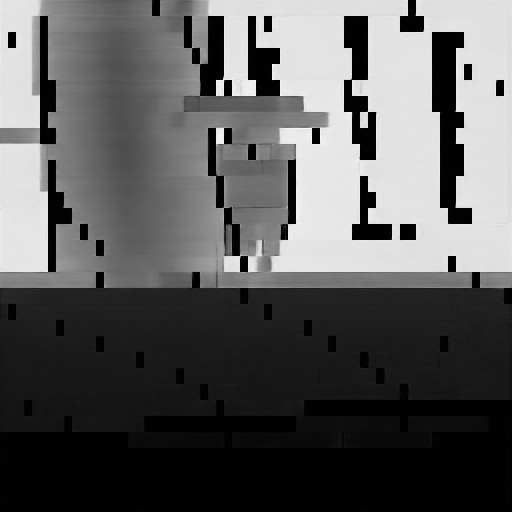} \\
\includegraphics[width=0.135\linewidth]{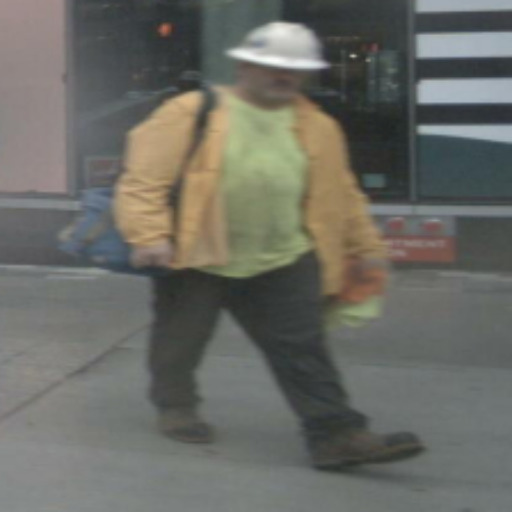} &
\includegraphics[width=0.135\linewidth]{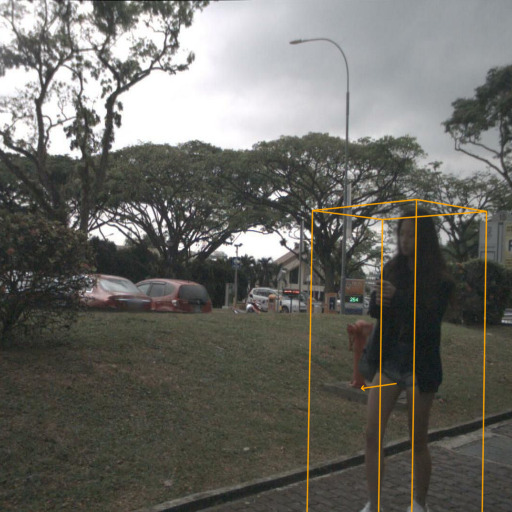} &
\includegraphics[width=0.135\linewidth]{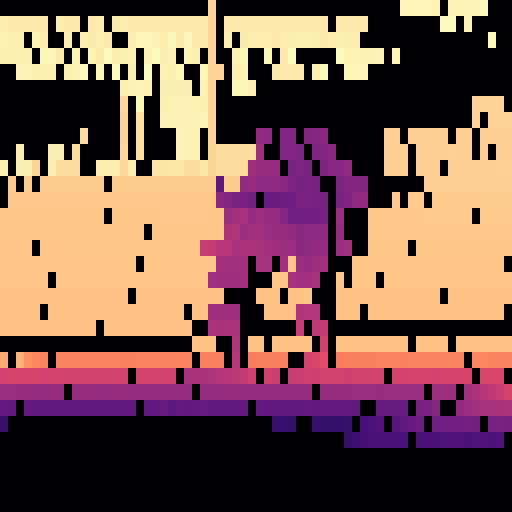} &
\includegraphics[width=0.135\linewidth]{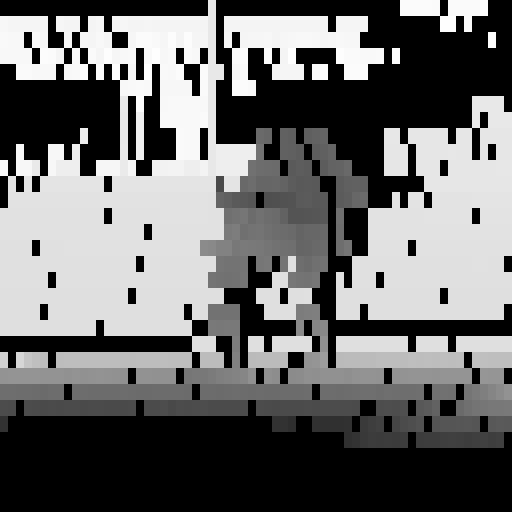} &
\includegraphics[width=0.135\linewidth]{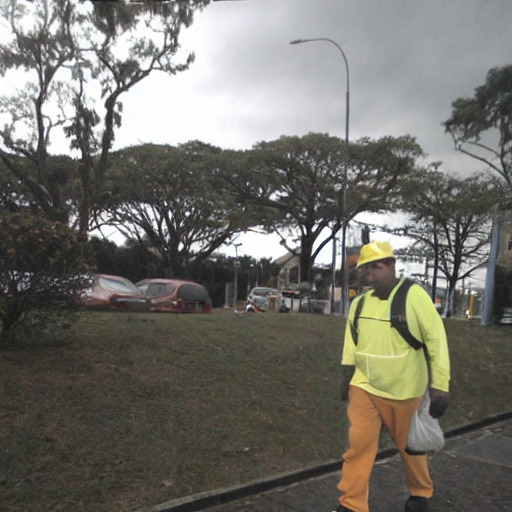} &
\includegraphics[width=0.135\linewidth]{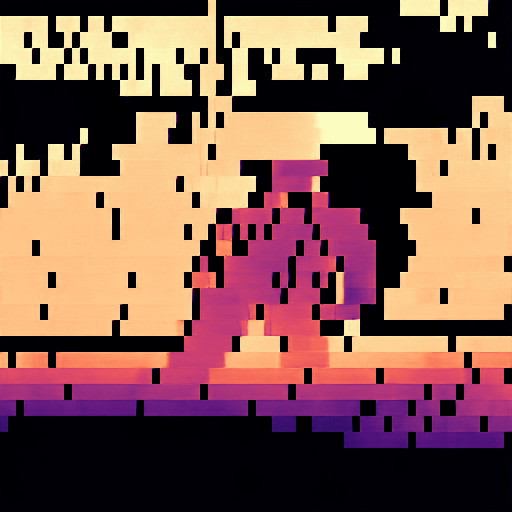} &
\includegraphics[width=0.135\linewidth]{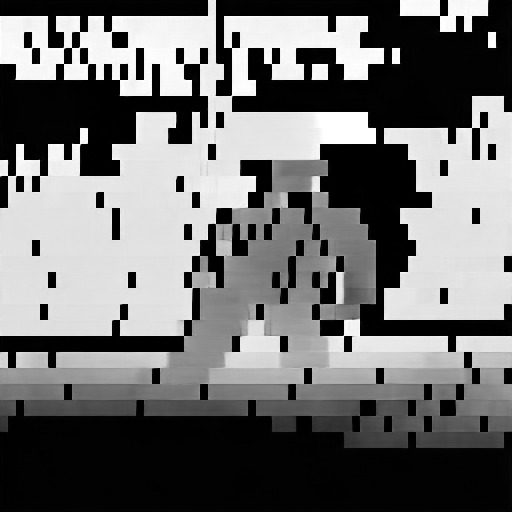} \\
\includegraphics[width=0.135\linewidth]{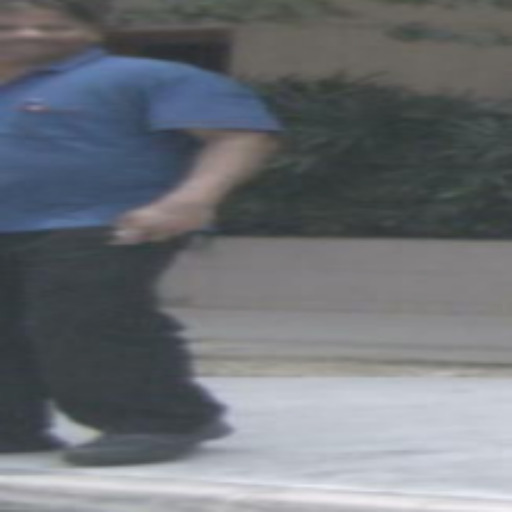} &
\includegraphics[width=0.135\linewidth]{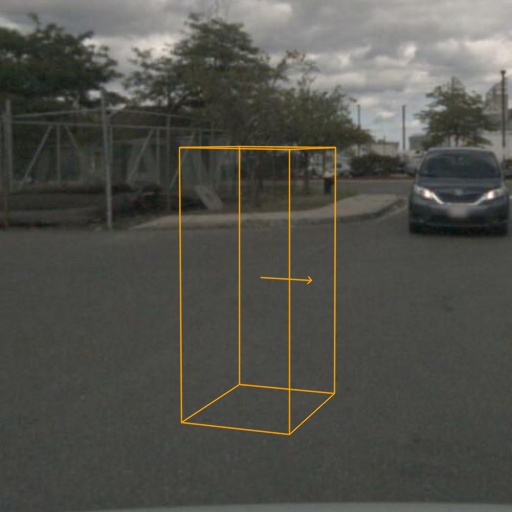} &
\includegraphics[width=0.135\linewidth]{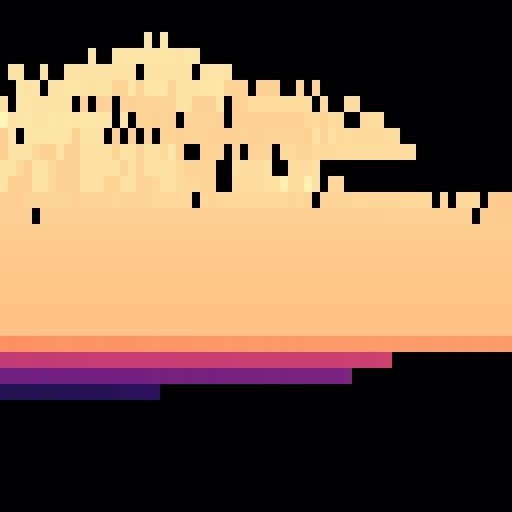} &
\includegraphics[width=0.135\linewidth]{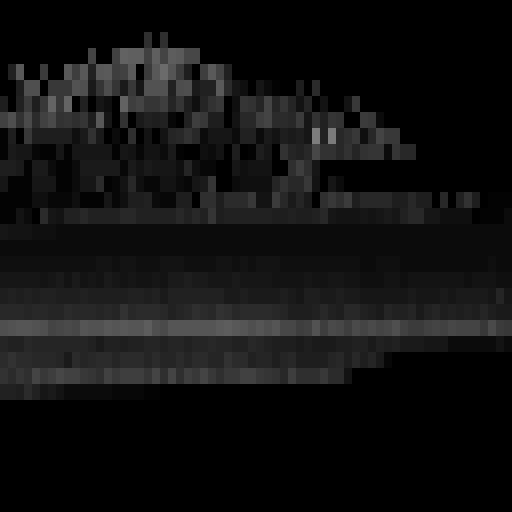} &
\includegraphics[width=0.135\linewidth]{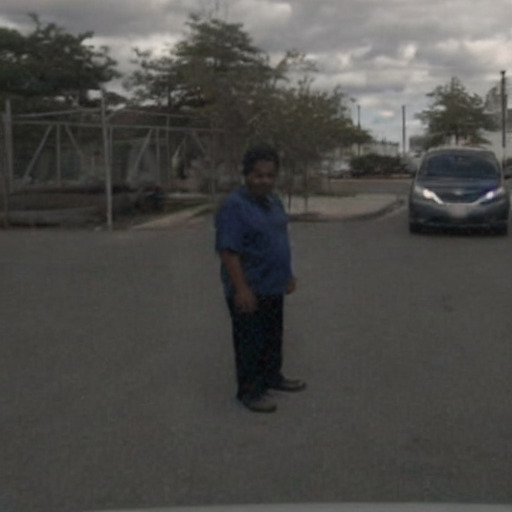} &
\includegraphics[width=0.135\linewidth]{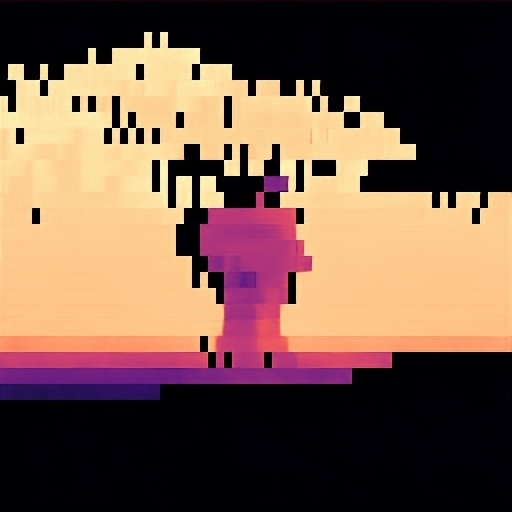} &
\includegraphics[width=0.135\linewidth]{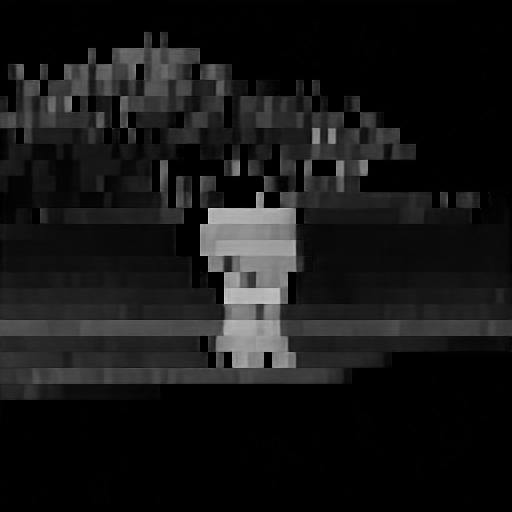} \\
\includegraphics[width=0.135\linewidth]{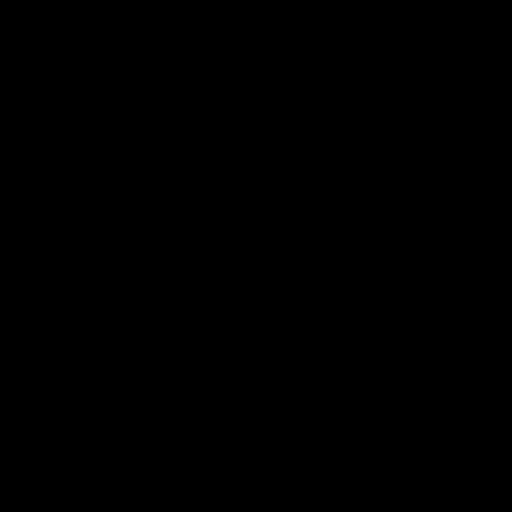} &
\includegraphics[width=0.135\linewidth]{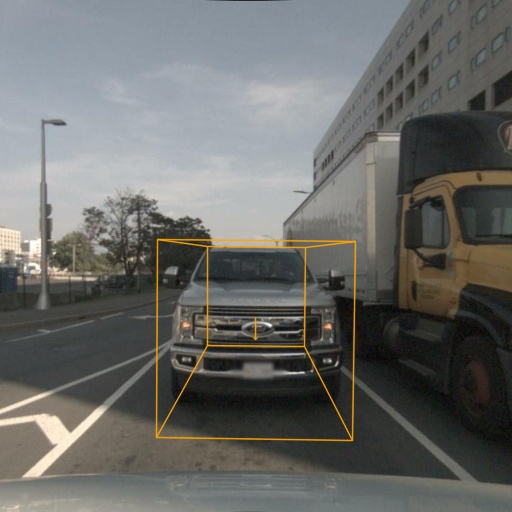} &
\includegraphics[width=0.135\linewidth]{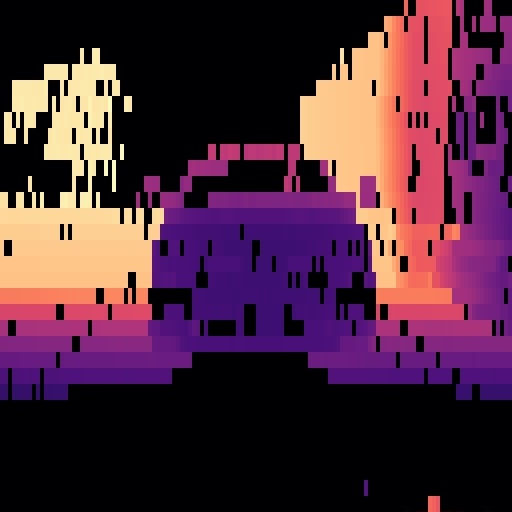} &
\includegraphics[width=0.135\linewidth]{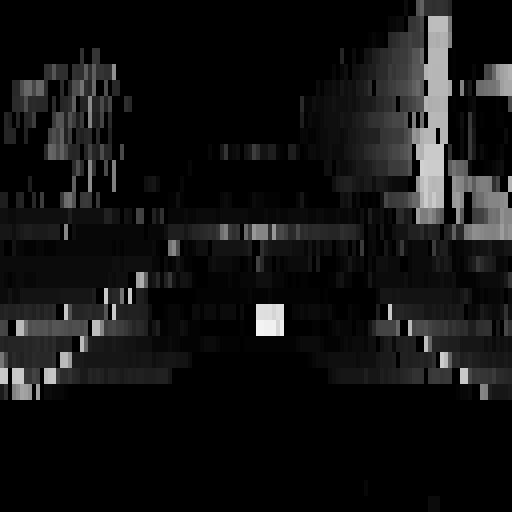} &
\includegraphics[width=0.135\linewidth]{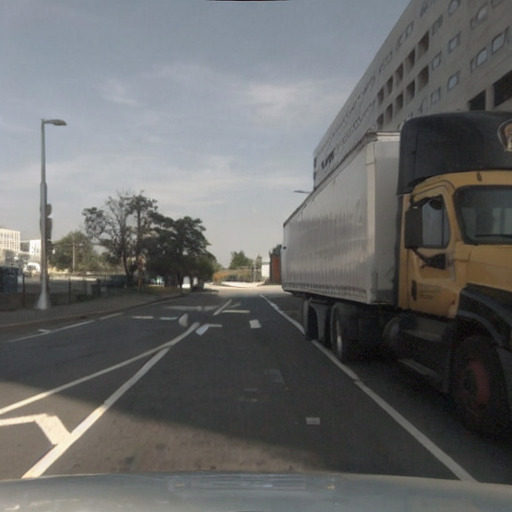} &
\includegraphics[width=0.135\linewidth]{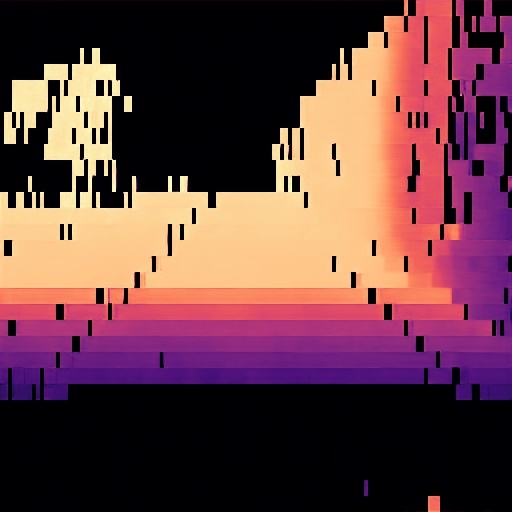} &
\includegraphics[width=0.135\linewidth]{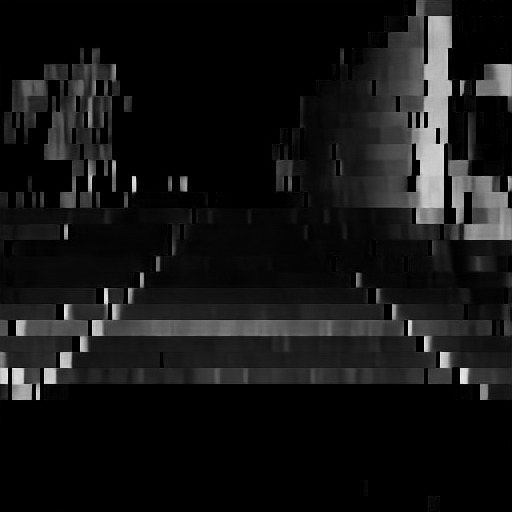} \\
\end{tabular}
\caption[Examples of using MObI for object replacement, insertion and removal.]{Examples of object inpainting using MObI in the following settings: replacement (rows 1--4), insertion (row 5), and deletion (row 6, using a black reference). The proposed method can inpaint objects corresponding to a 3D bounding box with a high degree of realism while preserving coherence with the rest of the scene.
Note that even though some references are from a different domain (time of day, weather condition), the model is able to preserve the coherence of the resulting insertion.
}
\label{fig:results}
\end{figure*}

\begin{figure*}[ht]
    \centering
    \footnotesize
    \renewcommand{\arraystretch}{0.3} % Adjust this value to control row spacing
    \begin{tabular}{c@{\hspace{1pt}}c@{\hspace{8pt}}c@{\hspace{1pt}}c@{\hspace{1pt}}c@{\hspace{1pt}}c@{\hspace{1pt}}c@{\hspace{1pt}}c@{\hspace{4pt}}l}
        {\textbf{Ref. image}} &
        {\textbf{Original scene}} &
        \multicolumn{6}{c}{\textbf{Edited scenes}} &
        \\[4pt]
        
        \includegraphics[width=0.12\linewidth]{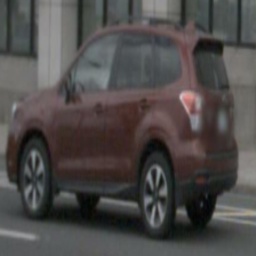} &
        \includegraphics[width=0.12\linewidth]{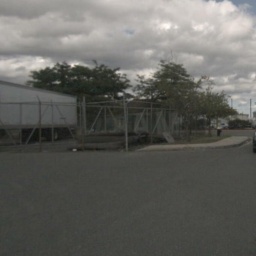} &
        \includegraphics[width=0.12\linewidth]{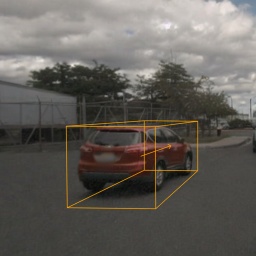} &
        \includegraphics[width=0.12\linewidth]{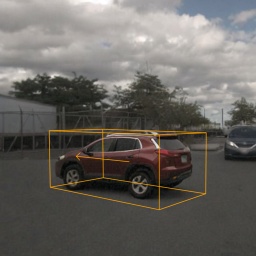} &
        \includegraphics[width=0.12\linewidth]{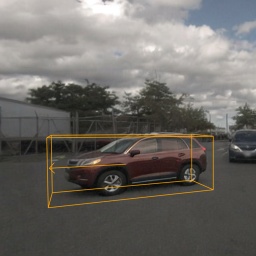} &
        \includegraphics[width=0.12\linewidth]{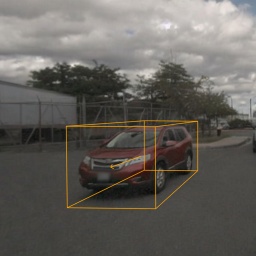} &
        \includegraphics[width=0.12\linewidth]{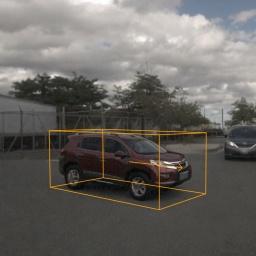} &
        \includegraphics[width=0.12\linewidth]{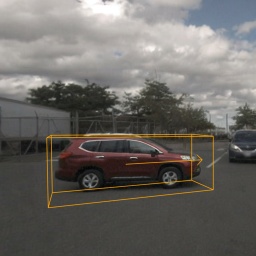} &
        \rotatebox{90}{\quad~ Camera} \\[1pt]
        
        &
        \includegraphics[width=0.12\linewidth]{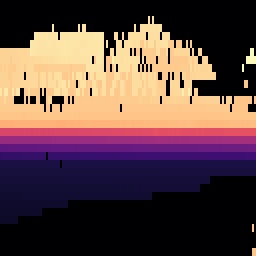} &
        \includegraphics[width=0.12\linewidth]{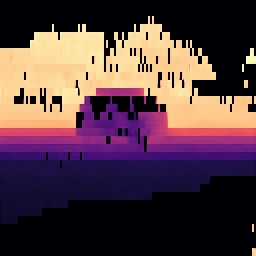} &
        \includegraphics[width=0.12\linewidth]{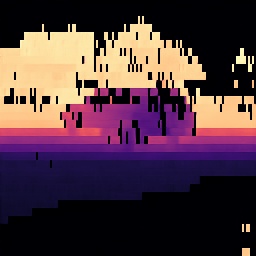} &
        \includegraphics[width=0.12\linewidth]{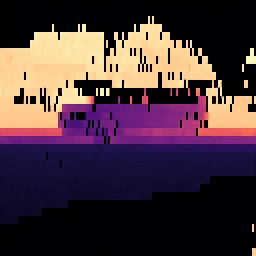} &
        \includegraphics[width=0.12\linewidth]{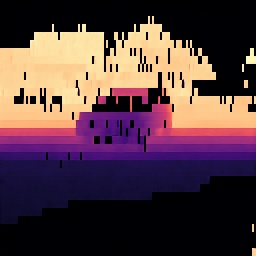} &
        \includegraphics[width=0.12\linewidth]{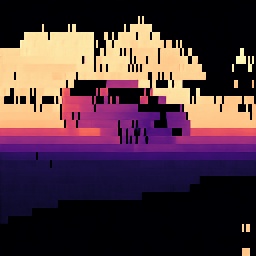} &
        \includegraphics[width=0.12\linewidth]{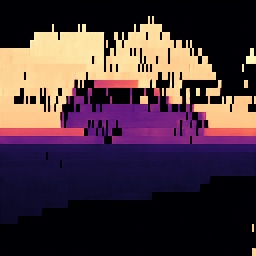} &
        \rotatebox{90}{~~ Lidar depth} \\[1pt]
        
        &
        \includegraphics[width=0.12\linewidth]{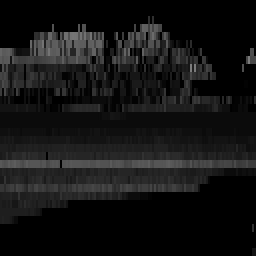} &
        \includegraphics[width=0.12\linewidth]{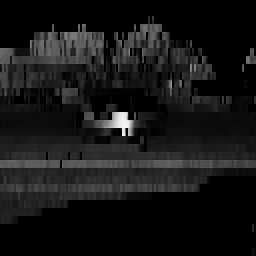} &
        \includegraphics[width=0.12\linewidth]{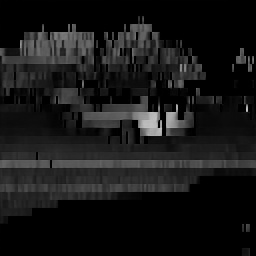} &
        \includegraphics[width=0.12\linewidth]{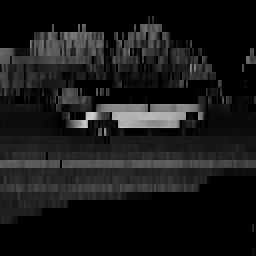} &
        \includegraphics[width=0.12\linewidth]{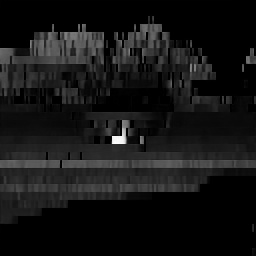} &
        \includegraphics[width=0.12\linewidth]{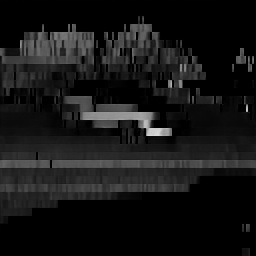} &
        \includegraphics[width=0.12\linewidth]{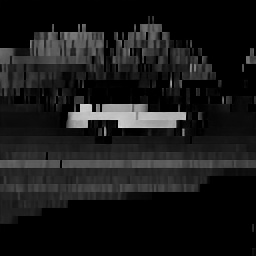} &
        \rotatebox{90}{Lidar intensity}\vspace{8pt}
    \end{tabular}
    
    \begin{tabular}{c@{\hspace{1pt}}c@{\hspace{8pt}}c@{\hspace{1pt}}c@{\hspace{1pt}}c@{\hspace{1pt}}c@{\hspace{1pt}}c@{\hspace{1pt}}c@{\hspace{4pt}}l}
        
        \includegraphics[width=0.12\linewidth]{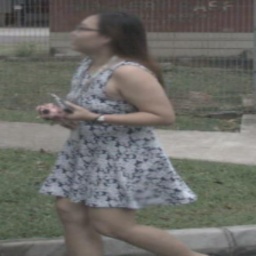} &
        \includegraphics[width=0.12\linewidth]{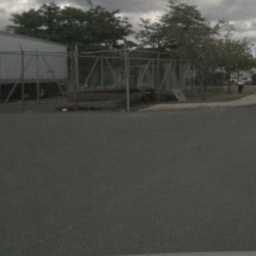} &
        \includegraphics[width=0.12\linewidth]{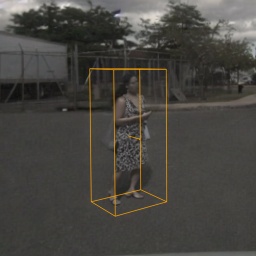} &
        \includegraphics[width=0.12\linewidth]{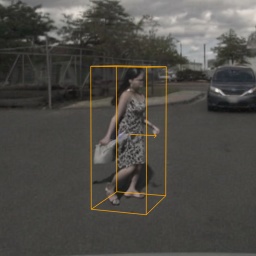} &
        \includegraphics[width=0.12\linewidth]{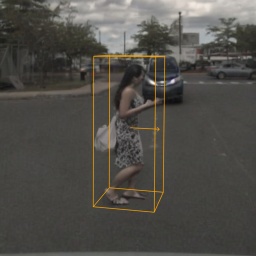} &
        \includegraphics[width=0.12\linewidth]{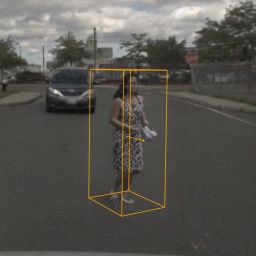} &
        \includegraphics[width=0.12\linewidth]{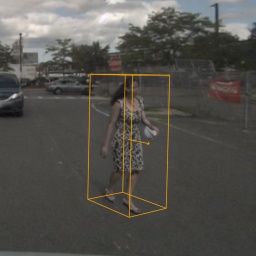} &
        \includegraphics[width=0.12\linewidth]{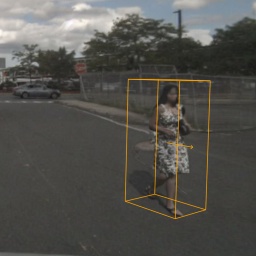} &
        \rotatebox{90}{\quad~ Camera} \\[1pt]
        
        &
        \includegraphics[width=0.12\linewidth]{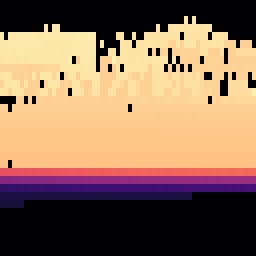} &
        \includegraphics[width=0.12\linewidth]{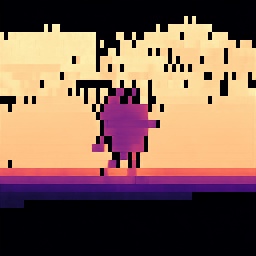} &
        \includegraphics[width=0.12\linewidth]{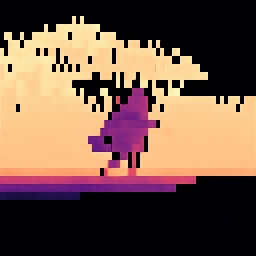} &
        \includegraphics[width=0.12\linewidth]{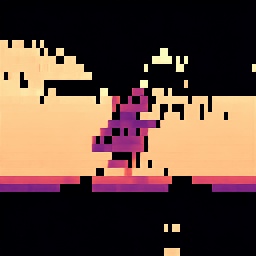} &
        \includegraphics[width=0.12\linewidth]{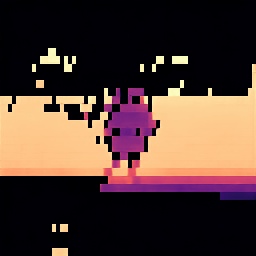} &
        \includegraphics[width=0.12\linewidth]{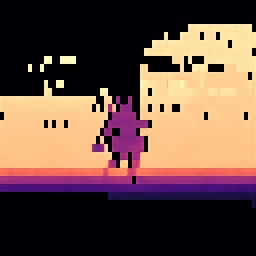} &
        \includegraphics[width=0.12\linewidth]{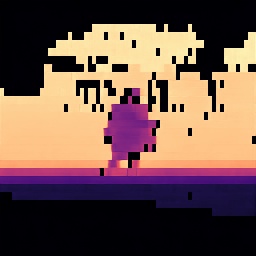} &
        \rotatebox{90}{~~ Lidar depth} \\[1pt]
        
        &
        \includegraphics[width=0.12\linewidth]{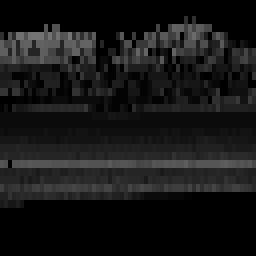} &
        \includegraphics[width=0.12\linewidth]{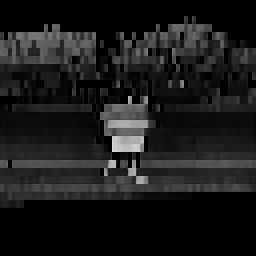} &
        \includegraphics[width=0.12\linewidth]{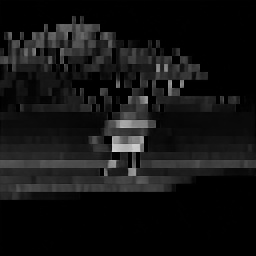} &
        \includegraphics[width=0.12\linewidth]{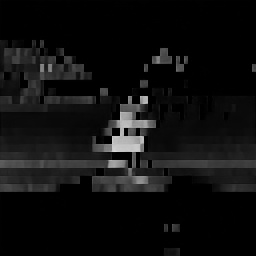} &
        \includegraphics[width=0.12\linewidth]{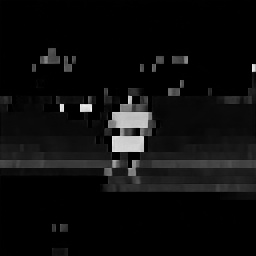} &
        \includegraphics[width=0.12\linewidth]{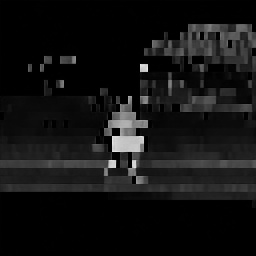} &
        \includegraphics[width=0.12\linewidth]{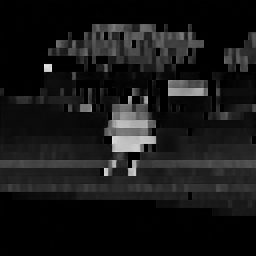} &
        \rotatebox{90}{Lidar intensity} \vspace{8pt}
    \end{tabular}
    \caption[Insertion with a rotating box showcasing the controllability of MObI.]{Examples showcasing the controllability of the proposed method. From left to right: reference image $ \mathbf{x}_{\text{ref}} $ extracted from a seperate source scene, original destination scene (original RGB image $ \mathbf{x}^{\text{(C)}} $, lidar range depth $ \mathbf{x}_0^{\text{(R)}} $ and intensity $ \mathbf{x}_1^{\text{(R)}} $), and edited scenes.}
    \label{fig:suppl:rotation_results}
  \end{figure*}

\begin{figure}[t!]
    \begin{center}
        \includegraphics[width=\linewidth, trim={0 24pt 0 66pt}, clip]{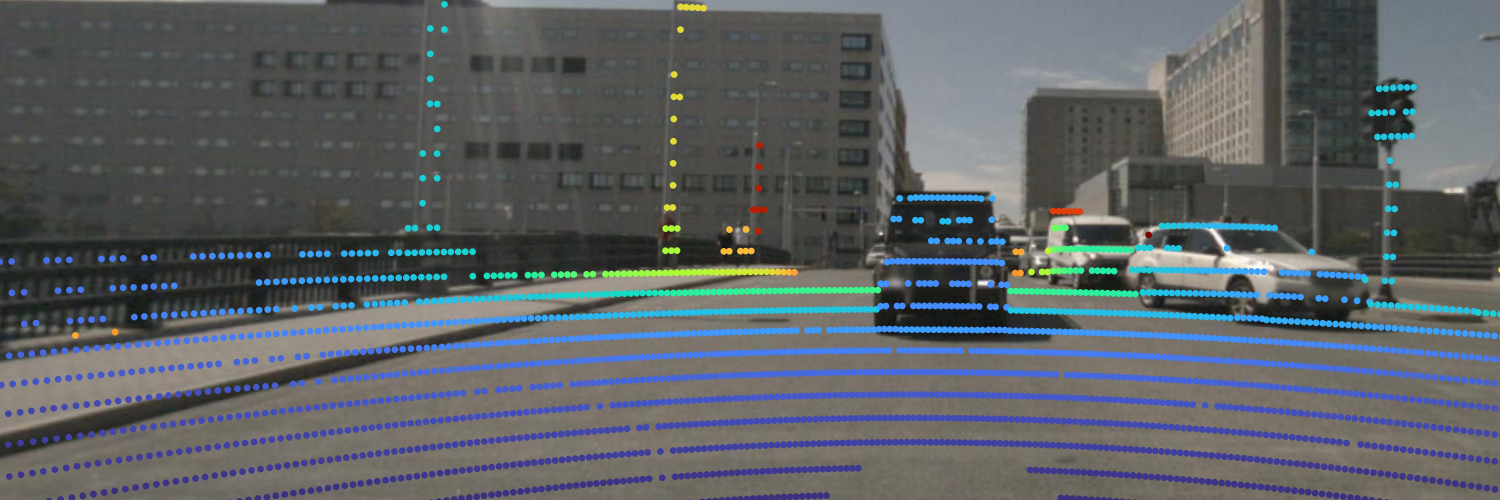}
        \includegraphics[width=\linewidth, trim={0 24pt 0 68pt}, clip]{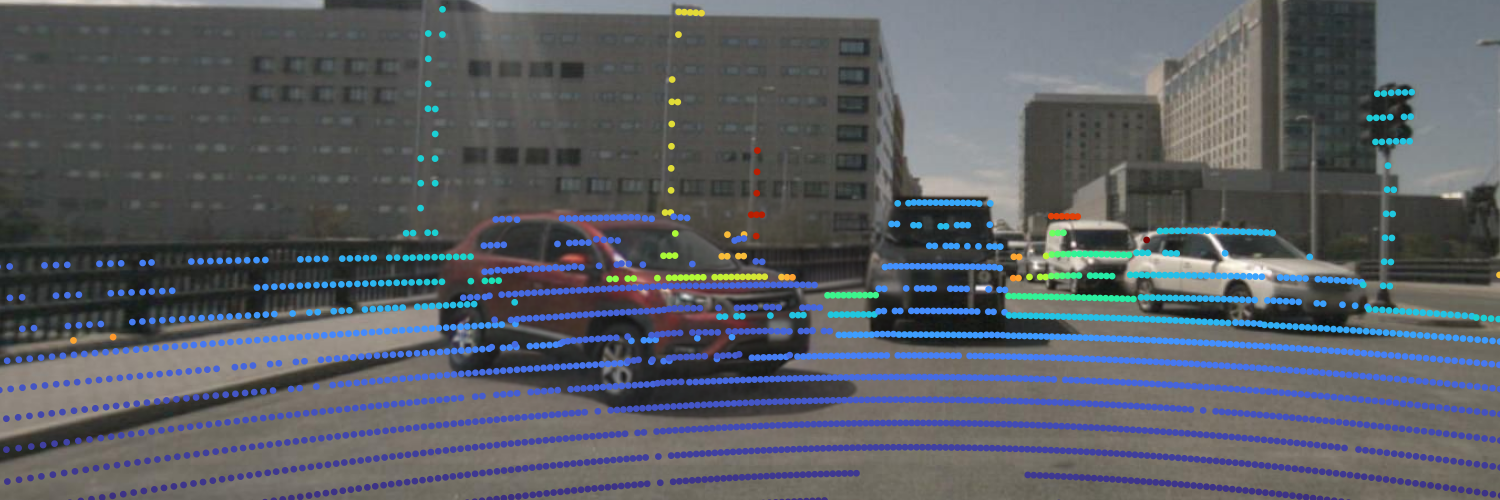}
    \end{center}
    \caption[Example of camera-lidar spatial compositing.]{Spatial compositing of camera-lidar object inpainting in a scene with complex lighting. Note that some background points are not overridden due to lidar reflections on the hood of the inserted car (bottom).}
    \label{fig:compositing}
\end{figure}

\begin{table}[h]
    \centering
    \footnotesize
    \begin{tabular}{l
                S[table-format=1.3] % to align with three decimal points
                S[table-format=1.3] 
                S[table-format=1.3] 
                S[table-format=1.3]}
        & \multicolumn{2}{c}{\textbf{Median depth error}} & \multicolumn{2}{c}{\textbf{MSE intensity}} \\
        \cmidrule(r){2-3} \cmidrule(r){4-5}
        \textbf{Lidar encoder} & {object} & {mask} & {object} & {mask} \\
        \midrule
        {pretrained image VAE~\cite{kingma2013auto}} & 0.451 & 0.320 & 7.854 & 7.372 \\
        {+ average pooling} & 0.306 & 0.263 & 3.496 & 3.236 \\
        {+ object-aware norm.} & 0.04 & 0.315 & 3.792 & 2.941 \\
        {+ fine-tune lidar adapter} & \textbf{0.037} & \textbf{0.180}  & \textbf{2.397} & \textbf{2.009} \\
        \bottomrule
        \\
    \end{tabular}
    \caption[Results with proposed techniques, improving object-centric lidar reconstruction.]{Adaptation methods of the pre-trained image VAE~\cite{kingma2013auto} from StableDiffusion~\cite{rombach2022high} showing improved lidar reconstruction for depth and intensity. Depth is reported in meters and intensity is on a scale of $ [0, 255] $.}
    \label{tab:lidar reconstruction}
\end{table}

\begin{figure*}[ht]
    \centering
    \footnotesize
    \renewcommand{\arraystretch}{0.3} % Adjust this value to control row spacing
    \begin{tabular}{c@{\hspace{1pt}}c@{\hspace{8pt}}c@{\hspace{1pt}}c@{\hspace{1pt}}c@{\hspace{1pt}}c@{\hspace{1pt}}c@{\hspace{1pt}}c@{\hspace{4pt}}l}
        {\textbf{Ref. image}} &
        {\textbf{Original scene}} &
        \multicolumn{6}{c}{\textbf{Edited scenes}} &
        \\[4pt]
        
        \includegraphics[width=0.12\linewidth]{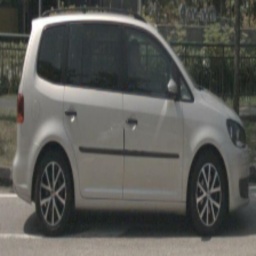} &
        \includegraphics[width=0.12\linewidth]{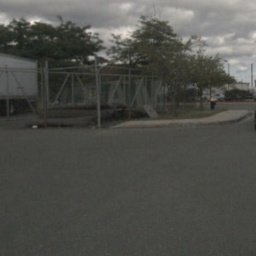} &
        \includegraphics[width=0.12\linewidth]{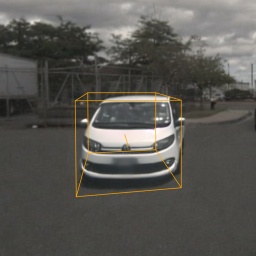} &
        \includegraphics[width=0.12\linewidth]{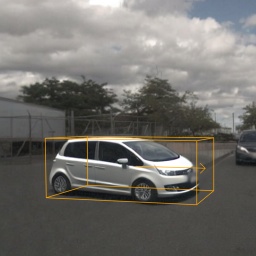} &
        \includegraphics[width=0.12\linewidth]{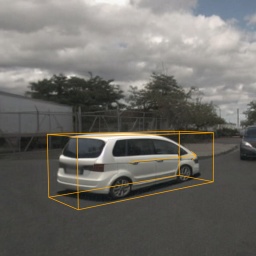} &
        \includegraphics[width=0.12\linewidth]{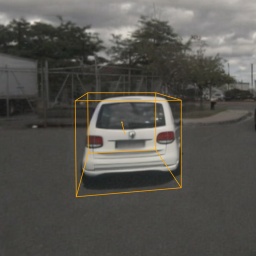} &
        \includegraphics[width=0.12\linewidth]{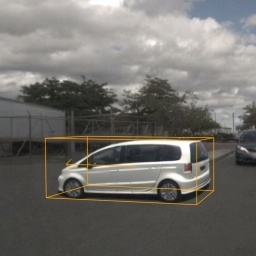} &
        \includegraphics[width=0.12\linewidth]{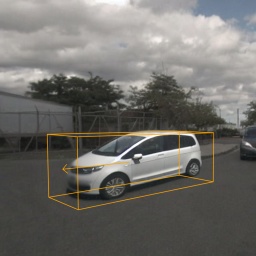} &
        \rotatebox{90}{\quad~ Camera} \\[1pt]
        
        &
        \includegraphics[width=0.12\linewidth]{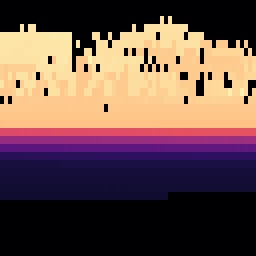} &
        \includegraphics[width=0.12\linewidth]{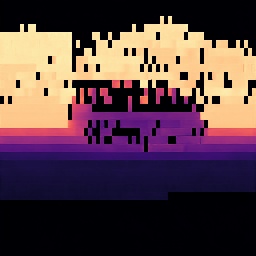} &
        \includegraphics[width=0.12\linewidth]{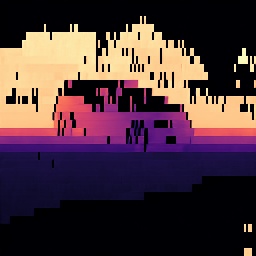} &
        \includegraphics[width=0.12\linewidth]{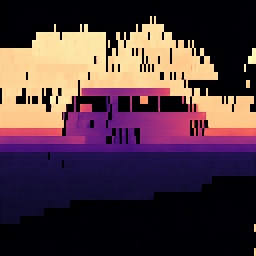} &
        \includegraphics[width=0.12\linewidth]{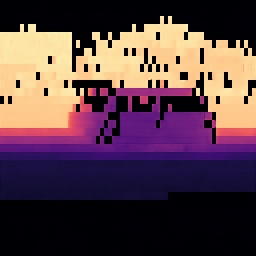} &
        \includegraphics[width=0.12\linewidth]{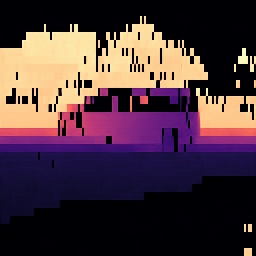} &
        \includegraphics[width=0.12\linewidth]{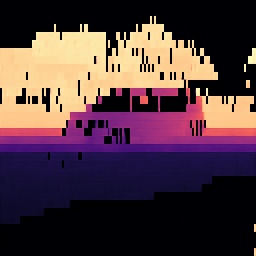} &
        \rotatebox{90}{~~ Lidar depth} \\[1pt]
        
        &
        \includegraphics[width=0.12\linewidth]{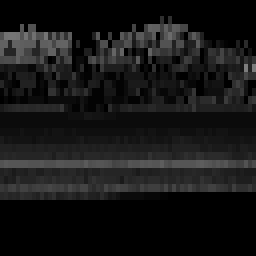} &
        \includegraphics[width=0.12\linewidth]{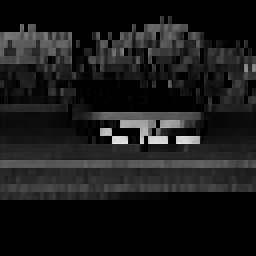} &
        \includegraphics[width=0.12\linewidth]{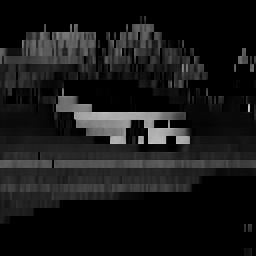} &
        \includegraphics[width=0.12\linewidth]{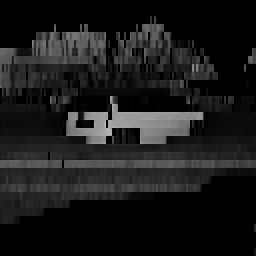} &
        \includegraphics[width=0.12\linewidth]{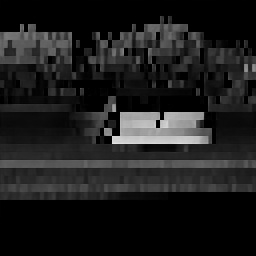} &
        \includegraphics[width=0.12\linewidth]{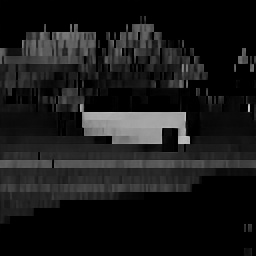} &
        \includegraphics[width=0.12\linewidth]{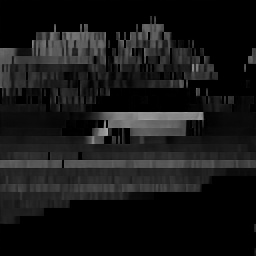} &
        \rotatebox{90}{Lidar intensity}\vspace{8pt}
    \end{tabular}
    
    \begin{tabular}{c@{\hspace{1pt}}c@{\hspace{8pt}}c@{\hspace{1pt}}c@{\hspace{1pt}}c@{\hspace{1pt}}c@{\hspace{1pt}}c@{\hspace{1pt}}c@{\hspace{4pt}}l}
        
        \includegraphics[width=0.12\linewidth]{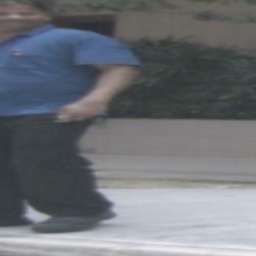} &
        \includegraphics[width=0.12\linewidth]{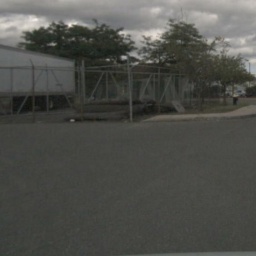} &
        \includegraphics[width=0.12\linewidth]{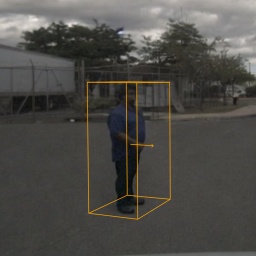} &
        \includegraphics[width=0.12\linewidth]{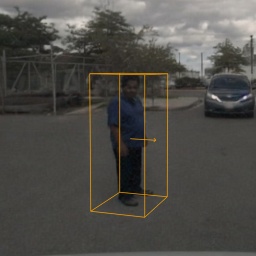} &
        \includegraphics[width=0.12\linewidth]{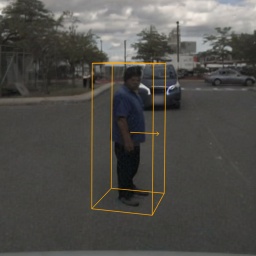} &
        \includegraphics[width=0.12\linewidth]{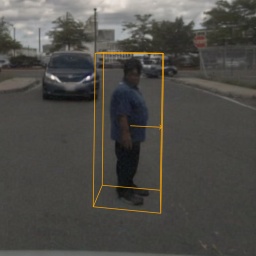} &
        \includegraphics[width=0.12\linewidth]{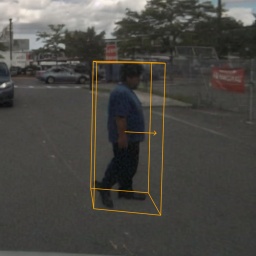} &
        \includegraphics[width=0.12\linewidth]{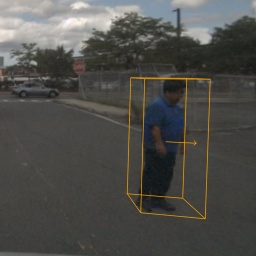} &
        \rotatebox{90}{\quad~ Camera} \\[1pt]
        
        &
        \includegraphics[width=0.12\linewidth]{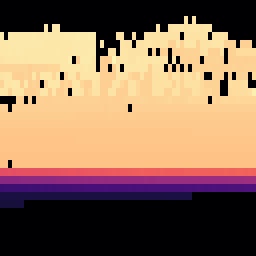} &
        \includegraphics[width=0.12\linewidth]{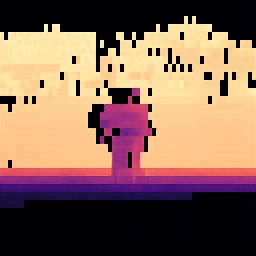} &
        \includegraphics[width=0.12\linewidth]{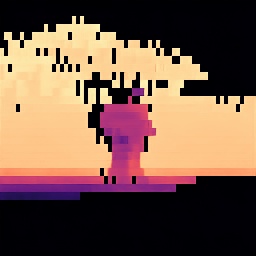} &
        \includegraphics[width=0.12\linewidth]{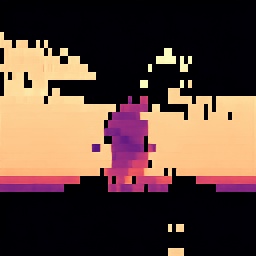} &
        \includegraphics[width=0.12\linewidth]{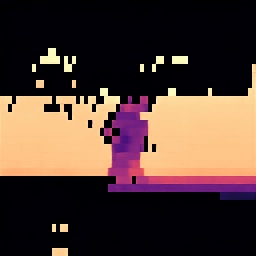} &
        \includegraphics[width=0.12\linewidth]{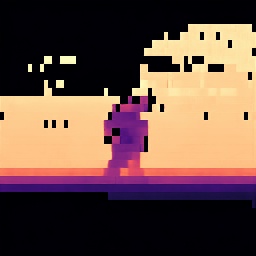} &
        \includegraphics[width=0.12\linewidth]{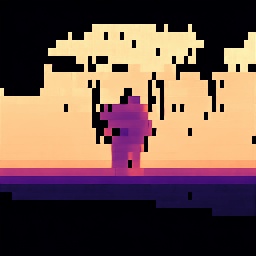} &
        \rotatebox{90}{~~ Lidar depth} \\[1pt]
        
        &
        \includegraphics[width=0.12\linewidth]{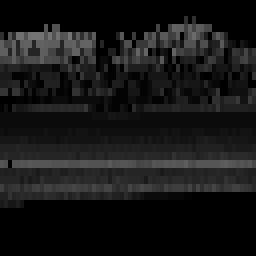} &
        \includegraphics[width=0.12\linewidth]{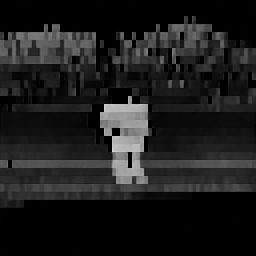} &
        \includegraphics[width=0.12\linewidth]{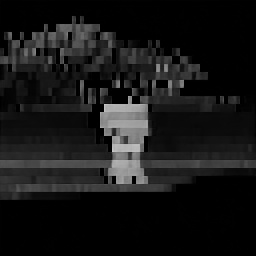} &
        \includegraphics[width=0.12\linewidth]{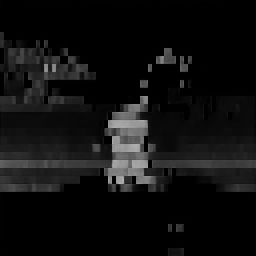} &
        \includegraphics[width=0.12\linewidth]{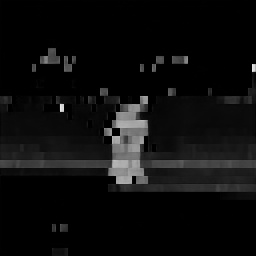} &
        \includegraphics[width=0.12\linewidth]{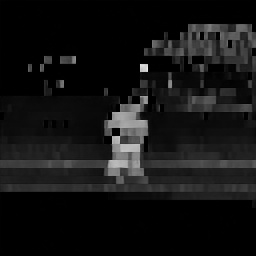} &
        \includegraphics[width=0.12\linewidth]{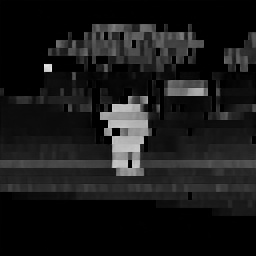} &
        \rotatebox{90}{Lidar intensity} \vspace{8pt}
    \end{tabular}
    \begin{tabular}{c@{\hspace{1pt}}c@{\hspace{8pt}}c@{\hspace{1pt}}c@{\hspace{1pt}}c@{\hspace{1pt}}c@{\hspace{1pt}}c@{\hspace{1pt}}c@{\hspace{4pt}}l}
        
        \includegraphics[width=0.12\linewidth]{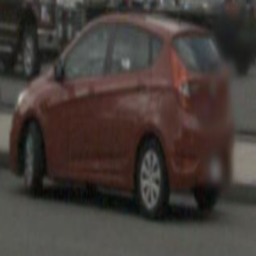} &
        \includegraphics[width=0.12\linewidth]{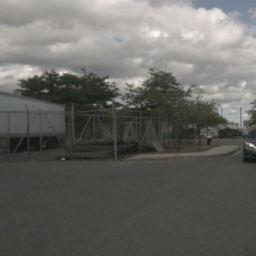} &
        \includegraphics[width=0.12\linewidth]{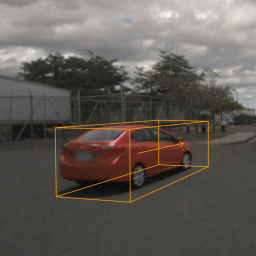} &
        \includegraphics[width=0.12\linewidth]{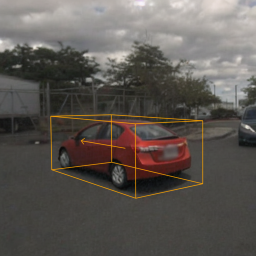} &
        \includegraphics[width=0.12\linewidth]{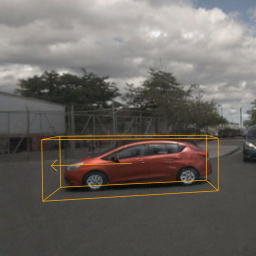} &
        \includegraphics[width=0.12\linewidth]{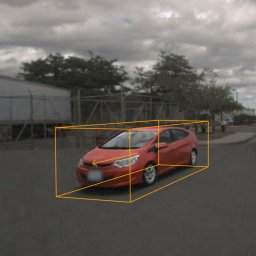} &
        \includegraphics[width=0.12\linewidth]{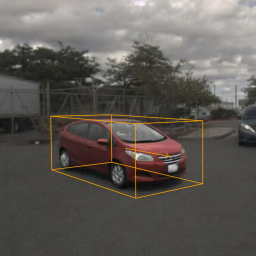} &
        \includegraphics[width=0.12\linewidth]{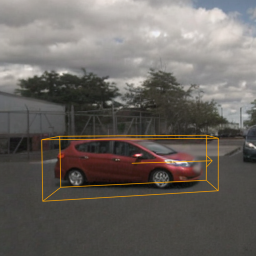} &
        \rotatebox{90}{\quad~ Camera} \\[1pt]
        
        &
        \includegraphics[width=0.12\linewidth]{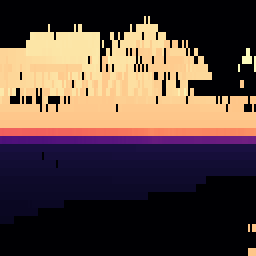} &
        \includegraphics[width=0.12\linewidth]{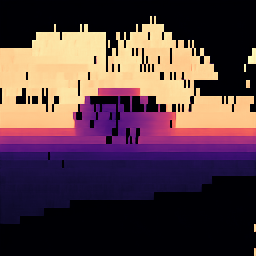} &
        \includegraphics[width=0.12\linewidth]{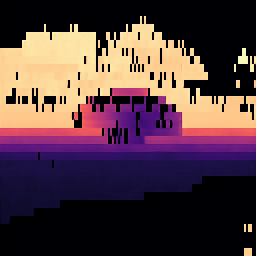} &
        \includegraphics[width=0.12\linewidth]{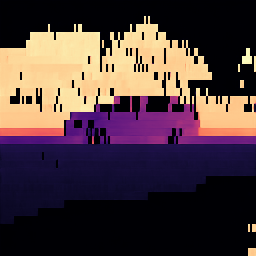} &
        \includegraphics[width=0.12\linewidth]{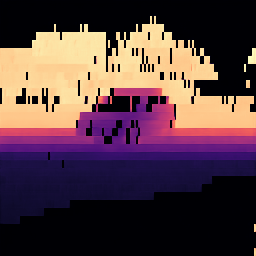} &
        \includegraphics[width=0.12\linewidth]{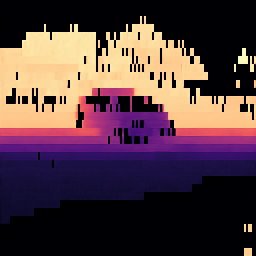} &
        \includegraphics[width=0.12\linewidth]{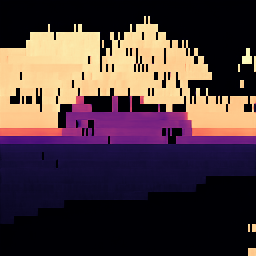} &
        \rotatebox{90}{~~ Lidar depth} \\[1pt]
        
        &
        \includegraphics[width=0.12\linewidth]{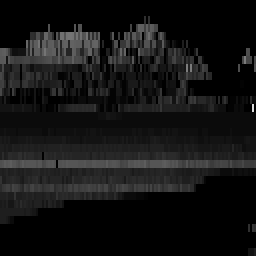} &
        \includegraphics[width=0.12\linewidth]{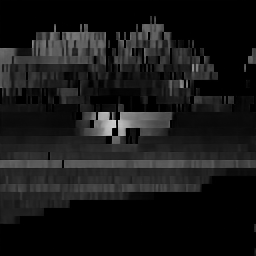} &
        \includegraphics[width=0.12\linewidth]{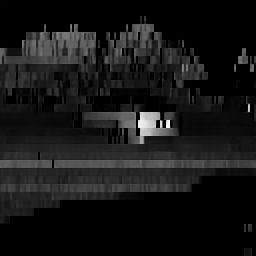} &
        \includegraphics[width=0.12\linewidth]{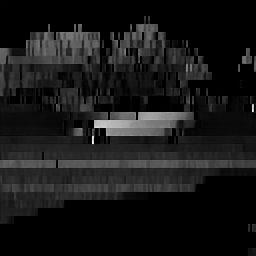} &
        \includegraphics[width=0.12\linewidth]{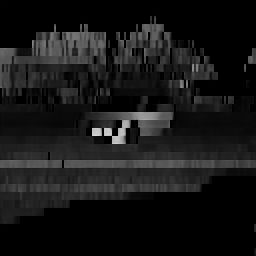} &
        \includegraphics[width=0.12\linewidth]{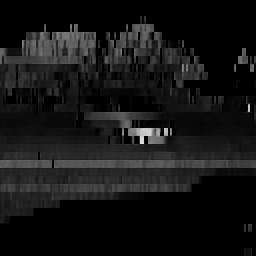} &
        \includegraphics[width=0.12\linewidth]{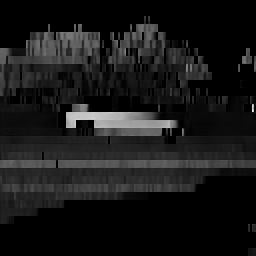} &
        \rotatebox{90}{Lidar intensity}
    \end{tabular}
    \caption[Additional examples of insertion with a rotating box.]{Additional examples showcasing the controllability of the proposed method. From left to right: reference image $ \mathbf{x}_{\text{ref}} $ extracted from a seperate source scene, original destination scene (original RGB image $ \mathbf{x}^{\text{(C)}} $, Lidar range depth $ \mathbf{x}_0^{\text{(R)}} $ and intensity $ \mathbf{x}_1^{\text{(R)}} $), and edited scenes.}
    \label{fig:suppl:controllability_main_full}
  \end{figure*}

\subsection{Realism of the inpainting} 
\label{sec:experiments:realism}

\paragraph{Camera realism metrics}
The realism of the camera inpainting is evaluated using the following metrics: Fréchet Inception Distance (FID)~\cite{heusel2017gans}, Learned Perceptual Image Patch Similarity (LPIPS)~\cite{zhang2018unreasonable}, and CLIP-I~\cite{ruiz2023dreambooth}. The CLIP-I score is computed by evaluating the cosine similarity between the CLIP embeddings of the inpainted region and the reference object. This score reflects how well the inpainted object preserves semantic and high-level visual characteristics, as captured by the CLIP image encoder~\cite{radford2021learning}. A higher CLIP-I score indicates better alignment with the reference object's identity and structure.

FID quantifies the realism of inpainted object patches by comparing their feature distribution to that of real patches. Specifically, features are extracted from the inpainted and real patches using a pretrained Inception network~\cite{szegedy2015going}, and each set of features is assumed to follow a multivariate Gaussian distribution. The Fréchet distance between these two distributions is then computed. Lower FID scores indicate that the distribution of inpainted patches is closer to the distribution of real ones, thus implying higher realism.

LPIPS measures perceptual similarity between inpainted and ground truth patches by comparing their feature maps across several layers of a deep neural network. Unlike pixel-wise metrics, LPIPS captures differences at multiple levels of abstraction, making it more aligned with human perception. A lower LPIPS score corresponds to higher perceptual similarity between the inpainted and real patches.

For FID and LPIPS, the evaluation is carried out on extended patches around the object, extracted from the final composited images, and compared against corresponding patches from the original images. For CLIP-I, the evaluation considers only the region within the 2D bounding box of the inpainted object and its corresponding reference image.

\paragraph{Lidar realism metrics}
To the best of our knowledge, metrics explicitly designed for lidar editing, particularly those capable of capturing fine perceptual differences, are not available. Existing metrics based on the Fréchet distance~\cite{lidardiffusion, lidargen, nakashima2024lidar} have been applied to full lidar point clouds, but they lack the granularity required to detect object-level differences, which are essential for tasks such as actor insertion and detailed editing.

To address this limitation, the differences in depth and intensity between the original and inpainted range images are assessed by applying LPIPS~\cite{zhang2018unreasonable} to rasterised patches. This results in the following adapted perceptual distances: \( \text{D-LPIPS}( \mathbf{x}^{\text{(R)}}_0, \tilde{\mathbf{x}}^{\text{(R)}}_0) \) for depth and \( \text{I-LPIPS}( \mathbf{x}^{\text{(R)}}_1, \tilde{\mathbf{x}}^{\text{(R)}}_1) \) for intensity. 

The output of the diffusion model (after the range decoder), normalised between 0 and 1, is used for this evaluation. Both the depth and intensity maps are tiled three times to form RGB images suitable for LPIPS. Individual scores for depth and intensity are then reported by averaging the corresponding perceptual distances across all patch pairs.

\paragraph{Results}
All realism metrics for camera-lidar object inpainting are reported in~\autoref{tab:full-realism} for the reinsertion and replacement settings. Compared to camera-only inpainting methods, MObI ($D=512$) achieves better results than PbE~\cite{yang2023paint} across almost all benchmarks. Note that this method achieves competitive results in terms of FID, producing samples which are close in distribution to the target ones, yet LPIPS is much worse. This perceptual misalignment, which is more severe than even MObI with $D=256$ with no bounding box conditioning, might indicate that the use of joint generation of camera and lidar improves semantic consistency within the scene.
A comparison was also made with a simple copy\&paste baseline, which was shown to produce unrealistic composited images when replacing objects, despite its occasional use in training object detectors~\cite{georgakis2017synthesizing, dwibedi2017cut, wang2021pointaugmenting, zhang2020exploring}. It should be noted that object reinsertion results for copy\&paste, as well as CLIP-I scores, were not computed, as such comparisons would not be fair. A breakdown of camera realism metrics for each evaluation setting is provided in~\autoref{tab:supp:camera-realism-comparison}.

Ablations were conducted for the 3D bounding box and the gated cross-attention adapter for $D=256$. When the adapter was removed, the box token was concatenated with the reference token in the PbE~\cite{yang2023paint} cross-attention layer, followed by direct fine-tuning. Due to the absence of established baselines for lidar object inpainting realism, comparative results were provided across all experiments and ablations, with the intention of establishing a foundation for future work. 

When MObI (256) with both bounding box conditioning and the adapter was compared to its counterpart without bounding box conditioning, significant improvements in perceptual alignment were observed. Using the gated cross-attention adapter resulted in more realistic samples in the camera space; however, no improvement was observed for lidar, suggesting differences in the training dynamics of the two modalities.

Finally, it is observed that realism scales strongly with resolution, indicating that models operating at higher resolutions could potentially achieve greater realism.

\begin{table*}
    \centering
    \scriptsize
    \begin{tabular}{cccccccccccccc}
        \multicolumn{3}{c}{} & \multicolumn{5}{c}{\textbf{Reinsertion}} & \multicolumn{5}{c}{\textbf{Replacement}} \\
        \cmidrule(r){4-8}  \cmidrule(r){9-13}
        \multicolumn{3}{c}{} & \multicolumn{3}{c}{\textbf{Camera Realism}} & \multicolumn{2}{c}{\textbf{Lidar Realism}} & \multicolumn{3}{c}{\textbf{Camera Realism}} & \multicolumn{2}{c}{\textbf{Lidar Realism}} \\
        \cmidrule(r){4-6} \cmidrule(r){7-8} \cmidrule(r){9-11} \cmidrule(r){12-13}
        \textbf{Model} & 3D Box & Adapter & FID\textdownarrow & LPIPS\textdownarrow & CLIP-I\textuparrow & D-LPIPS\textdownarrow & I-LPIPS\textdownarrow & FID\textdownarrow & LPIPS\textdownarrow & CLIP-I\textuparrow & D-LPIPS\textdownarrow & I-LPIPS\textdownarrow \\
        \midrule
        copy\&paste & \multicolumn{2}{c}{n/a} & \multicolumn{3}{c}{n/a} & \multicolumn{2}{c}{n/a} & 15.29 & 0.205 & n/a & \multicolumn{2}{c}{n/a} \\
        PbE~\cite{yang2023paint} & \multicolumn{2}{c}{n/a} & \underline{7.46} & 0.133 & \underline{83.91} & \multicolumn{2}{c}{n/a} & 10.08 & 0.149 & \textbf{77.25} & \multicolumn{2}{c}{n/a}\\
        \midrule
        & \ding{55} & \checkmark & 8.18 & 0.123 & 82.56 & 0.195 & 0.231 & 10.31 & 0.140 & \underline{77.22} & 0.198 & \underline{0.236} \\
        MObI (256)& \checkmark & \ding{55} & 8.31 & 0.120 & 82.88 & \underline{0.188} & 0.231 & 10.43 & 0.134 & 76.03 & \underline{0.191} & 0.237 \\
        & \checkmark & \checkmark & {7.74} & \underline{0.119} & 83.03 & 0.192 & \underline{0.230} & \underline{9.87} & \underline{0.133} & {76.75} & 0.195 & \underline{0.236} \\
        \midrule
        MObI (512) & \checkmark & \checkmark & \textbf{6.60} & \textbf{0.115} & \textbf{84.22} & \textbf{0.129} & \textbf{0.148} & \textbf{9.00} & \textbf{0.129} & {76.75} & \textbf{0.132} & \textbf{0.153}  \\
        \bottomrule
    \end{tabular}
    \caption[Quantitative results for reinsertion and replacement tasks given realism metrics.]{Camera and lidar realism metrics for the reinsertion and replacement tasks are reported, with values averaged over the \textit{tracked} and \textit{same reference} settings for reinsertion, and the \textit{in-domain} and \textit{cross-domain reference} settings for replacement. Comparisons are made with camera-only methods, and separate ablations on the use of the 3D bounding box and the gated cross-attention adapter are provided. The best result is denoted in \textbf{bold}, and the second-best is indicated with \underline{underline}.}
    \label{tab:full-realism}
\end{table*}

\begin{table*}
    \centering
    \scriptsize
    \begin{tabular}{lcccccccccccc}
        & \multicolumn{6}{c}{\textbf{Reinsertion}} & \multicolumn{6}{c}{\textbf{Replacement}}\\
        \cmidrule(r){2-7} \cmidrule(r){8-13}
        & \multicolumn{3}{c}{\textbf{same ref}} & \multicolumn{3}{c}{\textbf{tracked ref}} & \multicolumn{3}{c}{\textbf{in-domain ref}} & \multicolumn{3}{c}{\textbf{cross-domain ref}} \\
        \cmidrule(r){2-4} \cmidrule(r){5-7} \cmidrule(r){8-10} \cmidrule(r){11-13}
        \textbf{Method} & FID\textdownarrow & LPIPS\textdownarrow & CLIP-I\textuparrow & FID\textdownarrow & LPIPS\textdownarrow & CLIP-I\textuparrow & FID\textdownarrow & LPIPS\textdownarrow & CLIP-I\textuparrow & FID\textdownarrow & LPIPS\textdownarrow & CLIP-I\textuparrow \\
        \midrule
        copy\&paste & \multicolumn{6}{c}{n/a} & 13.50 & 0.196 & n/a & 17.08 & 0.213 & n/a \\
        PbE~\cite{yang2023paint} & 7.34 & 0.131 & 84.50 & 7.58 & 0.135 & 83.31 & 9.62 & 0.148 & 77.44 & 10.54 & 0.150 & \textbf{77.06} \\
        MObI & \textbf{6.50} & \textbf{0.114} & \textbf{84.94} & \textbf{6.70} & \textbf{0.115} & \textbf{83.50} & \textbf{8.95} & \textbf{0.127} & \textbf{77.50} & \textbf{9.05} & \textbf{0.130} & 76.00 \\
        \bottomrule
        \\
    \end{tabular}
    \caption[Breakdown of camera realism metrics for four evaluation settings.]{Breakdown of camera realism metrics for each evaluation setting, when compared with image inpainting methods, at $D=512$.}
    \label{tab:supp:camera-realism-comparison}
\end{table*}

\subsection{Object detection on reinserted objects}
\label{sec:domain-gap}

\paragraph{Setup}
The inpainted objects must correspond tightly to the 3D bounding box used during generation in order to be useful for various downstream tasks.
The quality of the 3D-box conditioning is analysed using an off-the-shelf object detector, and the detections are compared to the boxes used for conditioning.

The nuScenes validation split is employed, and objects to be reinserted are selected based on the same filters as in \autoref{sec:method:training details}. In cases where multiple such objects exist per frame, one is selected at random, resulting in a total of 372 objects. The \emph{tracked reference} procedure described in \autoref{sec:experiments:implementation} is followed, and each selected object is replaced using MObI, conditioned on a reference of the same object taken at a randomly chosen timestamp that is distant from the inpainting timestamp.

The evaluation is restricted to those scenes that have been edited. The multimodal BEVFusion~\cite{liu2023bevfusion} object detector, equipped with a SwinT~\cite{liu2021swin} backbone and trained on nuScenes, is used for detection. Lidar points are not accumulated over successive sweeps during the evaluation.

\begin{figure*}[ht]
        \begin{minipage}[b]{.6\linewidth}
        \scriptsize
        \raisebox{1.9cm}{
        \begin{tabular}{lcccccccc}
            & \multicolumn{2}{c}{\textbf{Scene-level}} & \multicolumn{6}{c}{\textbf{Restricted to reinserted objects}}\\
            \cmidrule(r){2-3} \cmidrule(r){4-9}
            & \multicolumn{2}{c}{\textbf{mAP}} & \multicolumn{2}{c}{\textbf{ATE}} & \multicolumn{2}{c}{\textbf{ASE}} & \multicolumn{2}{c}{\textbf{AOE}} \\
            \cmidrule(r){2-3} \cmidrule(r){4-5} \cmidrule(r){6-7} \cmidrule(r){8-9}
            & car & ped. & car & ped. & car & ped. & car & ped. \\
            \midrule
            Original     & 0.89 & 0.87 & 0.15 & 0.10 & 0.14 & 0.28 & 0.02 & 0.46 \\
            Reinsertions & 0.88 & 0.86 & 0.30 & 0.14 & 0.15 & 0.30 & 0.16 & 0.75 \\
            \bottomrule
        \end{tabular}
        }
        \end{minipage}
        \hfill
        \begin{minipage}[b]{.39\linewidth}
            \includegraphics[width=\linewidth]{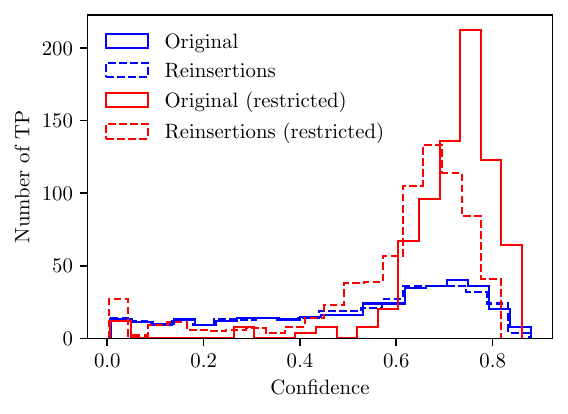}
            \vspace{-0.7cm}
        \end{minipage}
        \caption[Detection performance of existing 3D object detector on reinserted objects.]{Detection performance of an off-the-shelf BEVFusion~\cite{liu2023bevfusion} object detector on objects reinserted using the proposed method. Left: mAP is computed at the scene-level, and TP errors (translation, scale, and orientation) are computed on the reinserted objects only.
        % It can be observed that re-insertion is accompanied by a small cost in object detection performance but that errors remain small (e.g. 0.161 AOE corresponds to a $9^{\circ}$ average error).
        Left: the distribution of the scores of the true-positives is shown to shift modestly towards lower scores for edited objects.
        }
        \label{fig:detection-results}
\end{figure*}

\begin{figure*}[ht]
    \centering
    \begin{minipage}{\textwidth}
    \centering
    \begin{tabular}{ccc}
    \text{Ground Truth} & \text{Vanilla detection} & \text{Detection on reinserted samples} \\
    % \hline
    % ~\\
    \includegraphics[width=0.33\linewidth]{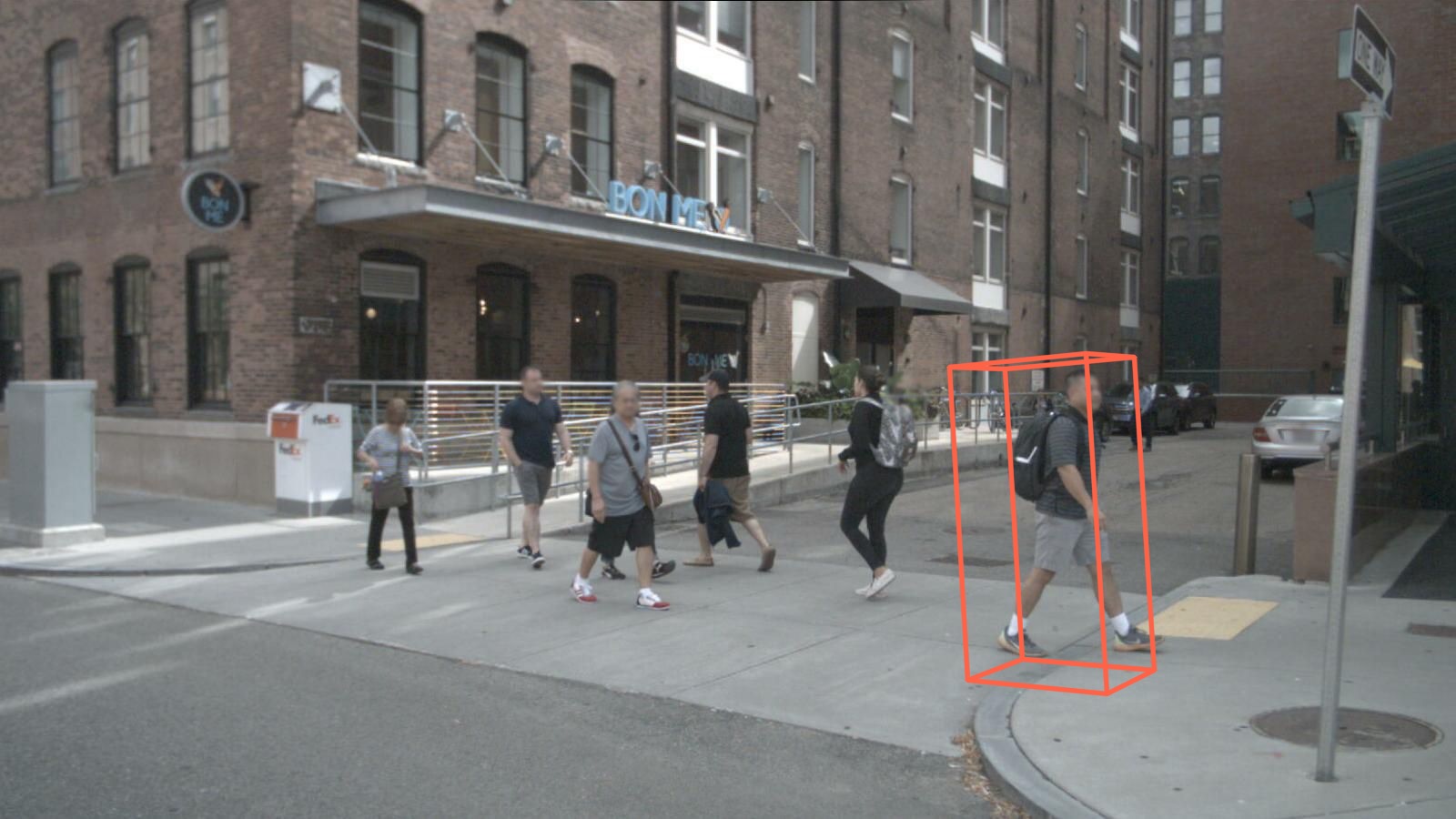} & 
    \includegraphics[width=0.33\linewidth]{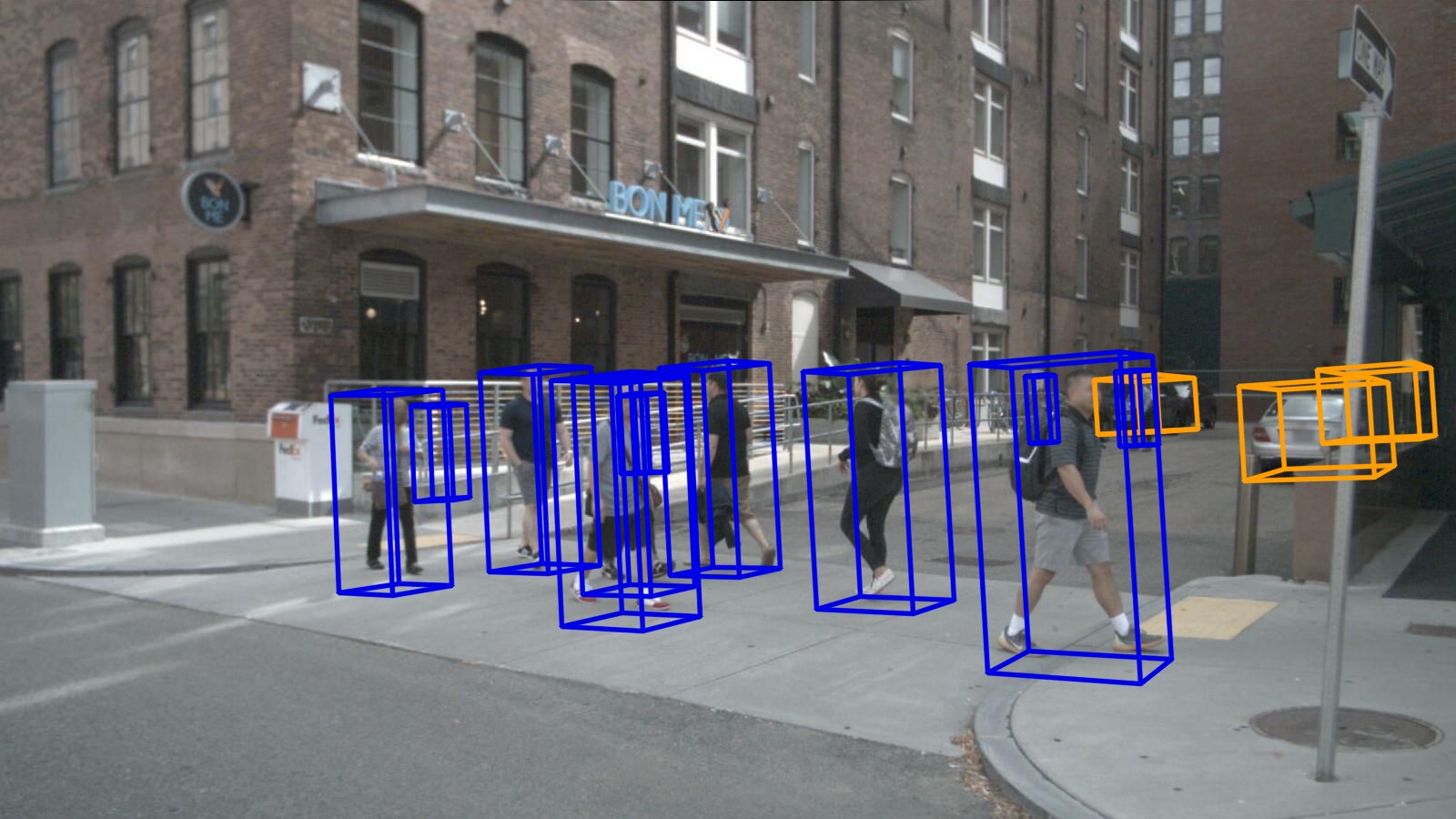} & 
    \includegraphics[width=0.33\linewidth]{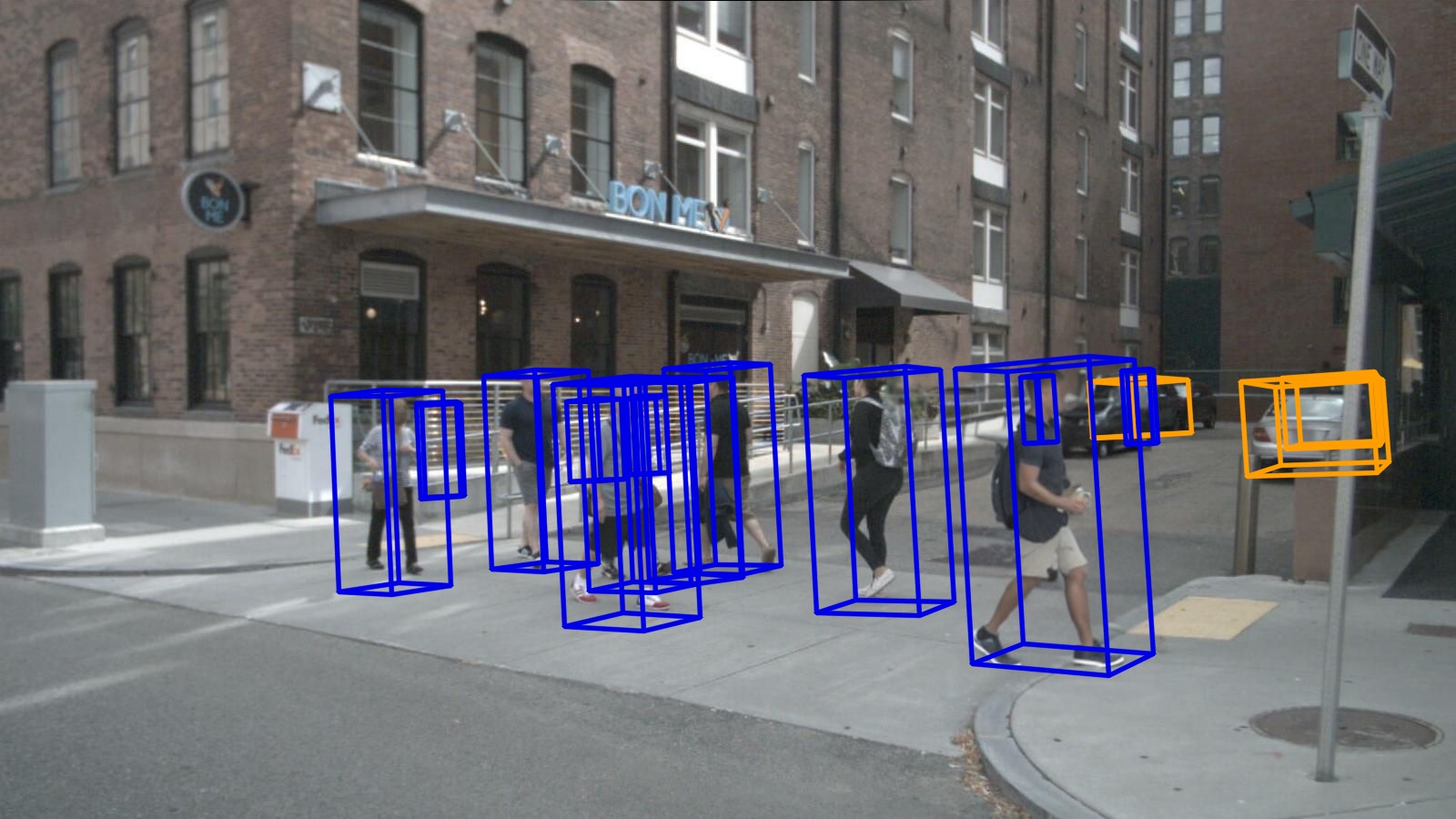} \\
    \includegraphics[width=0.33\linewidth]{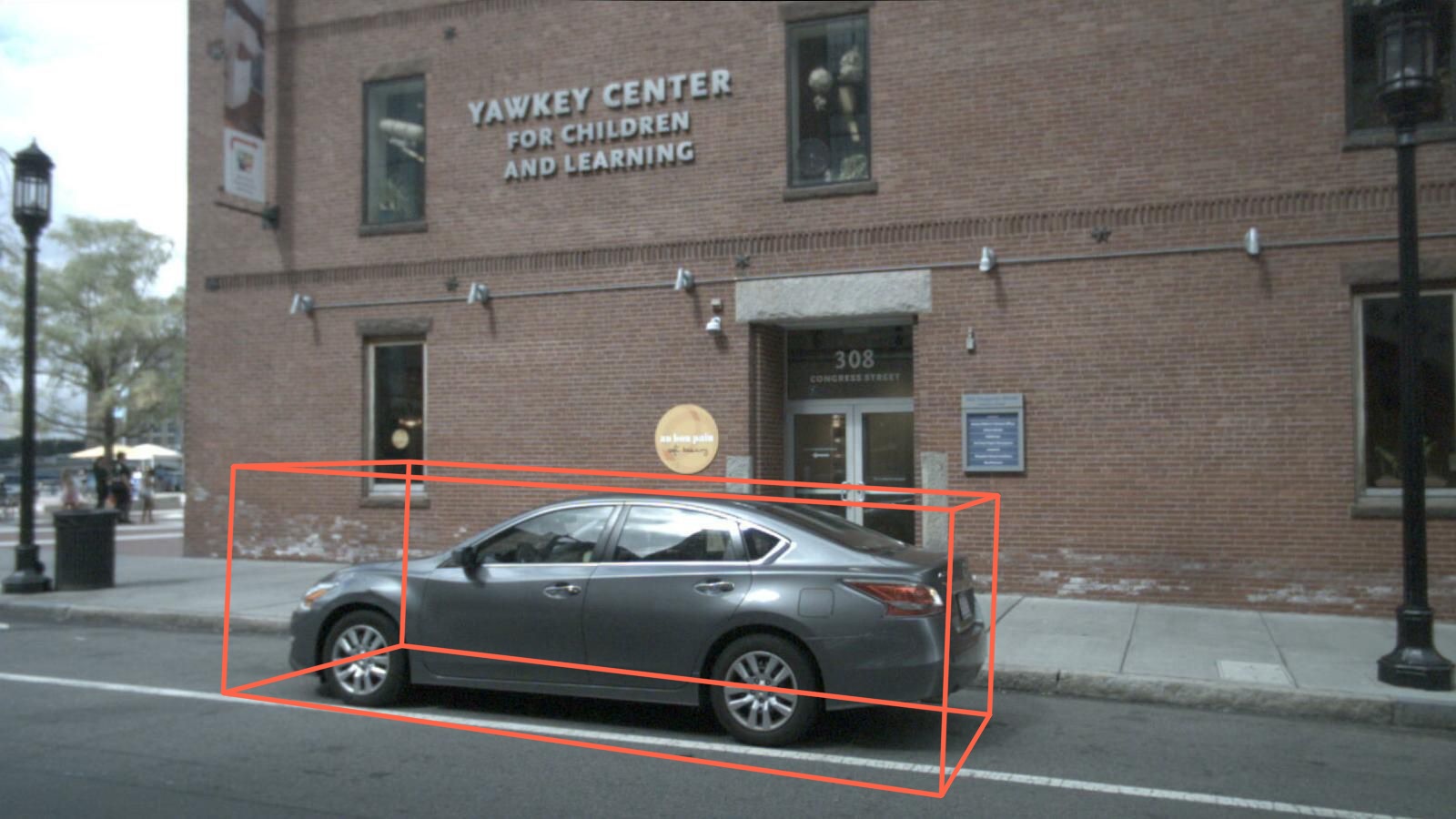} & 
    \includegraphics[width=0.33\linewidth]{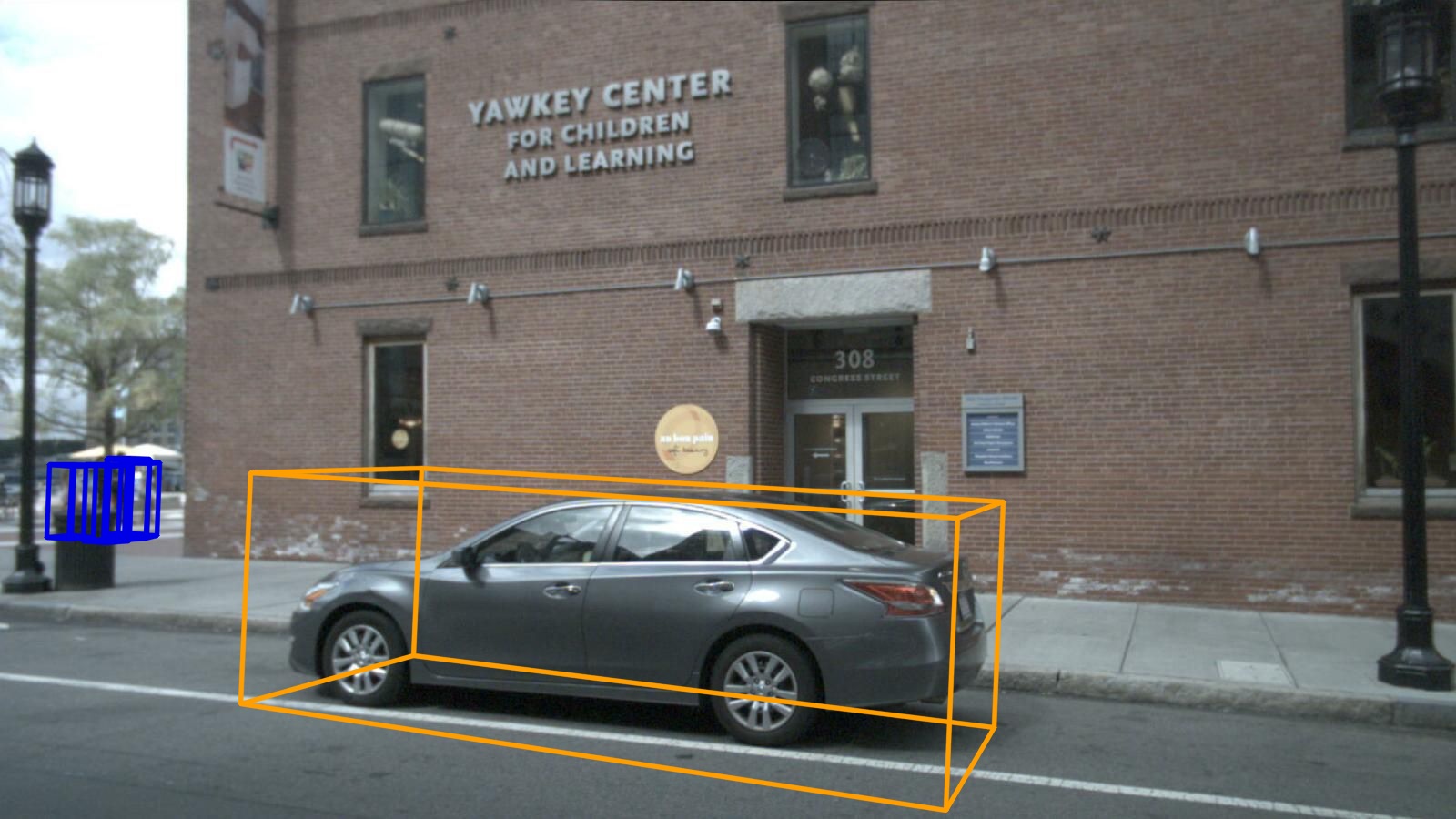} & 
    \includegraphics[width=0.33\linewidth]{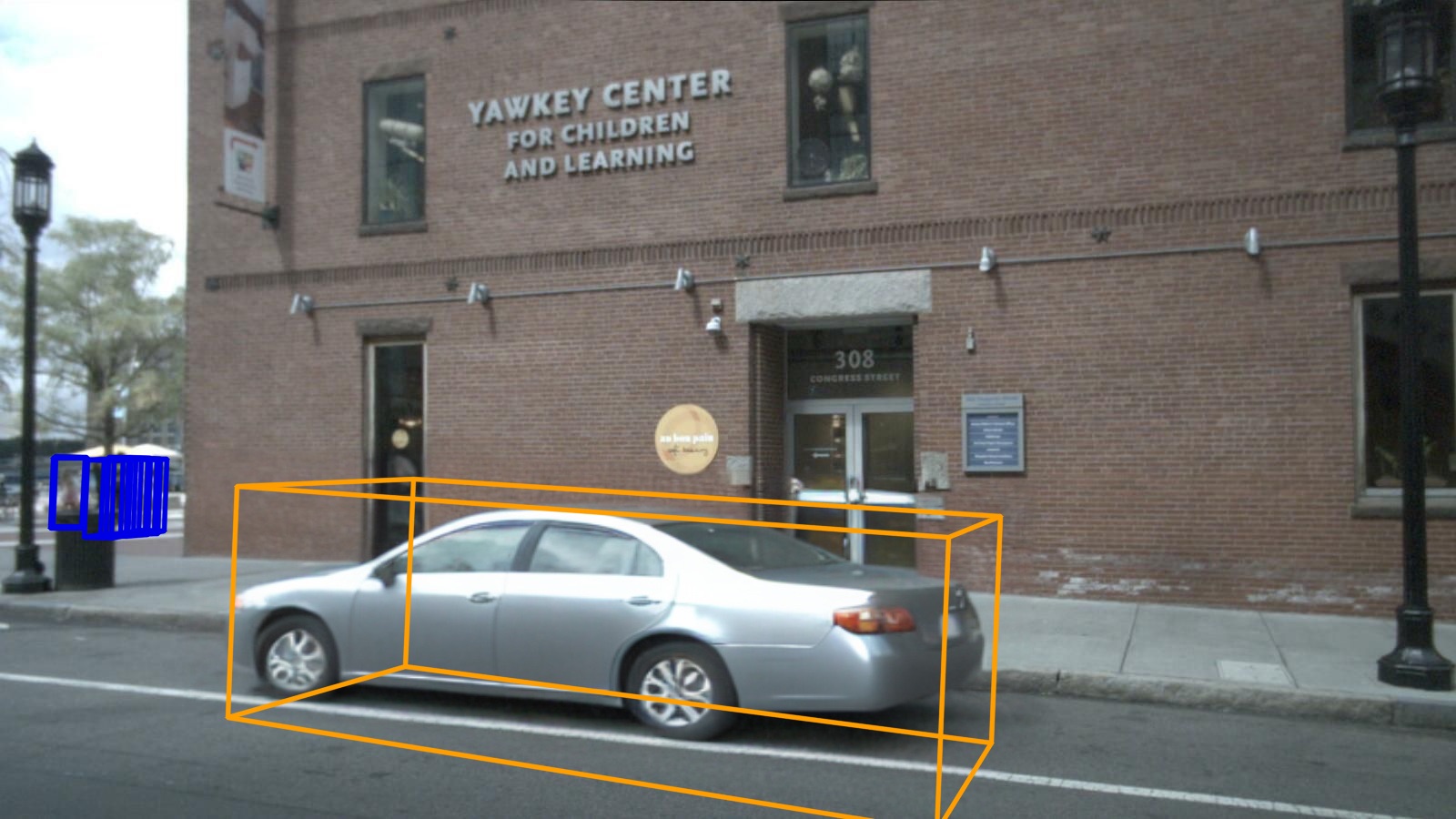} \\
    \includegraphics[width=0.33\linewidth]{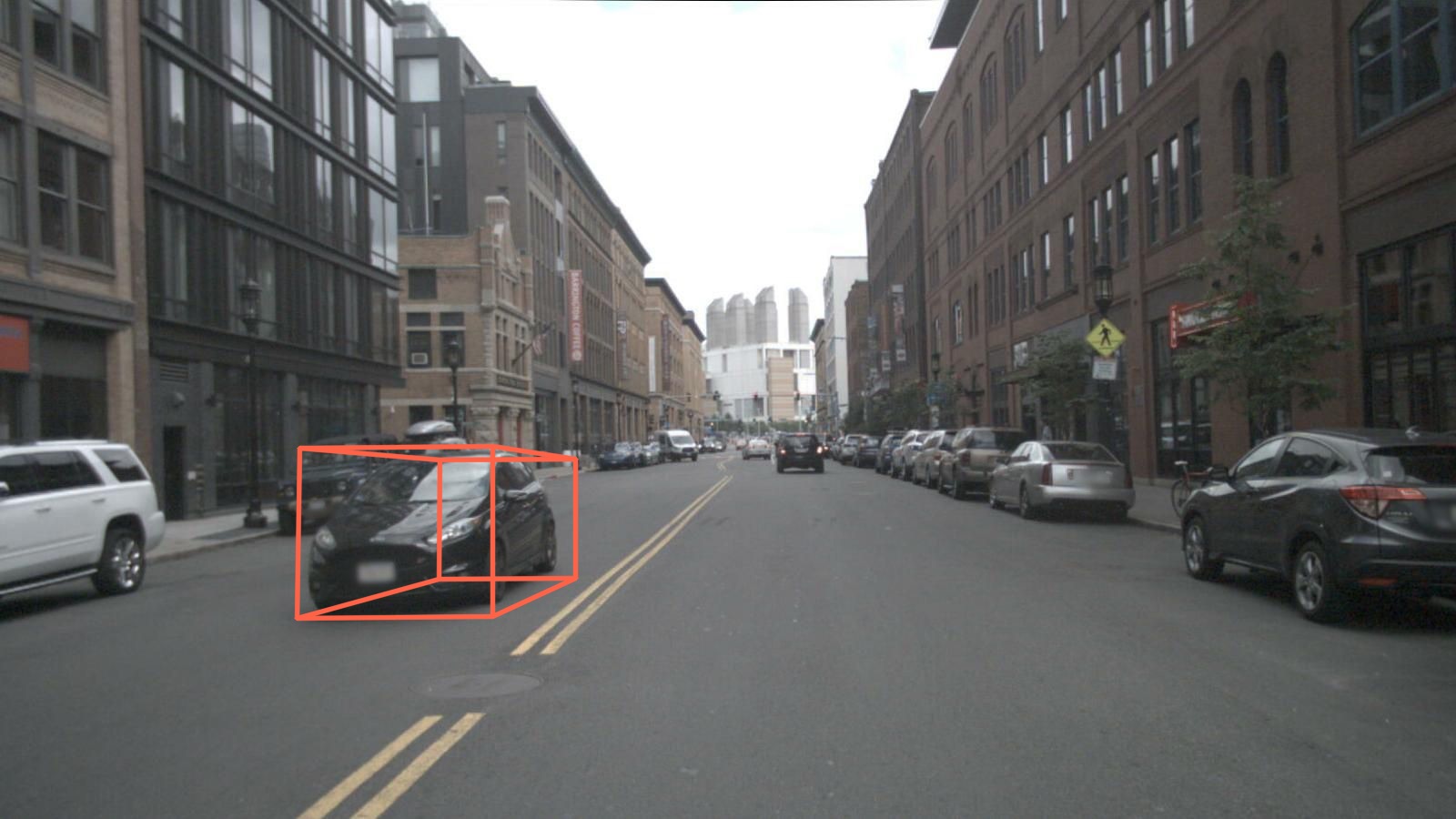} & 
    \includegraphics[width=0.33\linewidth]{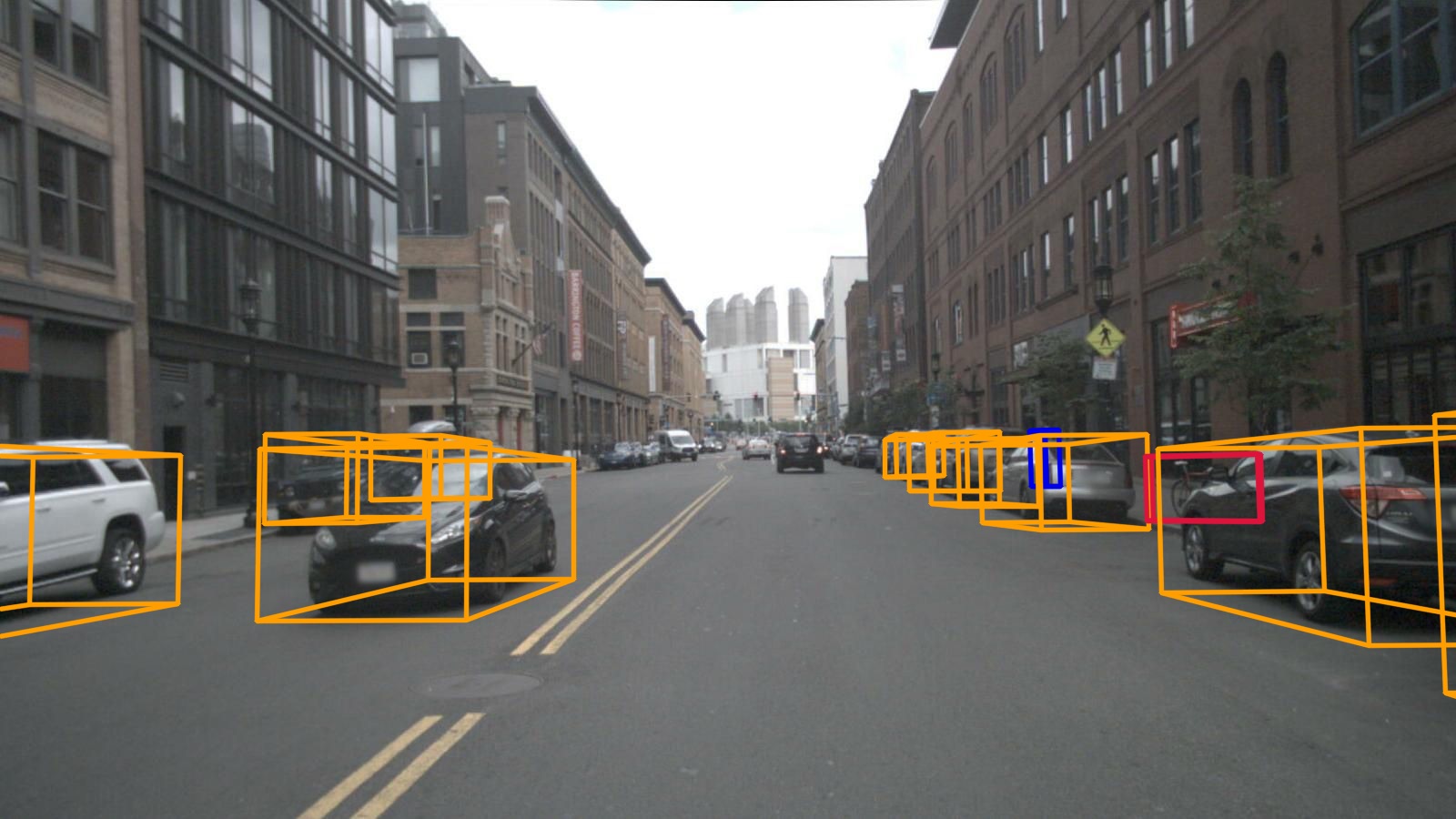} & 
    \includegraphics[width=0.33\linewidth]{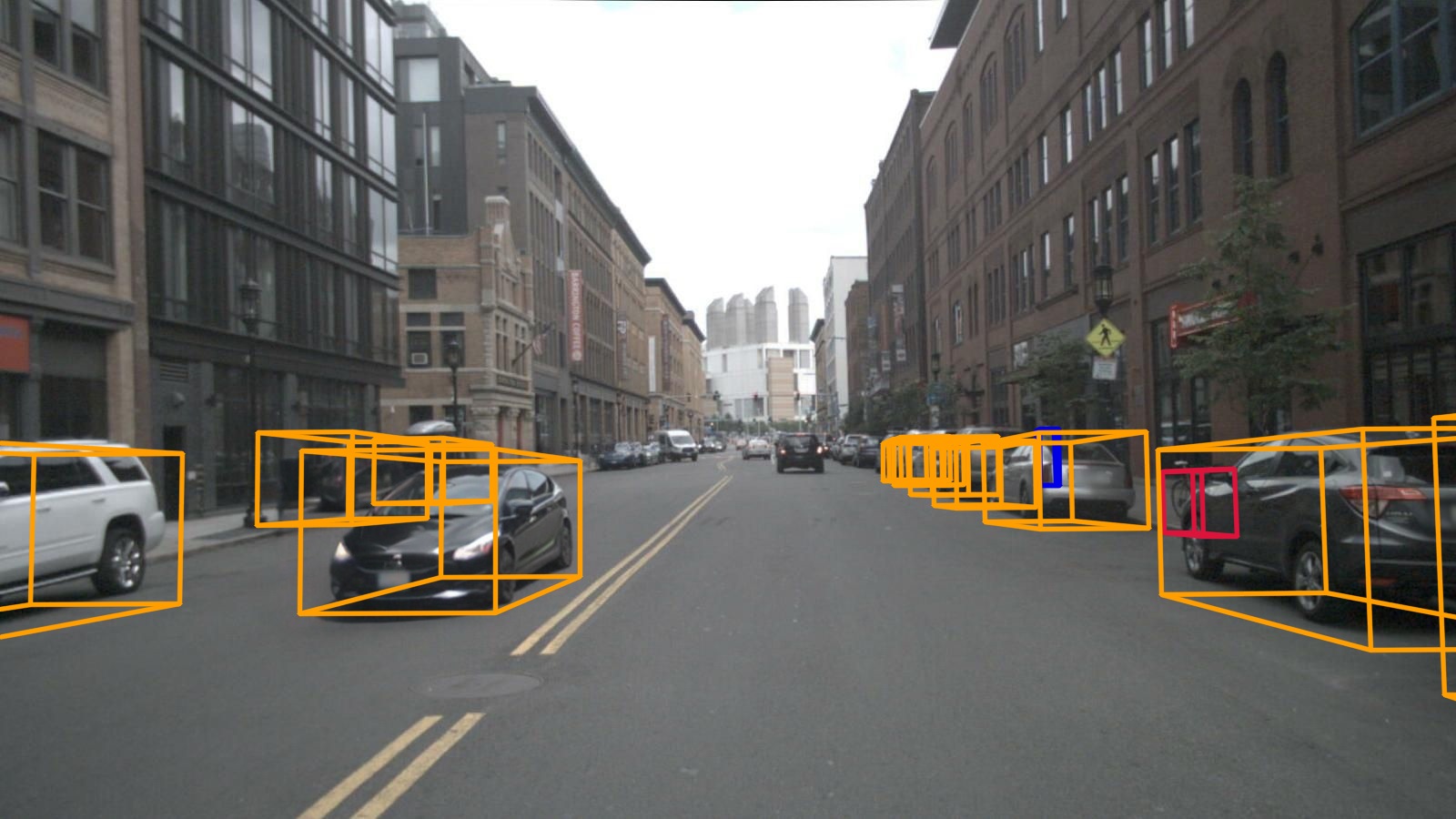} \\
    \end{tabular}
    \end{minipage}
    \caption[Object detection examples with original, compared to reinserted objects.]{Comparison of detection results between the original scene and the same scene with the object shown in red replaced. BEVFusion~\cite{liu2023bevfusion} achieves good detection performance on the object reinserted using the proposed method, while leaving the boxes of the other objects undisturbed.
    Interestingly, even though the aspect of the car behind the reinserted object in the third column is changed slightly, it does not seem to affect detection much.
    It is hypothesised that this is due to the fact that, while the camera view is sensitive to occlusions, the range view is much less so, since only the points within the box used for conditioning are reinserted, see \autoref{sec:method:spatial compositing}.
    All detections are filtered using a score threshold of 0.08.
    }
    \label{fig:detection-comparison}
\end{figure*}

\paragraph{Metrics}
Mean Average Precision (mAP) and error metrics on the re-inserted objects are computed.
Scene-level metrics such as mAP can not be easily restricted to edited objects\footnote{This is because such metrics usually require false-positives which are detections that have not been matched to \emph{any} ground truth-object within a scene, but not to a specific subset of ground-truth objects.} and will not be very sensitive to detection errors on these objects.
This is why mAP is complemented with true-positive error metrics restricted to the re-inserted objects, which are computed following the usual matching procedure from the nuScenes devkit~\cite{caesar2020nuscenes} but considering only the ground-truth/detection pairs that correspond to inpainted objects.

\paragraph{Results}
Object detection results are presented on \autoref{fig:detection-results} (left), and it can be seem that reinsertion comes at a small cost in object detection performance but that errors remain small (e.g. 0.161 AOE corresponds to a $9^{\circ}$ average error) while scene-level mAP is very similar.
The distribution of the scores of the true-positives from~\autoref{fig:detection-results} (right) shows that the scores suffer a modest decrease when the detector is applied to the reinserted samples.
Overall, this highlights that while a small domain gap exists, the proposed bounding box conditioning is able to produce samples that are both realistic and accurate, and that off-the-shelf detectors can successfully detect such objects even though they have not been trained on any synthetic data generated by the proposed method.
A sample of detections is displayed in~\autoref{fig:detection-comparison} where the reinserted object is detected accurately and the bounding boxes of the untouched objects remain almost identical.

%%%%%%%%%%%%%

\chapter{AnydoorMed: Reference-Guided Anomaly Inpainting for Medical Counterfactuals}
\label{chapter:anydoor}

\section{Introduction}

High-fidelity data is essential for developing and validating reliable computer-aided diagnostics (CAD) systems in the medical domain. However, real-world clinical datasets are challenging to collect and frequently exhibit severe class imbalance, particularly concerning rare pathologies such as malignant breast lesions. Synthetic data offers a promising avenue to mitigate these limitations by augmenting existing datasets with diverse and realistic counterfactual examples. For such data to be clinically useful, it must adhere to strict anatomical constraints, preserve fine-grained tissue structures, and allow controlled generation of abnormalities within plausible spatial contexts.

Recent diffusion-based approaches have achieved considerable progress in inpainting and object insertion through conditioning on text or segmentation masks~\cite{oh2024controllable, kumar2025prism, durrer2024denoising}, thereby enabling the synthesis of plausible anomalies in radiological scans. Nevertheless, text prompts and coarse masks often fail to capture the subtle visual and structural variations of medical anomalies, thus limiting the controllability of the generated content. In contrast, reference-guided inpainting in natural images~\cite{chen2023anydoor, ruiz2024magicinsertstyleawaredraganddrop} has demonstrated promising results in preserving object structure and texture, although this approach remains largely unexplored in the medical imaging domain.

In response to this gap, this work introduces \textbf{AnydoorMed} as a reference-guided inpainting framework designed specifically for mammography. Given a source image containing an anomaly and a target location within a--possibly healthy--scan, AnydoorMed can synthesise a new lesion that retains the visual and structural characteristics of the reference while blending it semantically with the surrounding tissue in the target context. Diffusion-based generation is employed with patch-level conditioning, enabling anatomically plausible insertion of anomalies whilst maintaining high controllability and structural fidelity. This allows for producing realistic counterfactuals that may support the training and evaluation of diagnostic models under diverse scenarios.

\noindent The contributions of this work are:
\begin{itemize}
    \item A diffusion-based reference-guided inpainting method for mammography enables realistic anomaly synthesis without relying on textual guidance.
    \item A framework for conditional generating plausible counterfactuals by transferring anomalies across patients and contexts.
    \item Empirical validation showing high detail preservation and semantic blending.
\end{itemize}

\section{Related work}

This section first explores the progression of image compositing methods within computer vision, followed by a review of recent techniques aimed explicitly at medical image generation and inpainting.

\paragraph{Image compositing}

Early efforts in synthetic data generation often relied on copy-and-paste techniques, in which objects were directly inserted into destination scenes with minimal blending~\cite{georgakis2017synthesizing, ghiasi2021simple, wang2021pointaugmenting}. While these methods demonstrated improvements in detection and segmentation tasks, especially for underrepresented classes, they suffered from limited controllability and unrealistic blending artefacts, particularly in the image domain.

Significant progress has since been made in image compositing, where the objective is to seamlessly insert and blend objects into destination scenes in a visually coherent manner. Early approaches, such as ST-GAN~\cite{lin2018st}, addressed the problem of unrealistic foreground blending by employing Generative Adversarial Networks (GANs)~\cite{goodfellow2014generative} in combination with spatial transformer networks. In this framework, warping corrections were recursively predicted and applied to achieve more natural object integration via learned geometric transformations.

Further advances were realised with ObjectStitch~\cite{song2023objectstitch}, in which diffusion-based inpainting was applied within edit masks to enable smooth and localised patch-level blending. Paint-by-Example~\cite{yang2023paint}, in which a latent diffusion model was conditioned on both the scene context and an edit mask, further led to improvements in semantically meaningful and spatially aligned object insertion. This framework encoded reference information using CLIP~\cite{radford2021learning}, allowing for alignment between visual and semantic features without requiring paired data. Building upon this, AnyDoor~\cite{chen2023anydoor} introduced a more expressive and modular design by incorporating DINOv2~\cite{oquab2023dinov2} for visual reference encoding. In addition to segmentation masks extracted using the Segment Anything Model (SAM)~\cite{kirillov2023segment}, AnyDoor employed a dual-path encoder architecture to extract global context and high-frequency visual features. A dedicated detail encoder captured fine-grained spatial information from the destination image, facilitating sharper and more localised blending. This multi-scale conditioning pipeline significantly improved the model’s ability to adapt inserted content to local image structure while preserving semantic alignment with the reference.

Complementary strategies have been explored by Magic Insert~\cite{ruiz2024magicinsertstyleawaredraganddrop}, where drag-and-drop style transfer enables consistent object insertion across stylistically divergent domains, and by~\cite{kulal2023puttingpeopleplaceaffordanceaware}, in which affordance-aware pose adjustments are introduced to ensure the physical plausibility of inserted elements. Additionally, ObjectDrop~\cite{winter2024objectdropbootstrappingcounterfactualsphotorealistic} demonstrated that training on synthetically generated counterfactuals could improve photorealistic object placement and compositing.

While these methods represent substantial improvements in achieving context-aware and semantically aligned image compositing, they have predominantly focused on natural image domains. Medical imagery, by contrast, presents distinct challenges including limited data availability, strict anatomical constraints, and higher demands for clinical interpretability. These limitations have yet to be comprehensively addressed by the approaches above, motivating the development of more domain-specific solutions.

\paragraph{Medical image generation and inpainting}

Recent advances in medical image generation have seen diffusion models employed to synthesise high-quality, anatomically realistic data for tasks such as data augmentation and class balancing, with approaches ranging from segmentation-guided control~\cite{konz2024anatomically} to text-conditioned synthesis~\cite{oh2024controllable,kumar2025prism}. Inpainting, a specialised form of generation, has been used to create counterfactual examples by replacing or editing specific regions. For instance, healthy tissue may be synthesised in place of lesions~\cite{durrer2024denoising}, or a pathology may be inpainted for scenario analysis~\cite{alaya2024mededit,bercea2024diffusion,perez2024radedit,fontanella2024diffusion}. Most current methods in the medical domain are conditioned on text descriptions or segmentation masks to guide the generation process; however, text prompts often lack the granularity required to capture detailed anatomical or pathological variations, thereby limiting conditional control for diffusion inpainting methods. It may be considered that providing a reference image as conditioning allows for much finer control, enabling precise and realistic counterfactual image generation.

\section{Method}
\begin{figure}
    \centering
    \includegraphics[width=0.98\linewidth]{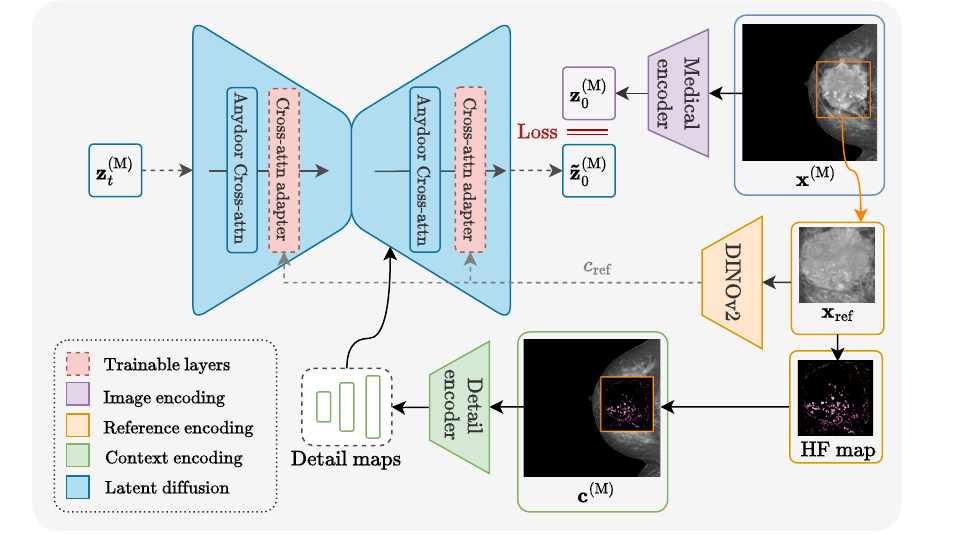}
    \caption[AnydoorMed architecture and training pipeline]{AnydoorMed architecture and training pipeline. The anomaly's High Frequency map (HF map) was coloured purple for visualisation purposes.}
    \label{fig:AnydoorMed}
\end{figure}

AnydoorMed extends the reference-based inpainting strategy of AnyDoor~\cite{chen2023anydoor} to the medical imaging domain, targeting the synthesis of anomalies in mammography scans. A latent diffusion model~\cite{rombach2022high, ho2020denoising, sohl2015deep} is trained to generate plausible insertions conditioned on both the visual context and a reference anomaly, as illustrated in \autoref{fig:AnydoorMed}.

The model \( \epsilon_\theta \) is trained to predict the noise added to the latent representation of the target image, denoted \( \mathbf{z}_0 \), at a given diffusion timestep \( t \). This representation is noised to obtain \( \mathbf{z}_t^{\text{(M)}} \), and the model is conditioned on both a reference anomaly encoding \( \mathbf{c}_{\text{ref}} \) and a contextual embedding \( \mathbf{c}^{\text{(M)}} \). The latter is computed via a dedicated detail encoder and incorporates a high-frequency feature map to enhance structural fidelity. The training objective is defined as:
\begin{align*}
  \mathcal{L} = \mathbb{E}_{\mathbf{z}^{\text{(M)}}_t, \mathbf{z}_0, t, \mathbf{c}, \epsilon \sim \mathcal{N}(0, 1)} 
  \left[ \left\| \epsilon - \epsilon_{\theta}(\mathbf{z}^{\text{(M)}}_t, \mathbf{c}_{\text{ref}}, \mathbf{c}^{\text{(M)}}, t) \right\|^2 \right].
\end{align*}

Here, \( \mathbf{c}^{\text{(M)}} \) comprises the latent representation of the image context, the spatial edit mask, and its corresponding high-frequency information, all of which are aligned and fused through the encoder. These embeddings are forwarded to the decoder of the U-Net-style network~\cite{ronneberger2015u}, guiding the denoising process. 

This design enables the model to synthesise contextually appropriate and structurally coherent anomalies by leveraging global scene features and local texture cues from the reference.

\subsection{Mammography processing}

\paragraph{DICOM conversion}
The mammography scans from the VinDR-Mammo dataset~\cite{nguyen2023vindr} were preprocessed by converting the original DICOM files (Digital Imaging and Communications in Medicine) into standardised PNG images, which were then deemed suitable for subsequent analysis. Each scan was uniquely identified by a \texttt{study\_id} and \texttt{image\_id}, which were retrieved from a CSV file containing breast-level annotations.

For each image, the corresponding DICOM file was loaded, and key metadata fields were extracted, including \texttt{WindowCenter}, \texttt{WindowWidth}, \texttt{RescaleSlope}, and \texttt{RescaleIntercept}. The raw pixel data were then extracted and rescaled in accordance with the DICOM standard, employing the linear transformation:
\begin{equation}
I = (\text{Raw Pixel Value}) \times \texttt{RescaleSlope} + \texttt{RescaleIntercept},
\end{equation}
where \( I \) represents the rescaled pixel intensity.

Subsequent to rescaling, windowing was applied in order to enhance visual contrast. The pixel intensities were centred and scaled based on the specified window centre and width. The resulting values were clipped to the displayable range and normalised to an 8-bit scale within the interval \([0, 255]\).

\paragraph{Anomaly processing}
In addition to image processing, anomalies associated with each scan were extracted from the corresponding annotations. The anomaly classes considered in this study include: Architectural Distortion, Asymmetry, Focal Asymmetry, Global Asymmetry, Mass, Nipple Retraction, Skin Retraction, Skin Thickening, Suspicious Calcification, and Suspicious Lymph Node.

Each mammographic finding is also assigned a BI-RADS score~\cite{bibrads2013}, categorised from 1 to 5, with 1 indicating the lowest and 5 indicating the highest level of suspicion for malignancy.

For each anomaly, a bounding box is provided to localise the anomaly within the image. The bounding box is defined by the coordinates \( \text{box} = (x_{\text{min}}, y_{\text{min}}, x_{\text{max}}, y_{\text{max}}) \). Examples of anomalies with their corresponding bounding boxes can be seen in \autoref{fig:vindr}.

\begin{figure}
    \centering
    \includegraphics[scale=0.18]{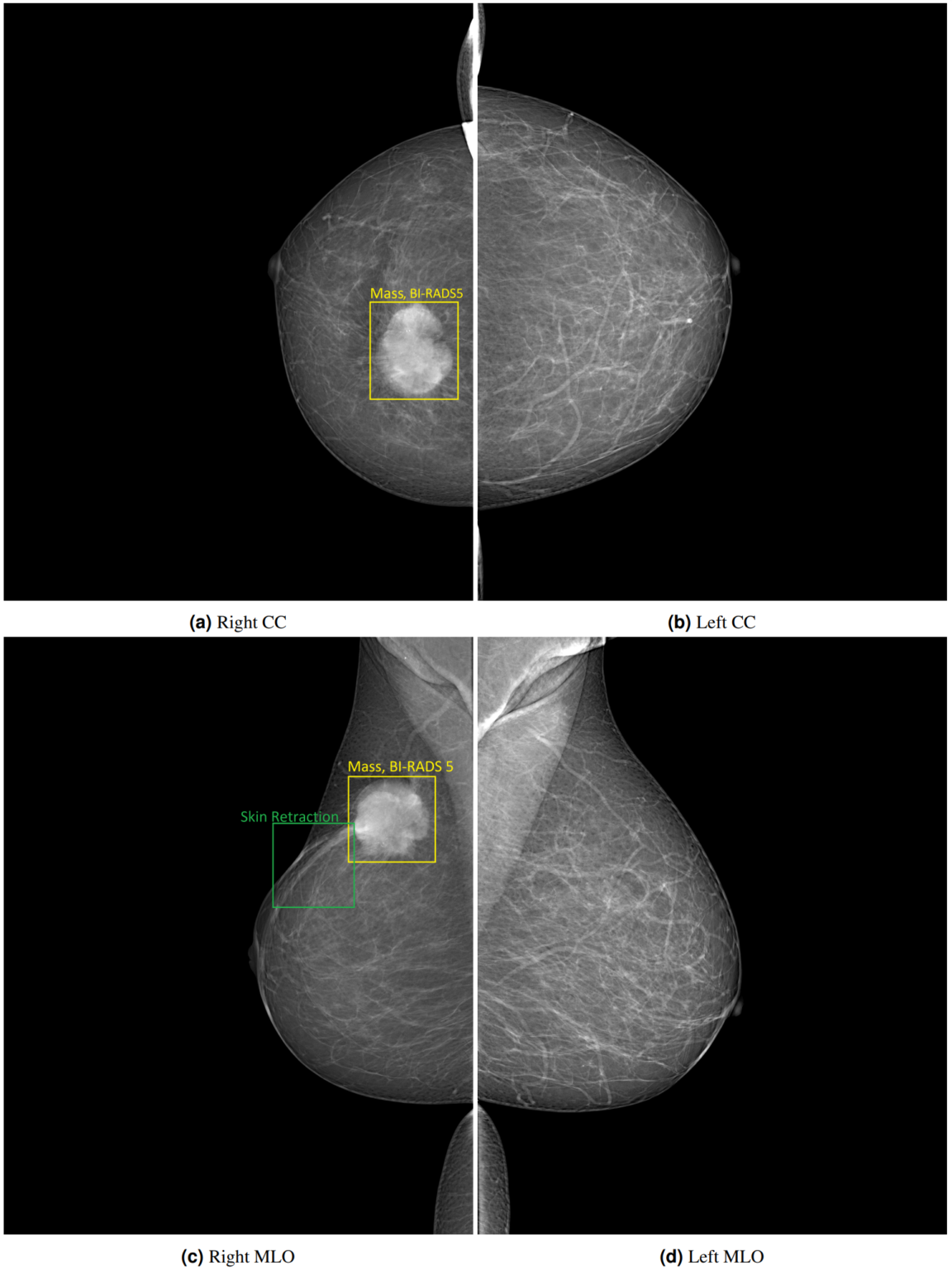}
    \caption[Samples from VinDR-Mammo dataset with bounding box annotations.]{Samples from VinDR-Mammo dataset~\cite{nguyen2023vindr} with bounding box annotations.}
    \label{fig:vindr}
\end{figure}

The distribution of anomaly types and BI-RADS scores is visualised in \autoref{fig:findings} and \autoref{fig:birads} to assess dataset characteristics, which reveals a significant class imbalance.

\begin{figure}
    \centering
    \includegraphics[width=\linewidth]{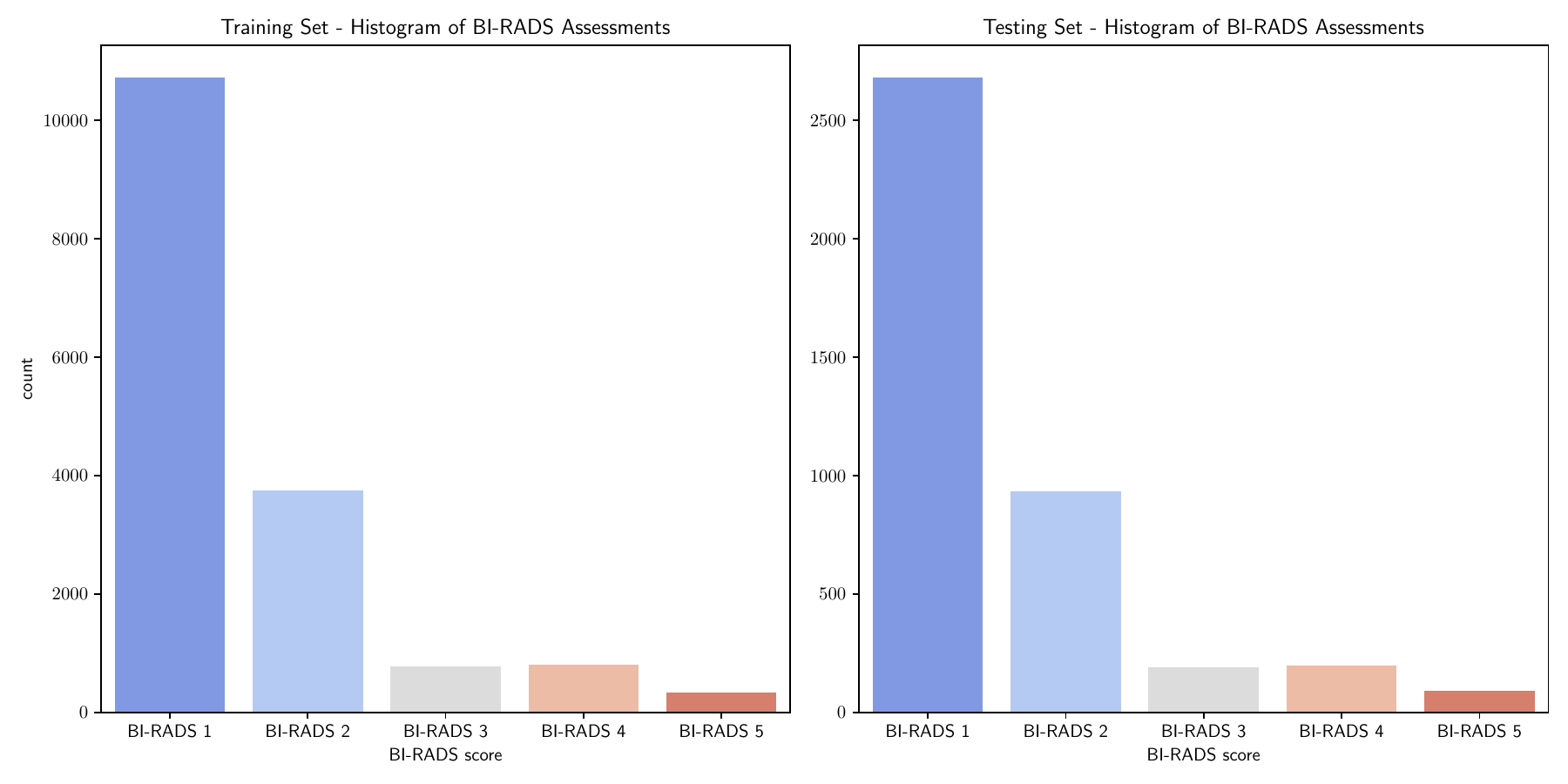}
    \caption{Distribution of BI-RADS malignancy scores, showcasing class imbalance.}
    \label{fig:birads}
\end{figure}

\begin{figure}
    \centering
    \includegraphics[width=\linewidth]{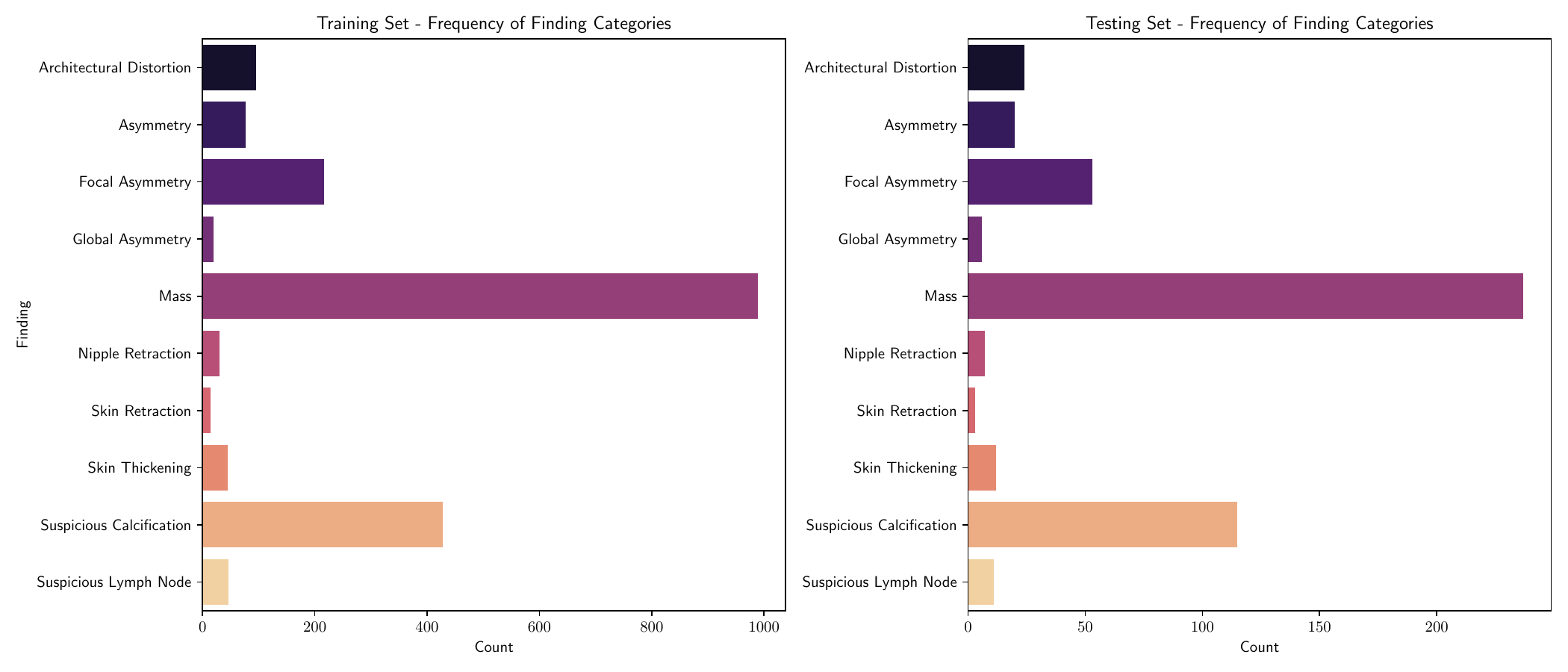}
    \caption{Distribution of anomalies based on their class, showcasing class imbalance.}
    \label{fig:findings}
\end{figure}

\paragraph{Medical image processing}
The resulting medical images from the DICOM conversion typically have dimensions of approximately $3000 \times 3000$ pixels, making them prohibitively large for generative modelling. To address this, we adopt the zoom-in strategy from Anydoor~\cite{chen2023anydoor} to crop around each anomaly. Each edge of the bounding box is scaled to be 3 to 4 times larger than the corresponding edge of the resulting square crop. If this scaling leads to overflow, padding is applied to ensure the desired crop size. The images are then normalised to a range of \([-1, 1]\). A mask is generated to in-fill the bounding box, and for the context image, the anomaly is erased by applying the mask with zero values.

\subsection{Medical image encoding}
Similar to MObI, we adapt the pre-trained image variational autoencoder (VAE)~\cite{kingma2013auto} from StableDiffusion~\cite{rombach2022high} to now encode mammography scans through a simple fine-tuning process.

The input and output convolutions of the pre-trained image encoder and decoder are replaced with two residual blocks~\cite{he2016deep}, respectively. This modification introduces one input and output channels. The VAE~\cite{kingma2013auto} is fine-tuned at a resolution of 512, adding the perceptual loss of~\cite{esser2021taming}.

This design choice ensures that most of the original encoder architecture remains unchanged, preserving the mapping to the latent space with minimal disruption. This approach maintains the efficiency of the pre-trained model while tailoring it to the specific characteristics of mammography scans.

\subsection{Reference encoding}
The anomaly is extracted using the corresponding bounding box, resulting in a cropped reference image. Basic augmentation techniques, such as horizontal flipping and random brightness–contrast adjustment, are applied to improve generalisability.

This work employs DINOv2~\cite{oquab2023dinov2}, a foundation model for visual representation learning, which was trained in a fully self-supervised manner using a teacher–student framework. Specifically, DINOv2 uses a Vision Transformer (ViT)~\cite{dosovitskiy2020image} as both the student and teacher network, where the student is trained to match the output of the teacher across multiple crops of the same image without the use of labels. The teacher network is updated using an exponential moving average of the student weights, encouraging stability and consistency in feature learning.

Unlike CLIP~\cite{radford2021learning}, which relies on paired text–image data and is primarily trained on natural images, DINOv2 learns solely from visual information. This makes it more suitable for applications in the medical imaging domain, where textual supervision is limited and visual information containing fine-grained anatomical details must be preserved by the feature extractor. The final output is a set of reference tokens, denoted as $\mathbf{c}_{\text{ref}}$, which encode the semantic features of the anomaly.

\subsection{Detail extractor}
\paragraph{High-frequency map}
Same as AnyDoor~\cite{chen2023anydoor}, the proposed method incorporates a high-frequency (HF) map to guide the generation process with fine-grained structural detail. This map is computed using the horizontal and vertical Sobel filters to enhance edge information relevant to medical anomalies such as microcalcifications. Formally, the high-frequency map \( I_{\text{hf}} \) is defined as
\[
I_{\text{hf}} = \left( I \otimes K_h + I \otimes K_v \right) \odot I \odot M_{\text{erode}},
\]
where \( I \) denotes the greyscale mammogram, \( K_h \) and \( K_v \) are the horizontal and vertical Sobel kernels, and \( M_{\text{erode}} \) is an eroded binary mask used to suppress boundary noise. The resulting HF map emphasises sharp transitions and structural boundaries. It is collaged into the corresponding region of the context image in latent space, enhancing the model’s ability to preserve anatomical detail during synthesis.

\paragraph{Detail encoding}
To further improve spatial fidelity, the pre-trained Detail Encoder from AnyDoor~\cite{chen2023anydoor} is employed. This module extracts multi-scale feature maps from the context image, capturing coarse and fine contextual information. Following the approach introduced in ControlNet~\cite{zhang2023controlnet}, these features are added to the decoder layers of the U-Net~\cite{ronneberger2015u} denoising network. This conditioning strategy allows the decoder to utilise fine-grained structural cues while maintaining the overall generative capacity of the pre-trained encoder.

\subsection{Conditional Generation}
We finetune a single latent diffusion model, leveraging the pre-trained weights of Anydoor~\cite{chen2023anydoor}.
Similar to the adaptation strategy of Flamingo~\cite{alayrac2022flamingo}, we interleave gated cross-attention layers. We use a zero-initialised gating as in ControlNet~\cite{zhang2023controlnet}. This is the same strategy as MObI, which is the second key component of the presented fine-tuning recipe

\paragraph{Adaptation}
To adapt the model to the new modality, gated cross-attention layers are interleaved, attending to the reference tokens \( \mathbf{c}_{\text{ref}} \). The query, key, and value representations are derived from the input mammography latent representation and the reference \( \mathbf{c}_{\text{ref}} \), with layer normalisation applied for cross-attention from the mammography representation to the reference. The cross-attention mechanism is computed using learnable transformations \( W_Q, W_K^{\text{(ref)}}, W_V^{\text{(ref)}} \), as follows:

\[
\text{Attn} = \text{softmax}\left(\frac{Q K^T}{\sqrt{d_{\text{head}}}}\right) V,
\]
where \( Q = W_Q \mathbf{z}^{\text{(M)}} \) represents the query tokens from the mammography latent representation, \( K^{\text{(ref)}} = W_K^{\text{(ref)}} \mathbf{c}_{\text{ref}} \) denotes the key tokens from the reference, and \( V^{\text{(ref)}} = W_V^{\text{(ref)}} \mathbf{c}_{\text{ref}} \) corresponds to the value tokens from the reference. The attention is subsequently used to update the features via a residual connection, applied through a zero-initialised gating module:

\[
\mathbf{h} \leftarrow \mathbf{h} + \text{Gate}(\text{Attn}).
\]

These adaptation layers facilitate the model's ability to incorporate information from the reference modality, thereby steering the network towards effectively capturing relevant features from both the mammography representation and the reference tokens.

\subsection{Inference and compositing}

\paragraph{Inference process}
The method begins with pure Gaussian noise, which is iteratively denoised over \( T = 50 \) steps using the DDIM sampler~\cite{song2020denoising}. This process is conditioned on the reference and detail maps, ultimately yielding the final latent representation \( \tilde{\mathbf{z}}_0^{(\text{M})} \). The resulting latent representation, which has dimensions \( 64 \times 64 \times 4 \), is decoded using the decoder of the medical VAE to generate the edited mammography images.

\paragraph{Medical Image Compositing}
The final edited scan is obtained by compositing the zoomed-in edited image. Inpainting is performed within the bounding box, and the extracted inpainted region is composited back into its corresponding location within the original scan. A Gaussian kernel could further be applied to improve the blending of the inpainted region. This approach is particularly effective as the latent diffusion models only modify a smaller region within the high-resolution scan. This strategy works because the model is trained to avoid altering areas outside the bounding box edit region.

\subsection{Training details}
\paragraph{Dataset}
The split used in Vindr-Mammo~\cite{nguyen2023vindr} is followed, with 4000 images allocated for training and 1000 for validation. Only positive samples are considered.

\paragraph{Fine-Tuning Procedure}
Training begins by adapting the newly added input and output adapters of the range autoencoder, while the rest of the image VAE~\cite{kingma2013auto} from Stable Diffusion~\cite{rombach2022high} remains frozen. This training phase spans 16 epochs (7k steps) at a batch size of 4, with a learning rate of \(4.5 \times 10^{-5}\), consistent with the original model training, and is optimised using Adam~\cite{kingma2014adam}. The checkpoint with the lowest reconstruction loss is selected.

During the fine-tuning of the latent diffusion model, the autoencoder and all layers of Anydoor~\cite{chen2023anydoor} are kept frozen, except for the gated cross-attention adapter, which is trained. An input dimension of \(D = 512\) and a latent dimension of \(D_h = 64\) are used, with training lasting for 30 epochs (approximately 4k steps). The model is trained with a constant learning rate of \(2 \times 10^{-5}\) and a batch size of 16, using the Adam~\cite{kingma2014adam} optimiser. The top three models with the lowest validation loss are retained. The final model is selected based on the best Fréchet Inception Distance (FID)~\cite{heusel2017gans} achieved on a test set comprising 426 samples from the validation set, where anomalies are reinserted into the scan.

Note that this separate model selection procedure is necessary due to the inefficiency of evaluating perceptual realism during the training process, as it requires 50 steps of denoising and decoding to obtain the final image.

\paragraph{Hyperparameter tuning}
Six learning rate values are ablated in the hyperparameter tuning process. The final model is selected based on the best Fréchet Inception Distance (FID)~\cite{heusel2017gans} achieved on the test set with 426 samples, where anomalies are reinserted into the scan.

\section{Experiments and results}

\subsection{Setup}
The model's performance was evaluated using three distinct tasks: Insertion, Replacement, and Reinsertion. These tasks were designed to test the model's ability to integrate anomalies into mammography scans under different conditions, with context and high-frequency maps used to guide the inpainting process. The experiments used the 426 positive scans with anomalies from the validation set. 

\paragraph{Insertion}
In the Insertion task, a reference anomaly was inserted into a healthy mammography scan. A healthy medical scan was selected for this, and a reference anomaly was chosen. Twenty candidate bounding boxes were randomly generated across the breast tissue. If the overlap between the reference anomaly and a candidate box was less than 90\%, the box was discarded. If none of the boxes met the 90\% overlap requirement, the highest overlap box was selected. This heuristic method ensured the reference anomaly was inserted into the most appropriate location within the scan, blending as naturally as possible with the surrounding tissue.

\paragraph{Replacement}
In the Replacement task, an existing anomaly in the scan was replaced by a reference anomaly of similar size. The replacement was guided by context and high-frequency maps, ensuring seamless integration of the new anomaly into the surrounding tissue while preserving the scan's anatomical structure.

\paragraph{Reinsertion}
In the Reinsertion task, a previously removed anomaly was reintroduced into the scan using the reference anomaly. The insertion was guided by context and high-frequency maps to ensure the anomaly blended naturally with the surrounding tissue, restoring the scan’s original structure while maintaining realism.

\begin{figure*}[htbp]
    \centering
    \begin{minipage}{0.87\textwidth}
    \begin{tabularx}{0.96\columnwidth}{@{}>{\centering\arraybackslash}X>{\centering\arraybackslash}X>{\centering\arraybackslash}X>{\centering\arraybackslash}X@{}}
        Context & Reference & Generation (insertion) & Original \\
    \end{tabularx}
    \includegraphics[width=\columnwidth]{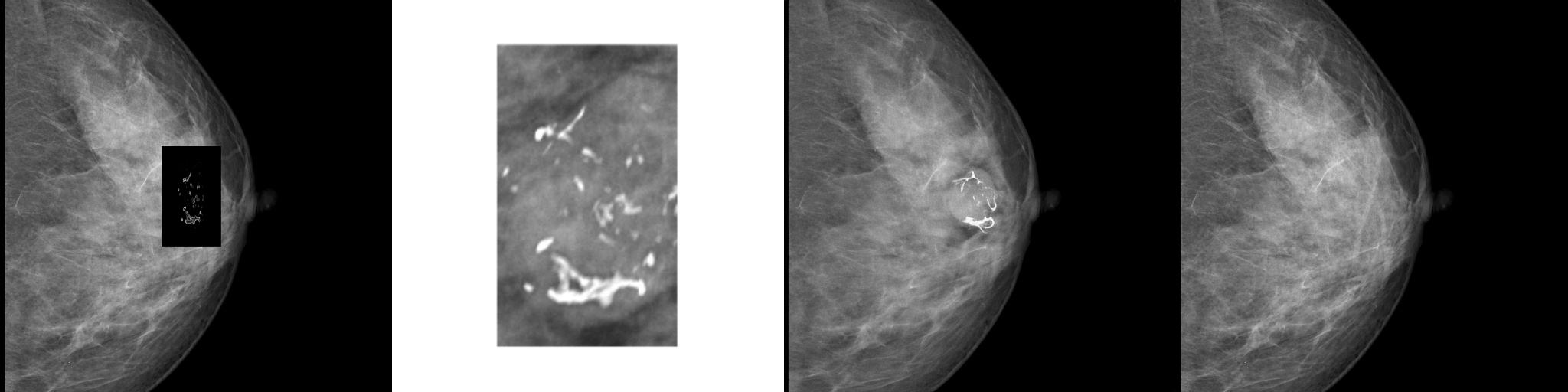} \\
    \includegraphics[width=\columnwidth]{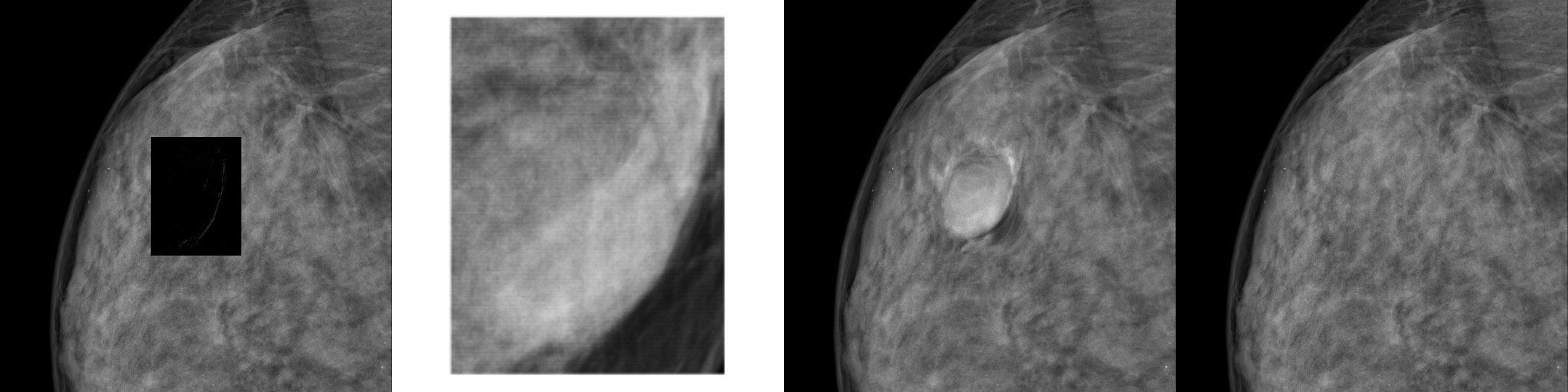} \\
    \includegraphics[width=\columnwidth]{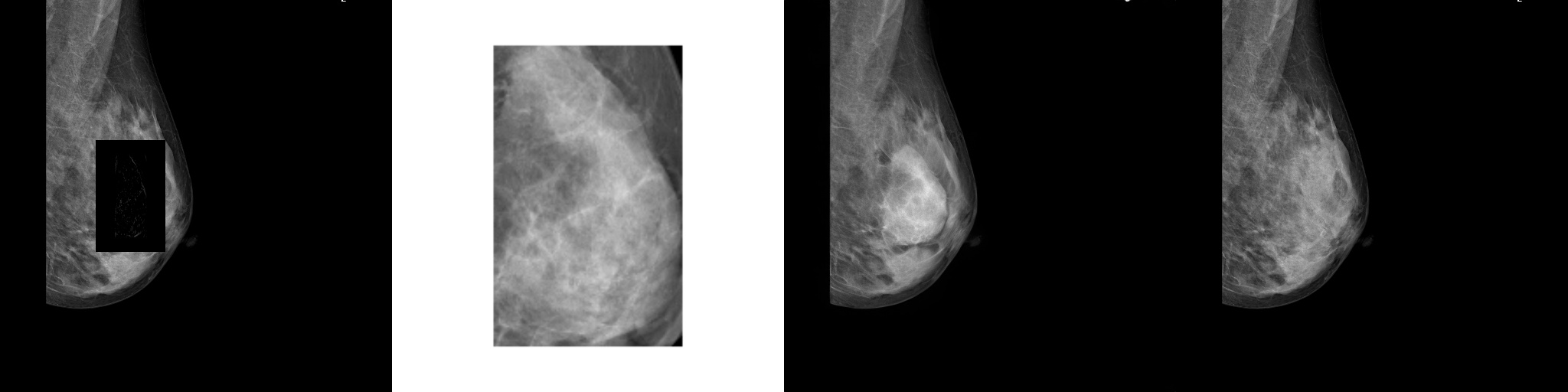} \\
    \includegraphics[width=\columnwidth]{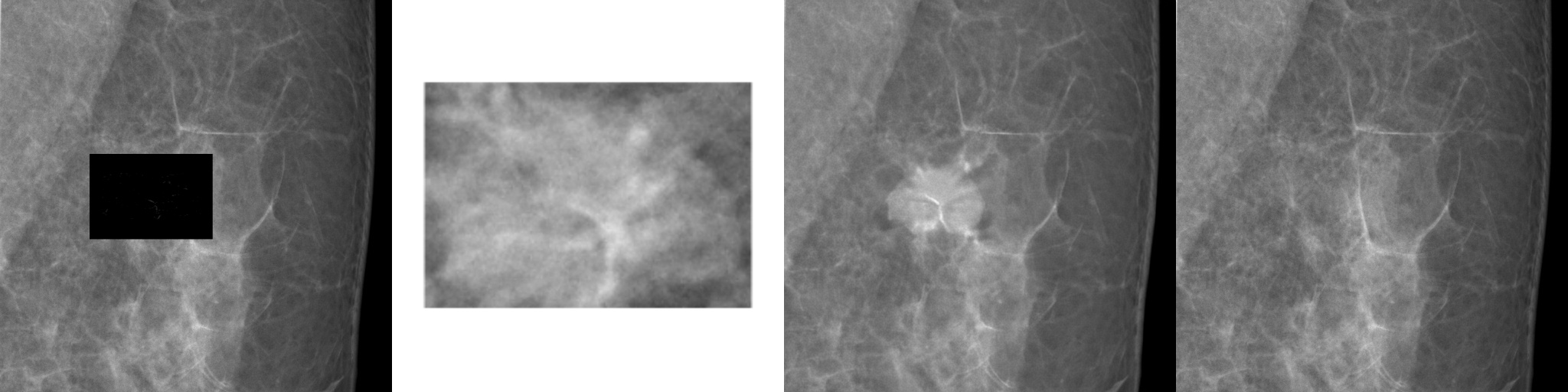} \\
    \includegraphics[width=\columnwidth]{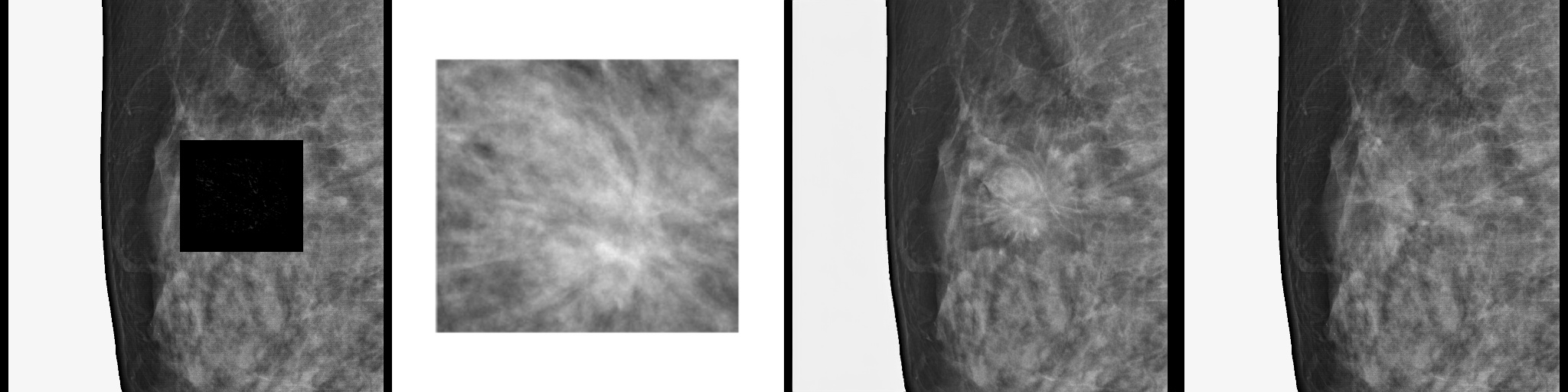} \\
    \includegraphics[width=\columnwidth]{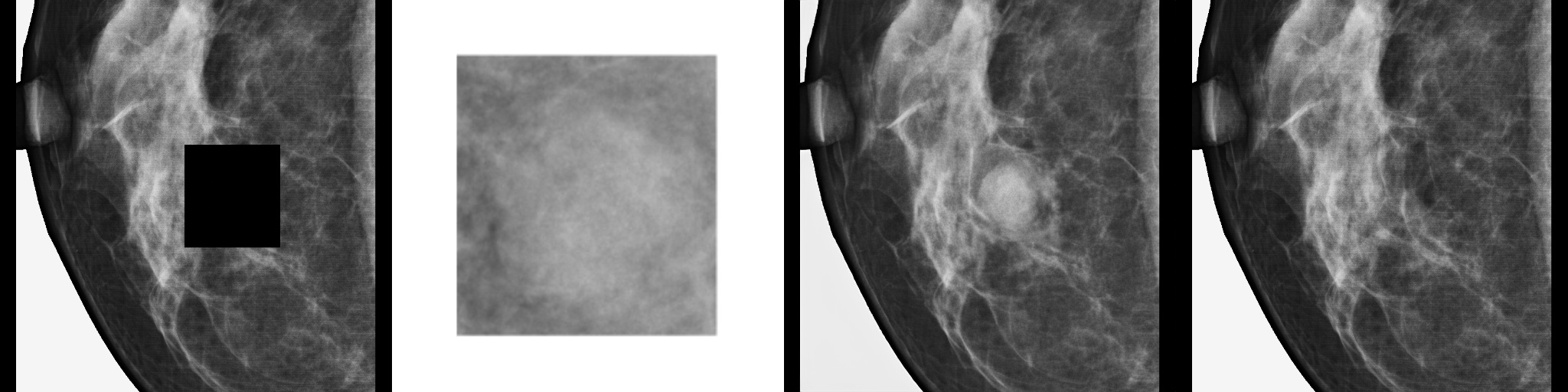} \\
    \end{minipage}
    \caption[Anomaly insertion qualitative results.]{Anomaly insertion results. \textbf{AnydoorMed} inserts the reference anomaly (second column), guided by the context and high-frequency map context (first column), into the healthy mammography scan (fourth column), producing the composited result (third column). The inpainted anomalies preserve some of the features present in the reference image, such as calcifications (first row) or spiculations (fifth row). For all examples, the anomaly was composited semantically in the destination scan, within the breast tissue.
    }
    \label{fig:insert_med}
\end{figure*}

\begin{table*}
    \centering
    \begin{minipage}{\textwidth}
        \centering
        \begin{tabular}{lcccc}
            & \multicolumn{4}{c}{\textbf{Reinsertion}} \\
            \cmidrule(r){2-5}
            \textbf{Method} & FID\textdownarrow & LPIPS\textdownarrow & CLIP-I\textuparrow & DINOv2$\uparrow$ \\
            \midrule
            copy\&paste & n/a & n/a & n/a & n/a \\
            AnyDoor~\cite{chen2023anydoor} & 6.80 & 0.19 & 89 & 34 \\
            AnydoorMed (ours) & \textbf{1.83}$\pm$0.16 & \textbf{0.06}$\pm$0.01 & \textbf{92.4}$\pm$0.4 & \textbf{45.6}$\pm$0.4 \\
            \bottomrule
        \end{tabular}
    \end{minipage}

    \vspace{0.5cm}
    
    \begin{minipage}{\textwidth}
        \centering
        \begin{tabular}{lcccc}
            & \multicolumn{4}{c}{\textbf{Replacement}} \\
            \cmidrule(r){2-5}
            \textbf{Method} & FID\textdownarrow & LPIPS\textdownarrow & CLIP-I\textuparrow & DINOv2\textuparrow \\
            \midrule
            copy\&paste & 4.39 & 0.08 & n/a & n/a \\
            AnyDoor~\cite{chen2023anydoor} & 7.42 & 0.20 & 88 & 32 \\
            AnydoorMed (ours) & \textbf{2.78}$\pm$0.21 & \textbf{0.07}$\pm$0.01 & \textbf{90.3}$\pm$0.3 & \textbf{39.3}$\pm$1 \\
            \bottomrule
        \end{tabular}
    \end{minipage}
    
    \vspace{0.5cm}
    
    \begin{minipage}{\textwidth}
        \centering
        \begin{tabular}{lcccc}
            & \multicolumn{4}{c}{\textbf{Insertion}} \\
            \cmidrule(r){2-5}
            \textbf{Method} & FID\textdownarrow & LPIPS\textdownarrow & CLIP-I\textuparrow & DINOv2\textuparrow \\
            \midrule
            copy\&paste & \textbf{4.64} & 0.10 & n/a & n/a \\
            AnyDoor~\cite{chen2023anydoor} & 7.93 & 0.21 & 89 & 31 \\
            AnydoorMed (ours) & 4.78$\pm$ 0.14 & \textbf{0.08}$\pm$ 0.01 & \textbf{89.9}$\pm$0.3 & \textbf{38.6}$\pm$0.3 \\
            \bottomrule
        \end{tabular}
    \end{minipage}
    \caption[Comparison of realism metrics across reinsertion, replacement, and insertion experiments.]{Comparison of realism metrics across reinsertion, replacement, and insertion experiments. Results for AnydoorMed are averaged across three models trained on distinct seeds. Standard deviation is also reported for the presented method.}
    \label{tab:full-camera-realism-dinov2}
\end{table*}

\begin{figure*}[htbp]
    \centering
    \begin{minipage}{0.87\textwidth}
    \begin{tabularx}{0.96\columnwidth}{@{}>{\centering\arraybackslash}X>{\centering\arraybackslash}X>{\centering\arraybackslash}X>{\centering\arraybackslash}X@{}}
    Context & Reference & Generation (reinsertion) & Original \\
    \end{tabularx}
    \includegraphics[width=\columnwidth]{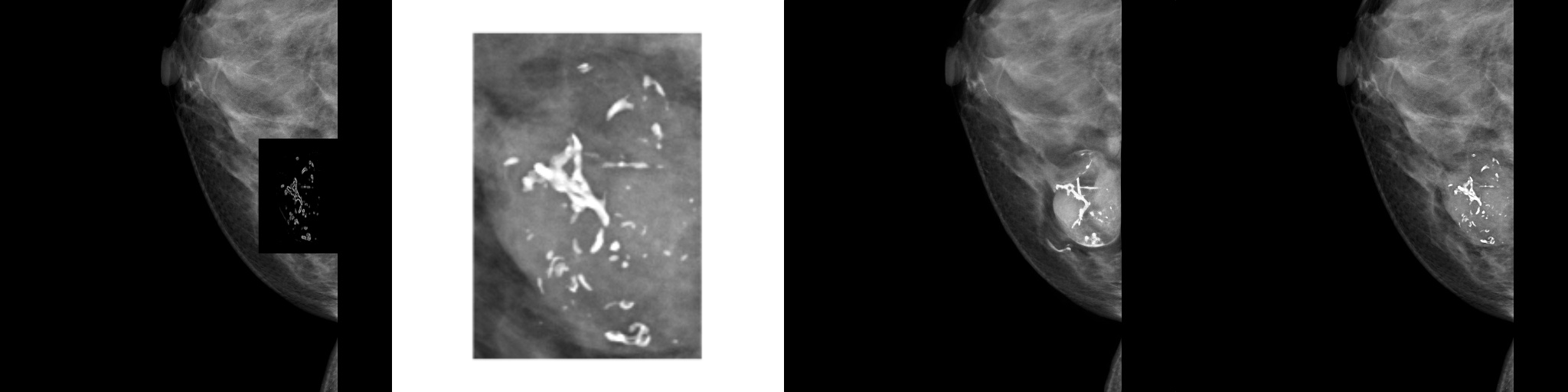} \\
    \includegraphics[width=\columnwidth]{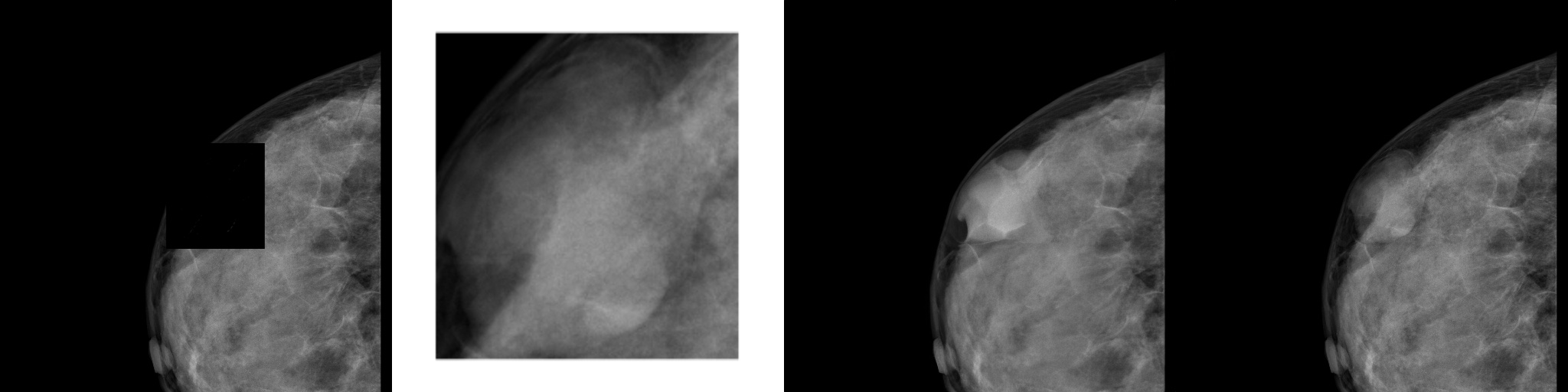} \\
    \includegraphics[width=\columnwidth]{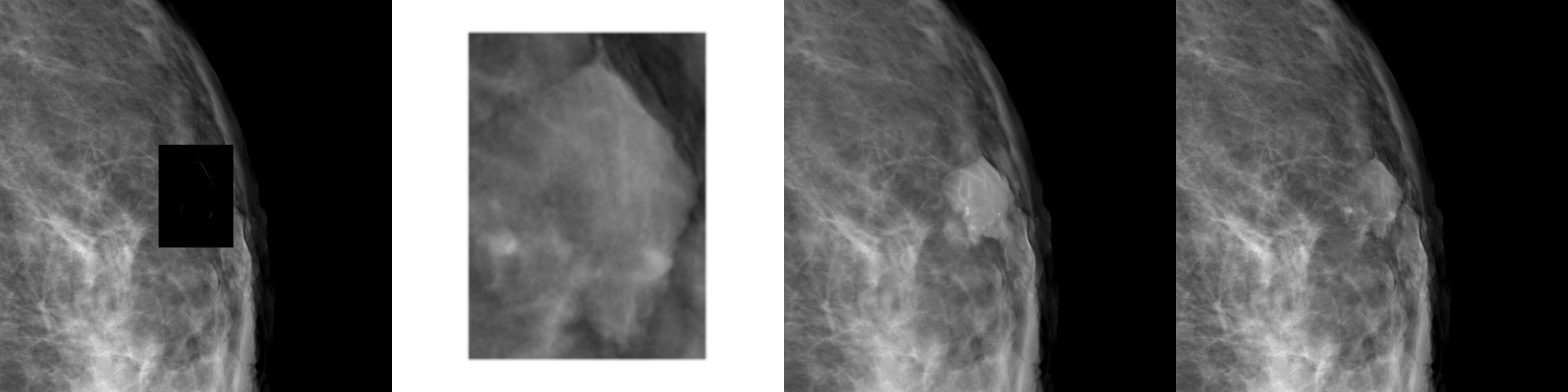} \\
    \includegraphics[width=\columnwidth]{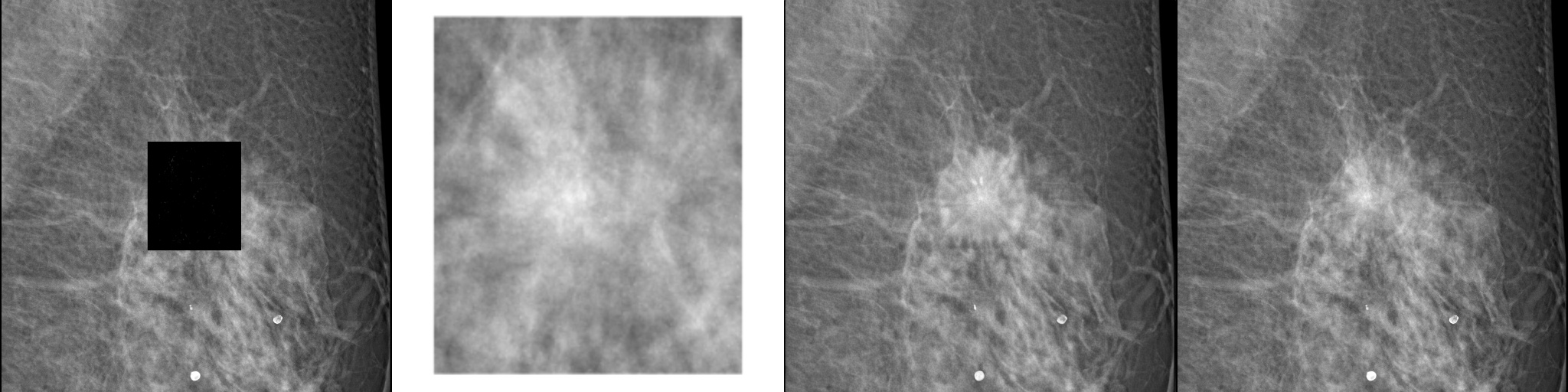} \\
    \end{minipage}
    \caption[Anomaly reinsertion qualitative results.]{Anomaly reinsertion results. \textbf{AnydoorMed} reinserts the anomaly (second column), guided by the context and high-frequency map context (first column), into the mammography scan (fourth column), producing the composited result (third column). This is done by removing the anomaly from the scan and using it as a reference. The inpainted anomalies preserve some of the features present in the reference image, such as calcifications (first row) or spiculations (last row). The original and reinserted anomalies are similar, yet not identical, which suggests the model is not performing plain copy\&paste.
    }
    \label{fig:reinsert_med}
\end{figure*}

\begin{figure*}[htbp]
    \centering
    \begin{minipage}{0.87\textwidth}
    \begin{tabularx}{0.96\columnwidth}{@{}>{\centering\arraybackslash}X>{\centering\arraybackslash}X>{\centering\arraybackslash}X>{\centering\arraybackslash}X@{}}
    Context & Reference & ~~Generation (replacement) & Original \\
    \end{tabularx}
    \includegraphics[width=\columnwidth]{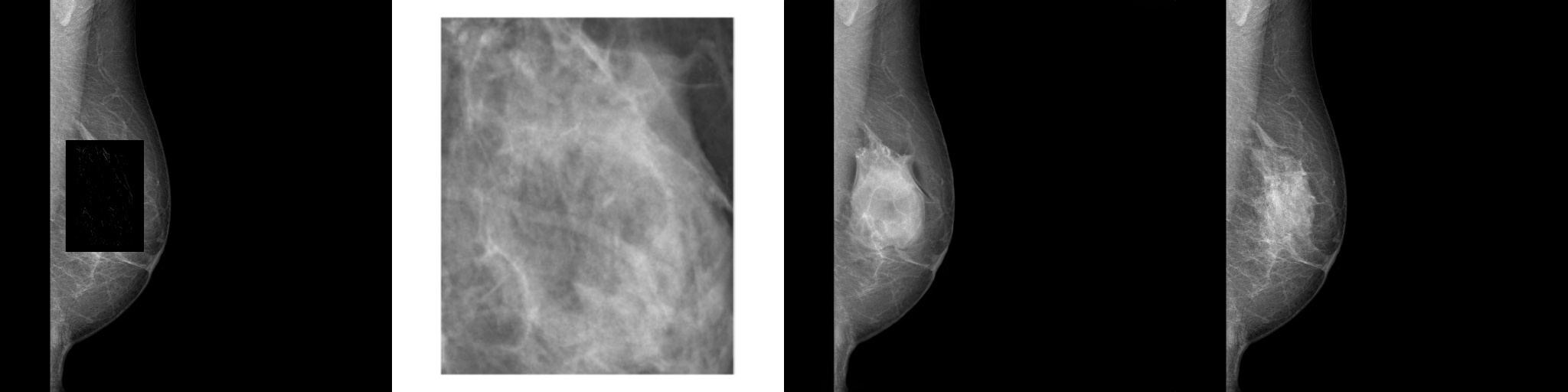} \\
    \includegraphics[width=\columnwidth]{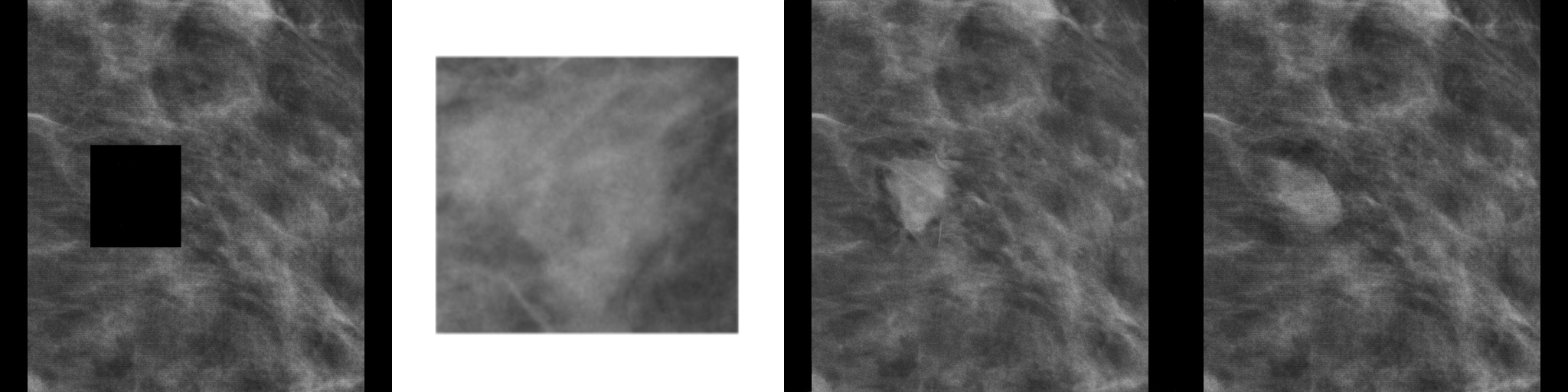} \\
    \includegraphics[width=\columnwidth]{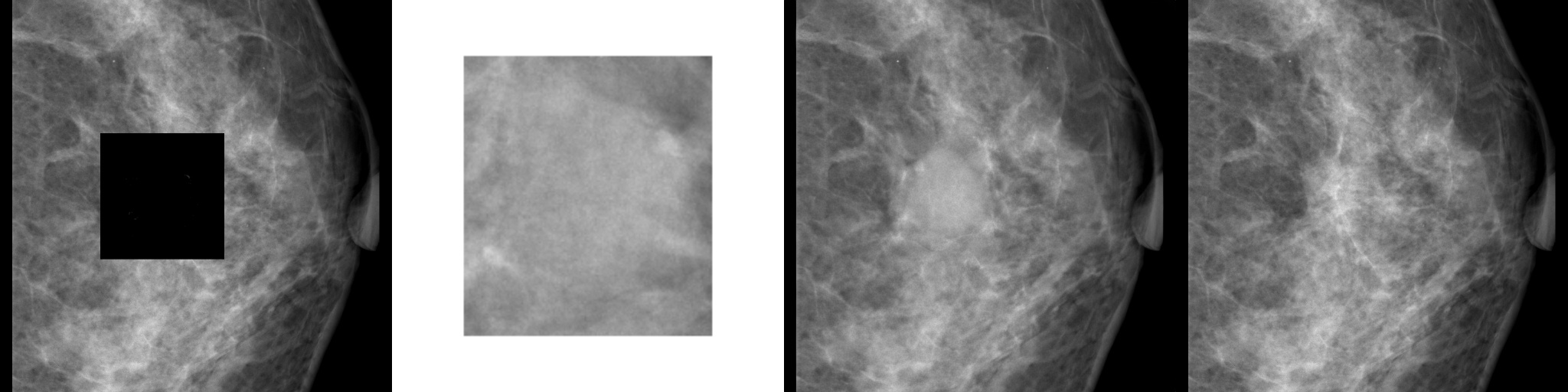} \\
    \includegraphics[width=\columnwidth]{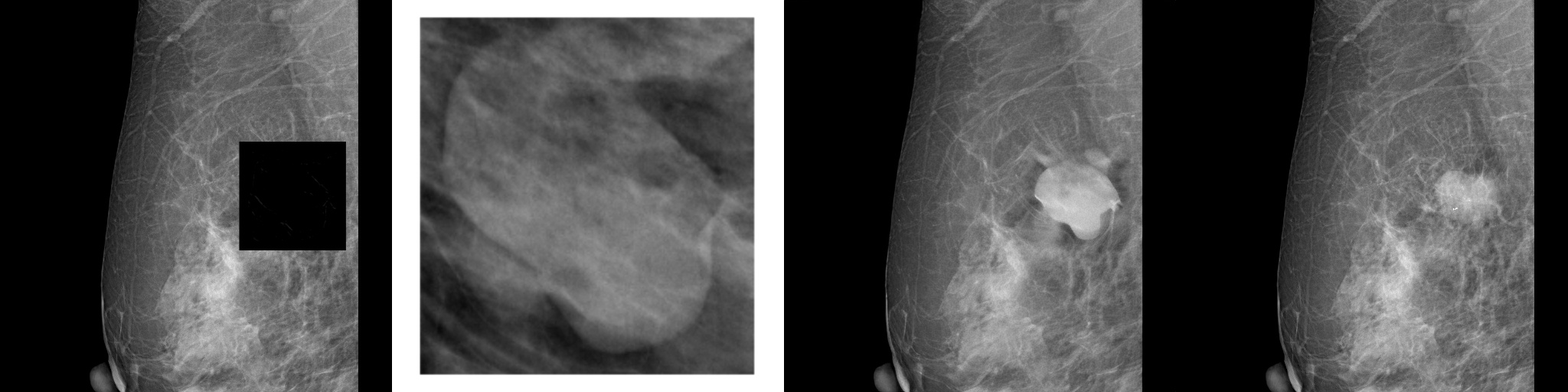} \\
    \end{minipage}
    \caption[Anomaly replacement qualitative results.]{Anomaly replacement results. \textbf{AnydoorMed} replaces the anomaly from the original scan (fourth column), with the reference anomaly (second column), guided by the context and high-frequency map context (first column), producing the composited result (third column). This is done by removing the anomaly from the scan and using a similarly-sized reference as a condition. The inpainted anomalies preserve some of the features present in the reference image, such as the ``excavation'' from the first row. All generated scans look highly realistic, with anomalies being semantically blended within the breast tissue.
    }
    \label{fig:replace_med}
\end{figure*}

\subsection{Qualitative results}
The qualitative results illustrate the model's ability to integrate anomalies into mammography scans while maintaining naturalness and anatomical consistency. In all tasks, whether inserting, reinserting, or replacing anomalies, the model effectively blended the anomalies into the breast tissue using context and high-frequency maps, ensuring a realistic outcome.

In the insertion task, multiple reference anomalies were inserted into healthy mammography scans, with each anomaly retaining key features such as calcifications and spiculations. These anomalies were inserted semantically within the breast tissue, as demonstrated in several examples (see \autoref{fig:insert_med}). For reinsertion, previously removed anomalies were reintroduced into the scans. While the reinsertion anomalies closely resembled their originals, some subtle differences were evident, suggesting that the model did not simply perform a copy-paste operation. The reinserted anomalies were smoothly integrated into the scans, maintaining overall realism (see \autoref{fig:reinsert_med}). In the replacement task, reference anomalies of similar size replaced existing anomalies in the scans. These replacements preserved key features, such as "excavation," and the anomalies were seamlessly blended into the breast tissue, with all examples looking highly realistic (see \autoref{fig:replace_med}).

\paragraph{Realism of the inpainting}
AnyDoorMed consistently outperforms {AnyDoor} and copy\&paste. Insertion results show that anomalies were seamlessly integrated into healthy scans, with low FID (4.89) and LPIPS (0.08) scores, indicating high realism. Reinsertion, acting as a sanity check, demonstrated that the reintroduced anomalies were realistic and similar to the original, with {AnyDoorMed} achieving the best FID (2.14) and LPIPS (0.05) scores. For replacement, {AnyDoorMed} effectively swapped anomalies while preserving scan integrity, with the lowest FID (3.06) and LPIPS (0.07) scores. {AnyDoorMed} generates highly realistic and semantically consistent anomalies across all tasks.

\chapter{Discussion}
\label{chapter:discussion}

\section{Strengths}

This work introduces two novel methods for reference-guided counterfactual generation across distinct perceptual domains, in autonomous driving and medical image analysis. It demonstrates the versatility of adapting inpainting foundation models to diverse modalities using a simple and data-efficient conditioning mechanism. Through this adaptation, both methods achieve fine-grained control, multimodal coherence, and strong semantic consistency without the need for handcrafted assets.

\textbf{MObI} enables realistic, 3D-conditioned object insertion across camera and lidar modalities in complex urban scenes captured by autonomous vehicles. Leveraging the expressive capacity of latent diffusion models, it performs high-fidelity object insertions while maintaining consistency across different viewpoints and sensing modalities. A particular strength of MObI lies in its ability to produce geometrically and semantically coherent results across sensor streams. This capability is especially valuable in safety-critical applications where synthetic multimodal data is needed for robust evaluation and training.

\textbf{AnydoorMed} showcases the proposed framework's adaptability to the medical imaging domain, with a specific focus on anomaly inpainting in mammography scans. By enabling the synthesis of perceptually plausible anomalies at precise spatial locations, the method provides a powerful tool for counterfactual data generation in medical imaging. This capability could aid in improving the robustness of diagnostic systems, particularly in underrepresented or edge-case scenarios. Both methods demonstrate state-of-the-art performance on realism metrics relative to their respective baselines, underscoring the effectiveness and generality of the approach across domains as varied as autonomous driving and digital mammography.

\section{Limitations}

\begin{figure*}[ht!]
\begin{minipage}{\textwidth}
\centering
\setlength{\tabcolsep}{1pt}
\footnotesize
\begin{tabular}{ccccc}
Original & Reference & Editied (C) & Edited (R) depth & Edited (R) intensity \\
\includegraphics[width=0.18\textwidth]{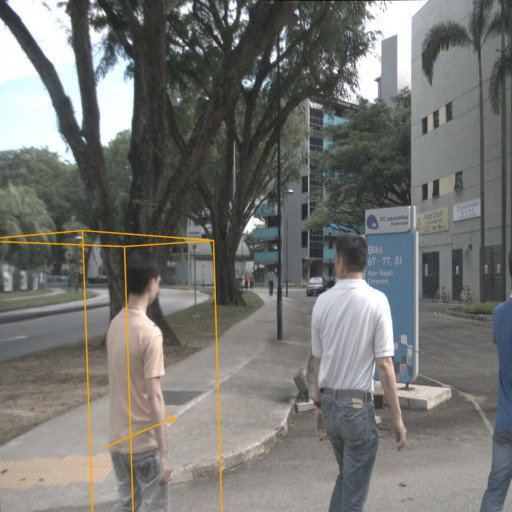} &
\includegraphics[width=0.18\textwidth]{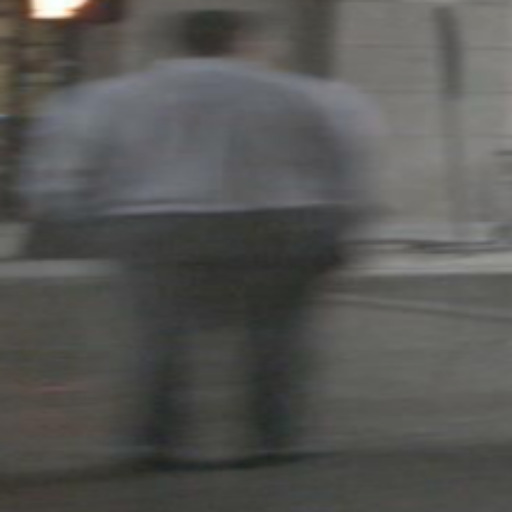} &
\includegraphics[width=0.18\textwidth]{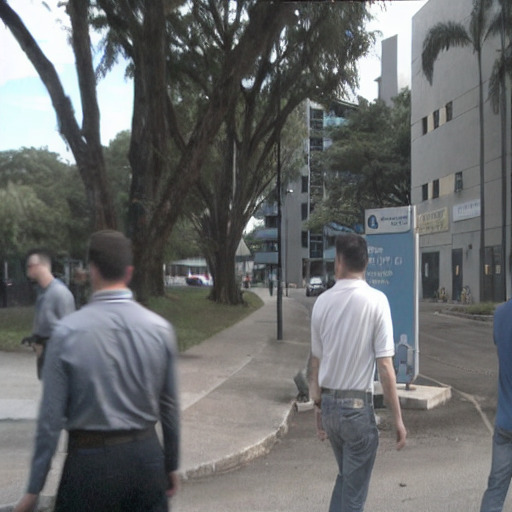} &
\includegraphics[width=0.18\textwidth]{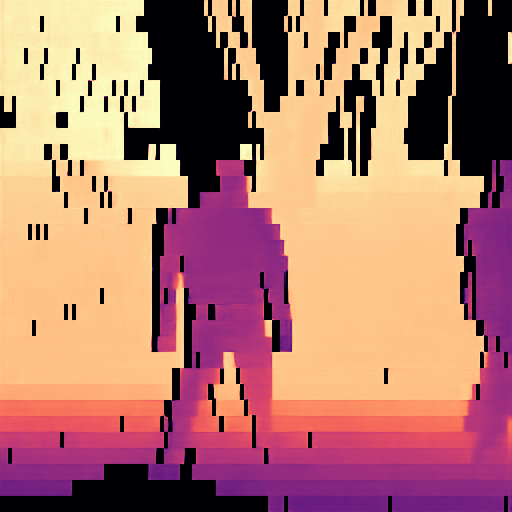} &
\includegraphics[width=0.18\textwidth]{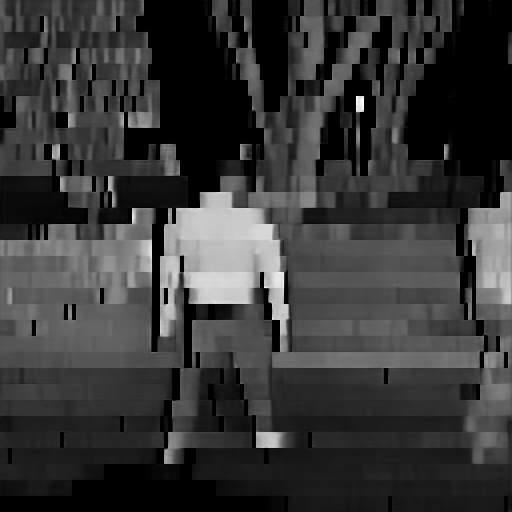} \\
\includegraphics[width=0.18\textwidth]{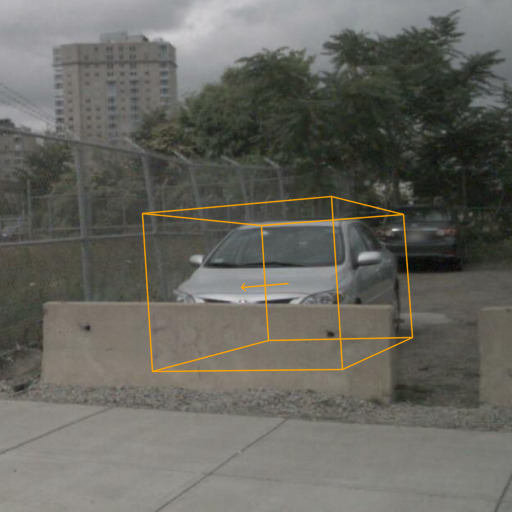} &
\includegraphics[width=0.18\textwidth]{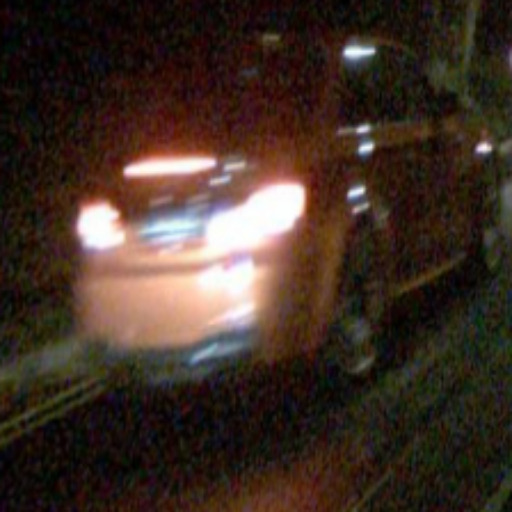} &
\includegraphics[width=0.18\textwidth]{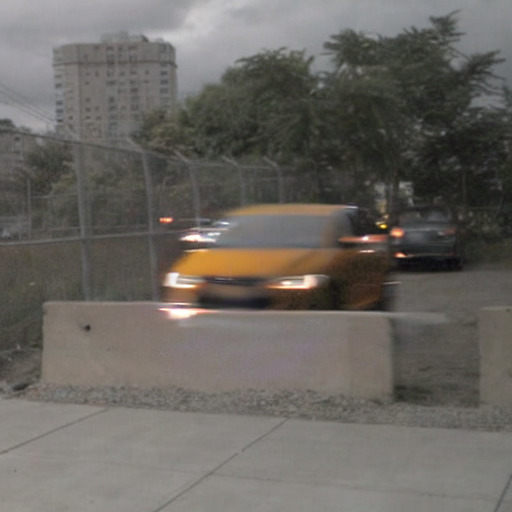} &
\includegraphics[width=0.18\textwidth]{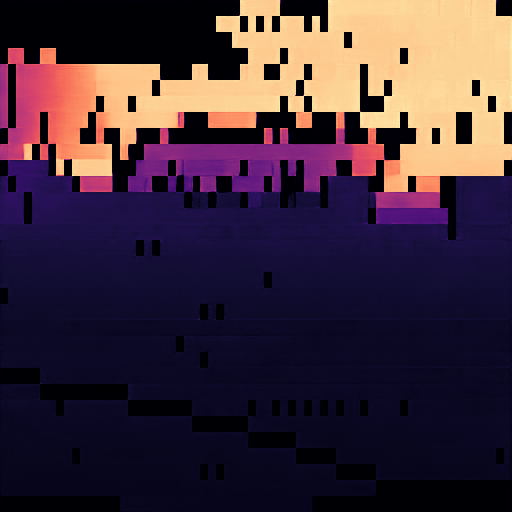} &
\includegraphics[width=0.18\textwidth]{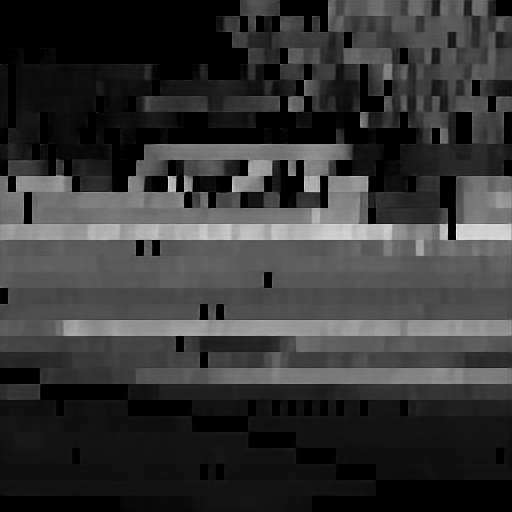} \\
\includegraphics[width=0.18\textwidth]{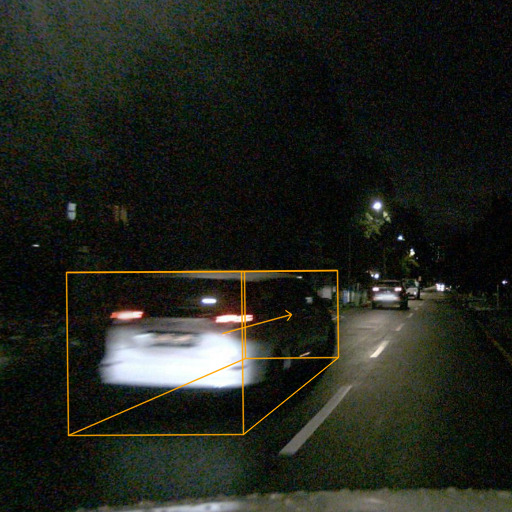} &
\includegraphics[width=0.18\textwidth]{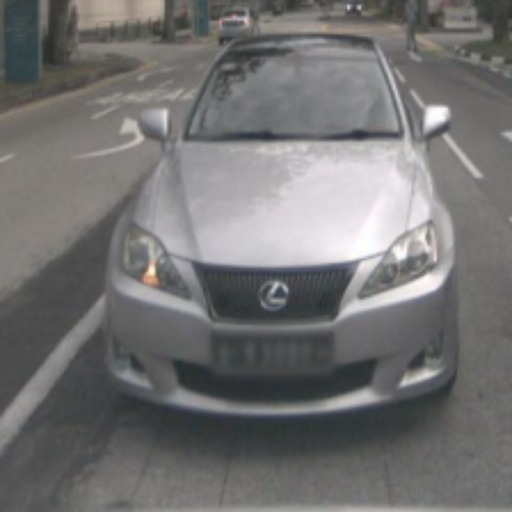} &
\includegraphics[width=0.18\textwidth]{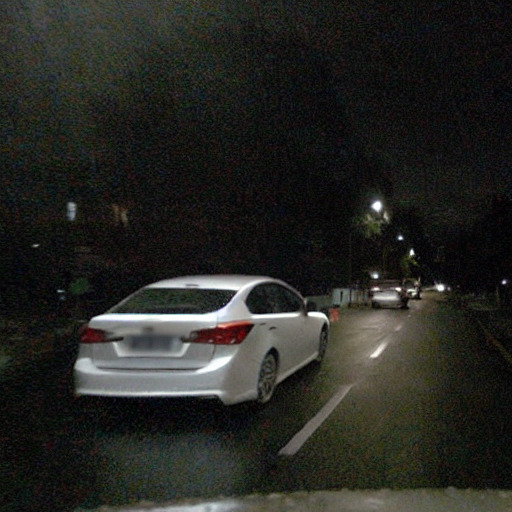} &
\includegraphics[width=0.18\textwidth]{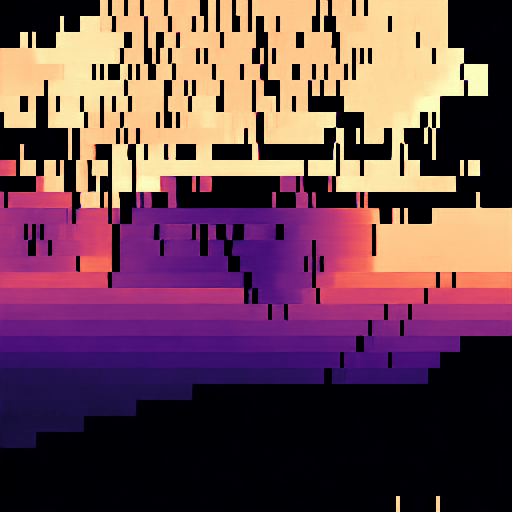} &
\includegraphics[width=0.18\textwidth]{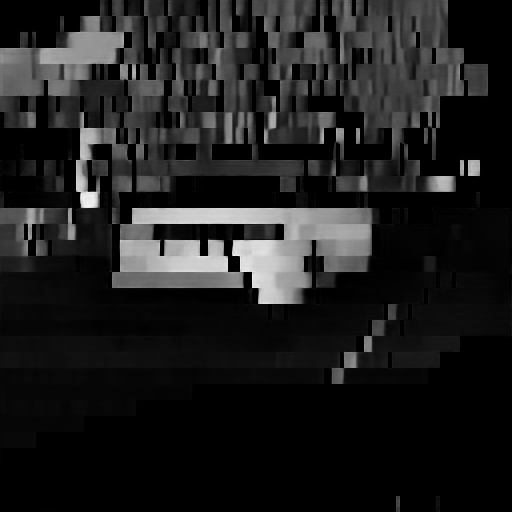} \\
\end{tabular}
\end{minipage}
\caption[Object replacement results using hard references.]{Object replacement results using hard references (different weather conditions or time of day, occlusions, etc.). MObI can successfully insert these hard references in the target bounding box. However, the quality in these examples is unsatisfactory. From top to bottom: a new pedestrian is hallucinated, the inserted car shows too much motion blur, and the lightning is not coherent with the overall scene.}
\label{fig:inpainting-hard-suppl}
\end{figure*}

\begin{figure*}[htbp]
    \centering
    \begin{minipage}{0.9\textwidth}
    \begin{tabularx}{0.96\columnwidth}{@{}>{\centering\arraybackslash}X>{\centering\arraybackslash}X>{\centering\arraybackslash}X>{\centering\arraybackslash}X@{}}
    ~Context & Reference & Edited & Original \\
    \end{tabularx}
    \includegraphics[width=\columnwidth]{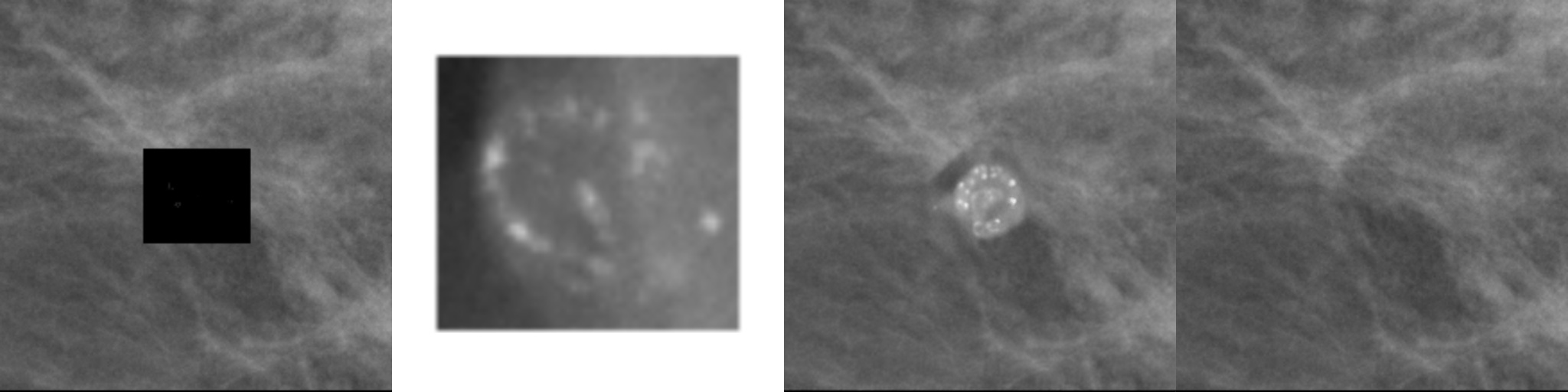} \\
    \includegraphics[width=\columnwidth]{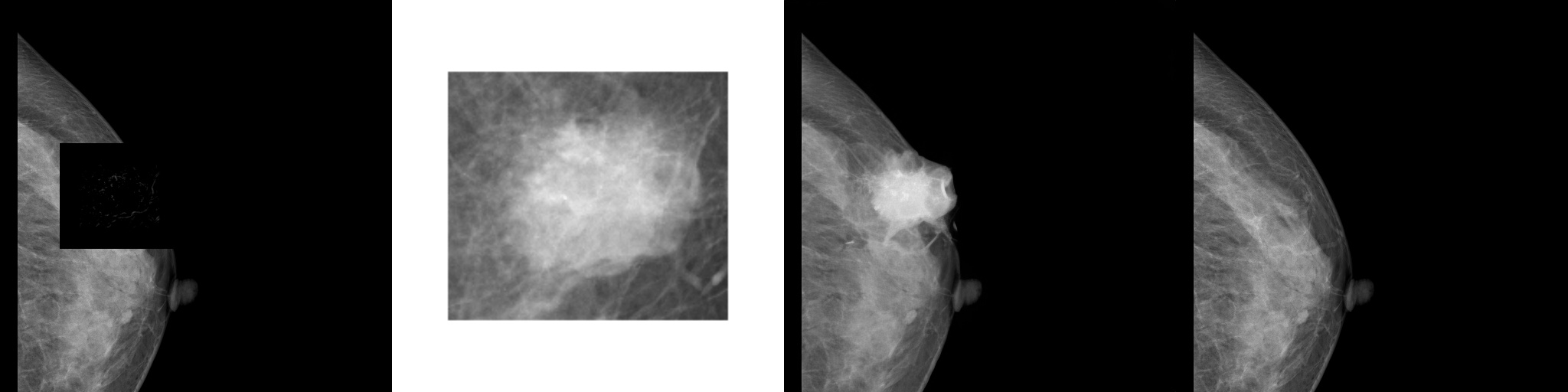} \\
    \includegraphics[width=\columnwidth]{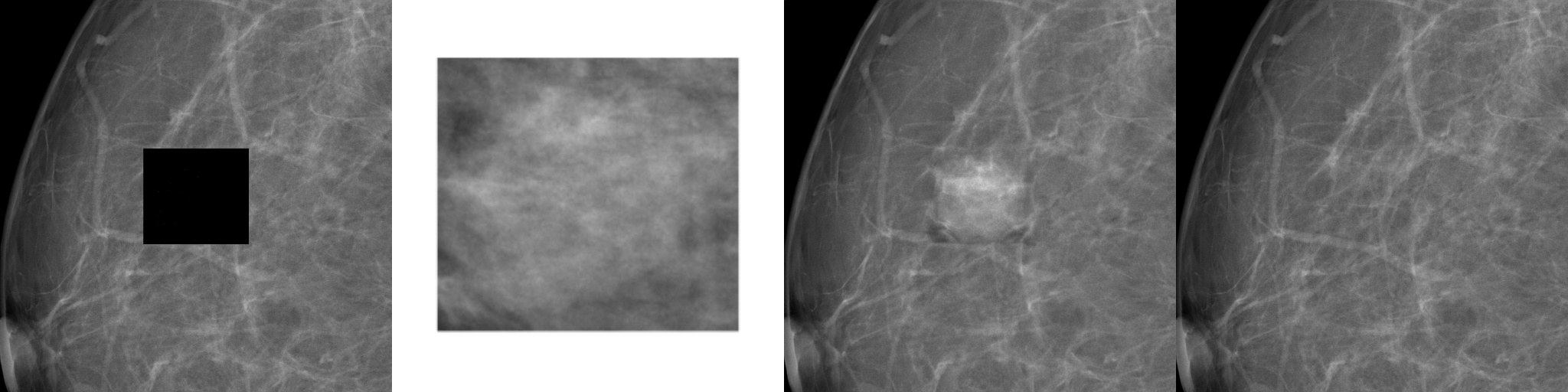} \\
    \includegraphics[width=\columnwidth]{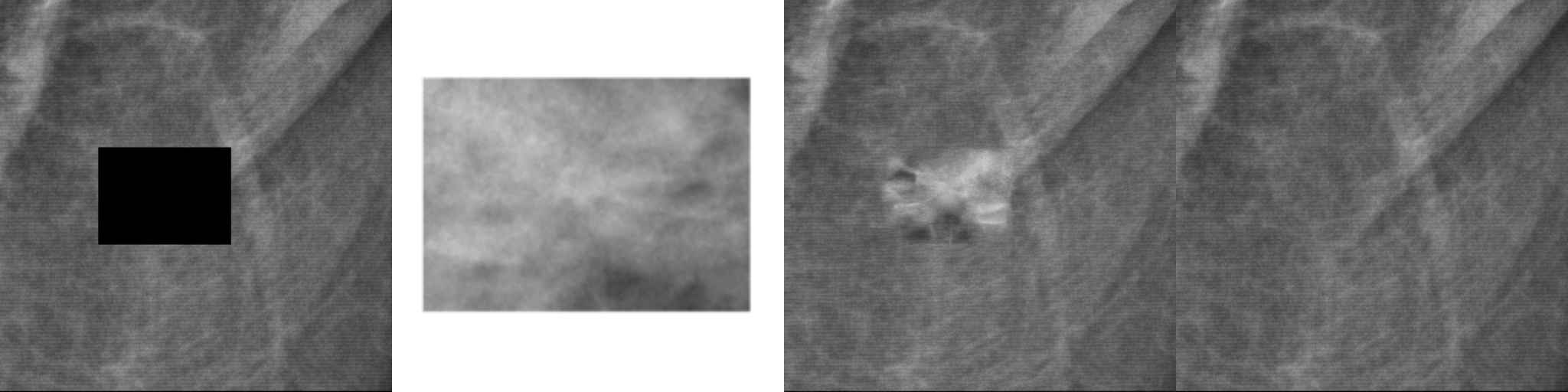} \\
    \end{minipage}
    \caption[AnydoorMed failure cases.]{Anomaly insertion results. \textbf{AnydoorMed} inserts the reference anomalies (second column), guided by the context and high-frequency collage (first column), into the healthy mammography scan (fourth column), producing the composited result (third column). However, these examples illustrate failure cases. From top to bottom: the inserted anomaly does not closely replicate the microcalcifications from the reference image (which may be undesirable in certain scenarios); the inpainting produces an anatomically implausible result due to the bounding box being placed primarily outside the breast tissue; and finally, the last two examples exhibit visible copy-and-paste artefacts.
    }
    \label{fig:failure med}
\end{figure*}

\subsection{MObI limitations}

While MObI can generate coherent objects across viewpoints, as demonstrated in \autoref{fig:suppl:rotation_results}, several limitations affect its robustness and generalisability. One key issue arises when the inserted object's location is in stark semantic conflict with the surrounding scene context. For instance, placing a truck on a pedestrian pavement might result in implausible completions. This limits the method’s utility for generating deeply out-of-distribution (OOD) counterfactuals, particularly valuable for testing autonomous vehicles.

Another limitation stems from dataset bias. Since the model is fine-tuned on a relatively narrow domain, it may occasionally override the bounding box conditioning if the scene context imposes a stronger prior. For example, it can favour common object placements encountered during training (such as when the lane could dictate the car's orientation, not the bounding box conditioning). This rare behaviour reveals the influence of implicit priors inherited from the training distribution, which may hinder controlled counterfactual generation in unexpected scenarios.

Additionally, the current conditioning mechanism relies solely on a single bounding box. In complex scenes, this can lead to unintended alterations of background objects, particularly when there is significant spatial overlap with the edit mask. This limitation could be alleviated through more accurate instance-level segmentation, which is not readily available in datasets such as nuScenes~\cite{caesar2020nuscenes}. This highlights the need for high-quality pseudo-labelling or enriched annotations.

The model also struggles when provided with completely open-world reference images. In such cases, the diffusion process tends to revert to in-domain representations. For instance, a horse may be transformed into a brown car. This behaviour, illustrated in \autoref{fig:suppl:ood}, highlights the difficulty of extending the method to a truly open-world setting.

\subsection{AnydoorMed limitations}

\textbf{AnydoorMed} faces several challenges when applied to anomaly inpainting in the medical domain. Firstly, the model does not always accurately preserve the structure and visual characteristics of the reference anomaly. This can lead to deviations in shape, intensity, or scale. While this may be tolerable in some use cases, it reduces the fidelity of counterfactual examples for tasks that require high clinical precision.

Secondly, artefacts arising from a copy-and-paste-like generation process can sometimes be observed in the output, particularly in complex tissue regions. These artefacts may degrade visual realism and, if used for training, could introduce shortcut opportunities for machine learning models to exploit non-semantic cues.

A critical limitation lies in the placement of the bounding box for insertion. If the bounding box extends beyond anatomically valid regions, such as outside breast tissue, the resulting counterfactual may be anatomically implausible, as illustrated in \autoref{fig:failure med}. In the medical domain, anatomical accuracy is paramount. Such implausible samples could degrade the training of diagnostic systems.

Moreover, the current approach lacks clinical interpretability and fine-grained control over lesion attributes such as type, severity, or BI-RADS category. This restricts the utility of \textbf{AnydoorMed} for generating realistic, targeted counterfactuals tailored to specific diagnostic tasks.

Finally, as with MObI, \textbf{AnydoorMed} is trained on a narrow distribution and may not generalise to other imaging modalities or anatomical regions. This highlights the importance of investigating multi-domain extensions that can handle a broader range of medical imaging tasks beyond mammography.

\section{Future work}

\begin{figure*}[t]
\centering
\begin{minipage}{0.48\textwidth}
\centering
\setlength{\tabcolsep}{1pt} % Adjust horizontal spacing
\renewcommand{\arraystretch}{0.8} % Adjust vertical spacing
\footnotesize
\begin{tabular}{ccc}
Original & Reference & Insertion (C) \\
\includegraphics[width=0.33\linewidth]{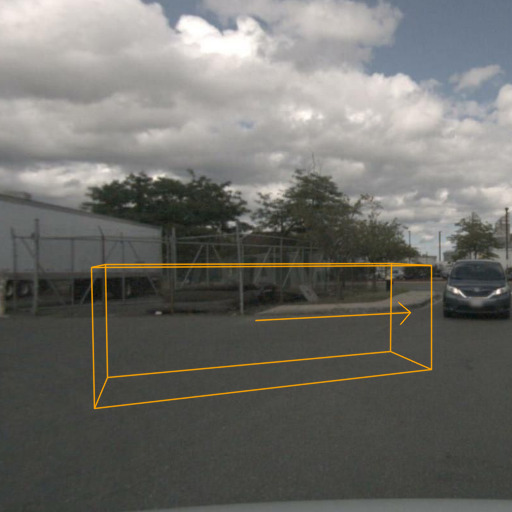} &
\includegraphics[width=0.33\linewidth]{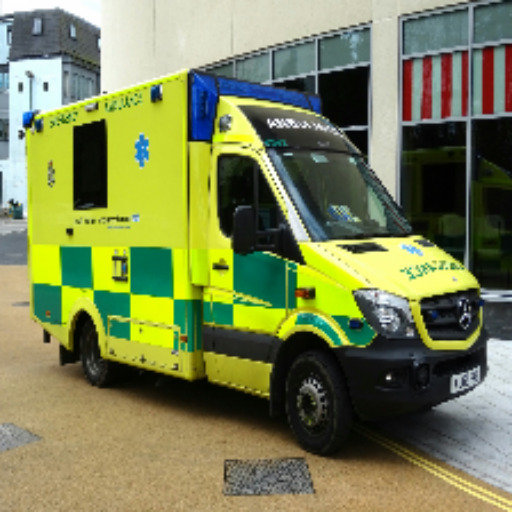} &
\includegraphics[width=0.33\linewidth]{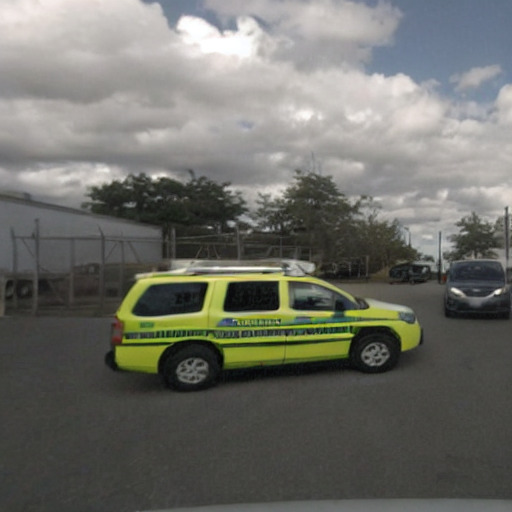} \\
\includegraphics[width=0.33\linewidth]{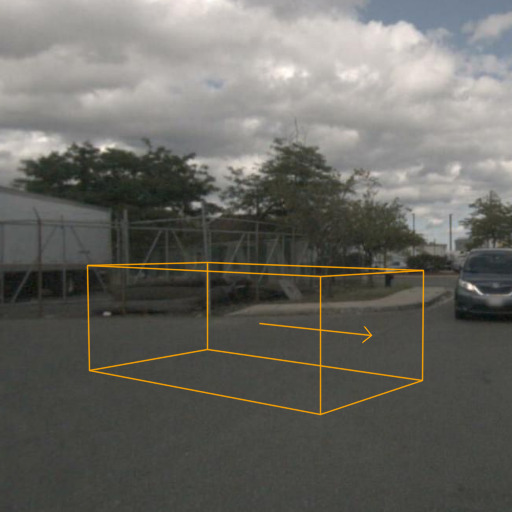} &
\includegraphics[width=0.33\linewidth]{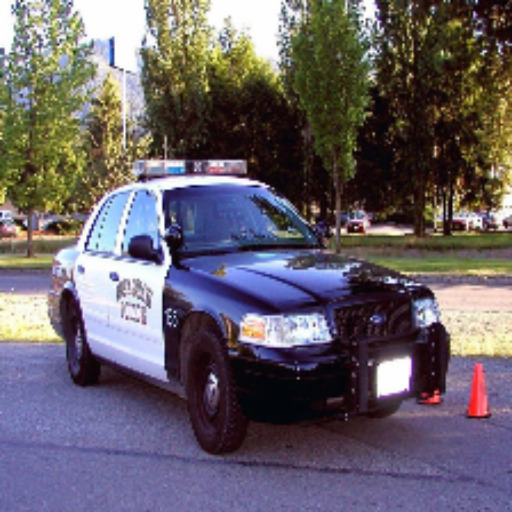} &
\includegraphics[width=0.33\linewidth]{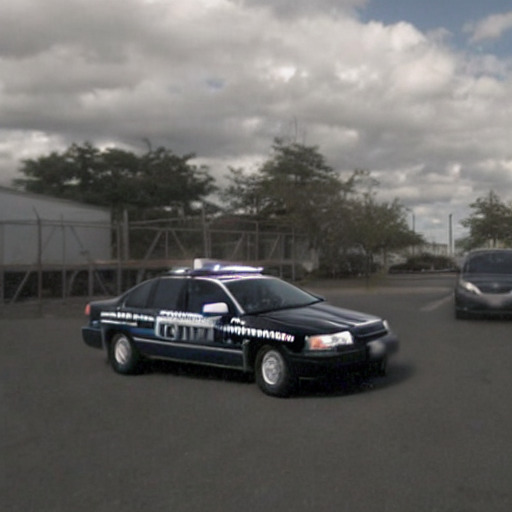} \\
\includegraphics[width=0.33\linewidth]{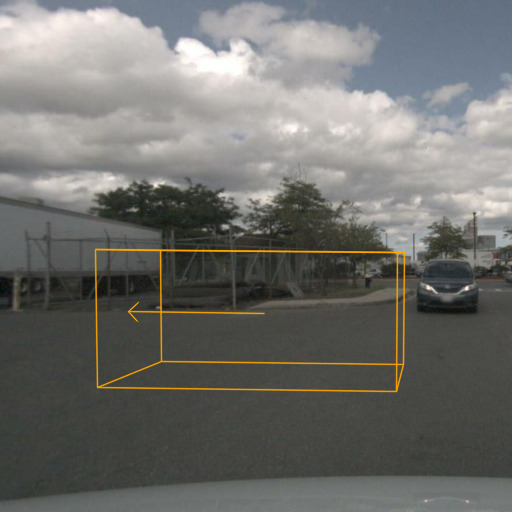} &
\includegraphics[width=0.33\linewidth]{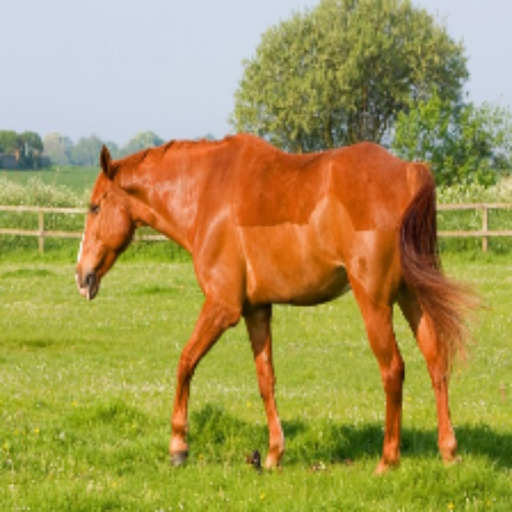} &
\includegraphics[width=0.33\linewidth]{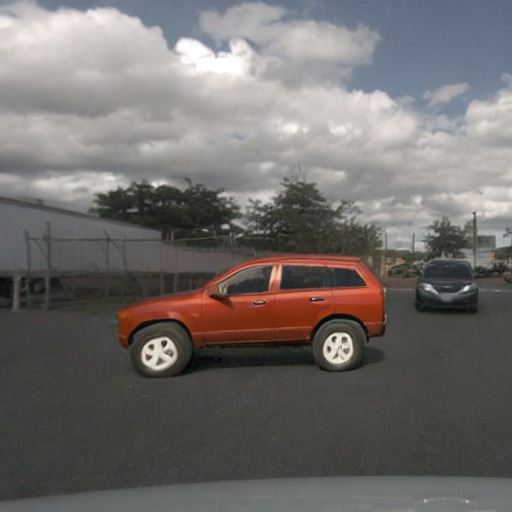} \\
\includegraphics[width=0.33\linewidth]{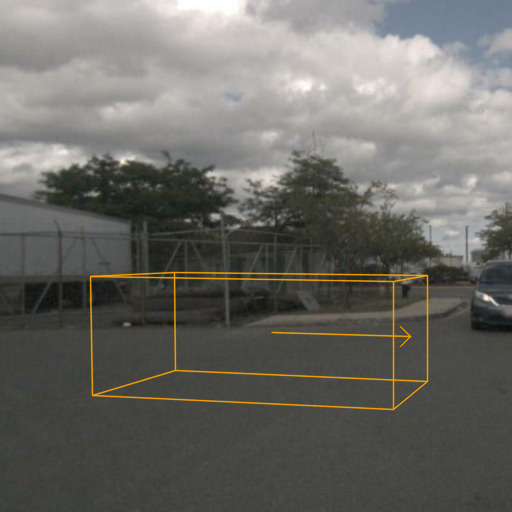} &
\includegraphics[width=0.33\linewidth]{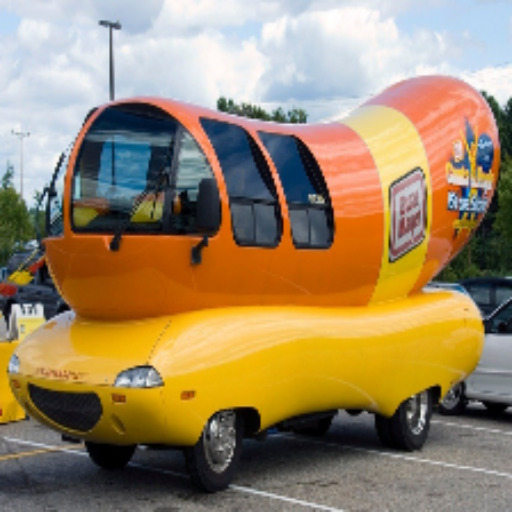} &
\includegraphics[width=0.33\linewidth]{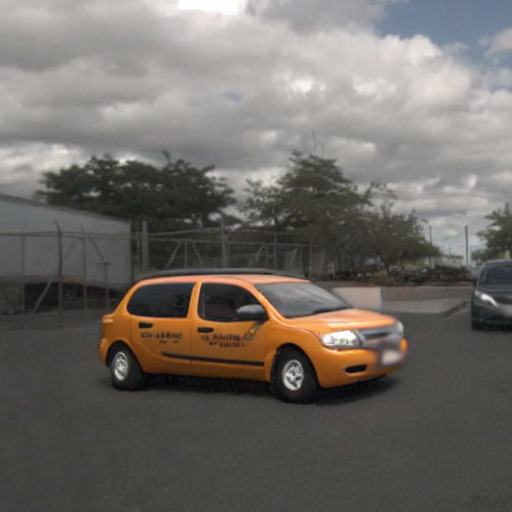}
\end{tabular}
\subcaption{}
\label{fig:suppl:ood_a}
\end{minipage}%
\hfill
\begin{minipage}{0.48\textwidth}
\centering
\setlength{\tabcolsep}{1pt} % Adjust horizontal spacing
\renewcommand{\arraystretch}{0.8} % Adjust vertical spacing
\footnotesize
\begin{tabular}{ccc}
Original & Reference & Replacement (C) \\
\includegraphics[width=0.33\linewidth]{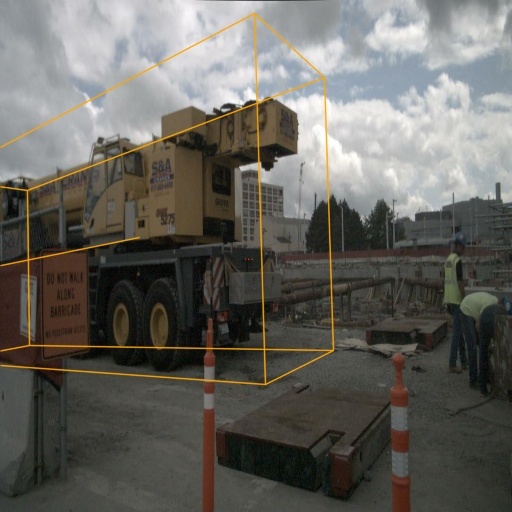} &
\includegraphics[width=0.33\linewidth]{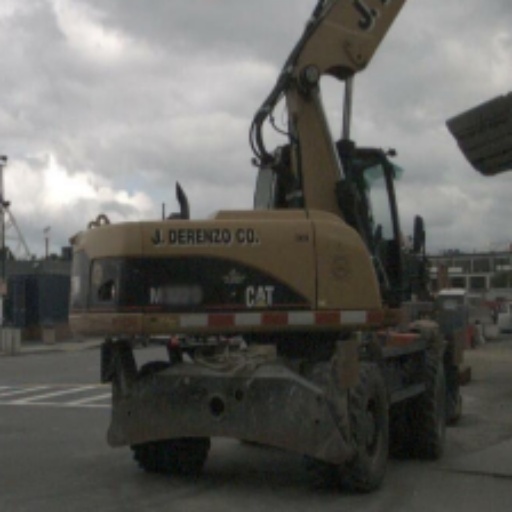} &
\includegraphics[width=0.33\linewidth]{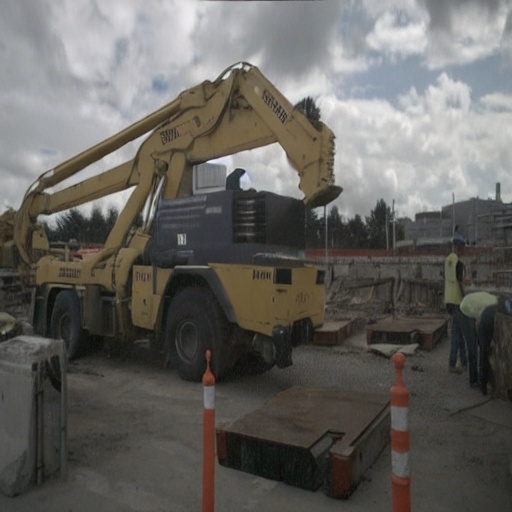} \\
\includegraphics[width=0.33\linewidth]{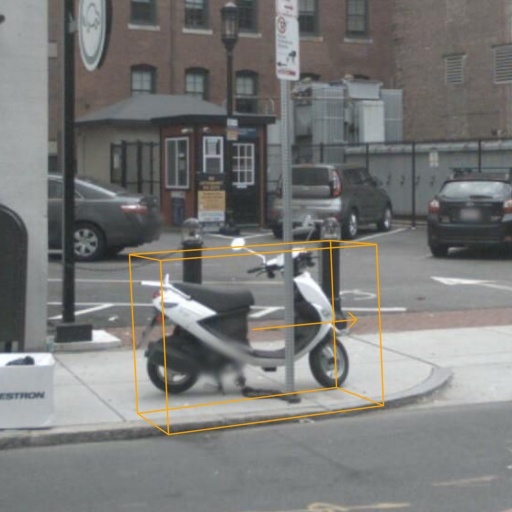} &
\includegraphics[width=0.33\linewidth]{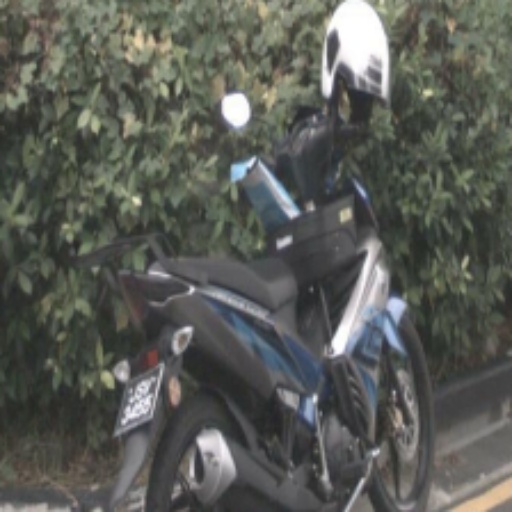} &
\includegraphics[width=0.33\linewidth]{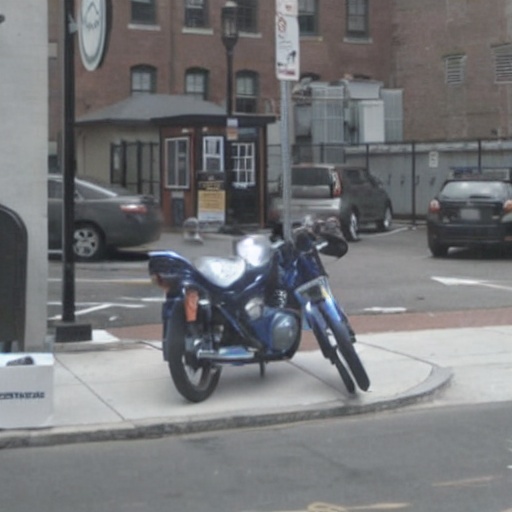} \\
\includegraphics[width=0.33\linewidth]{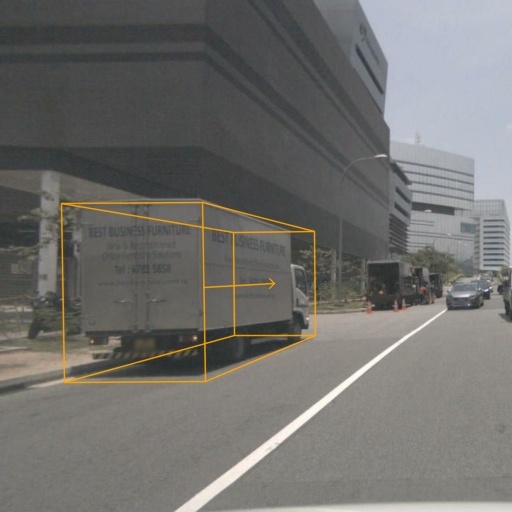} &
\includegraphics[width=0.33\linewidth]{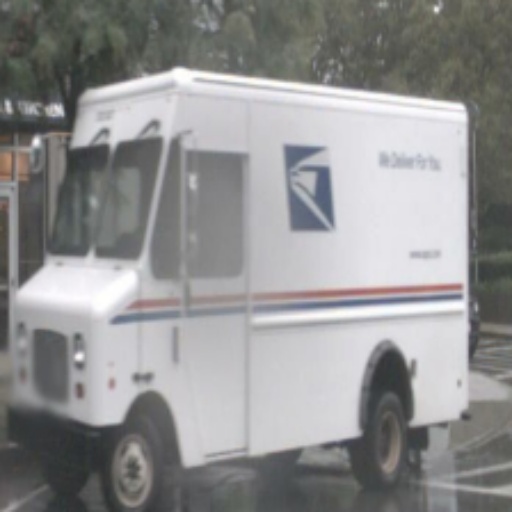} &
\includegraphics[width=0.33\linewidth]{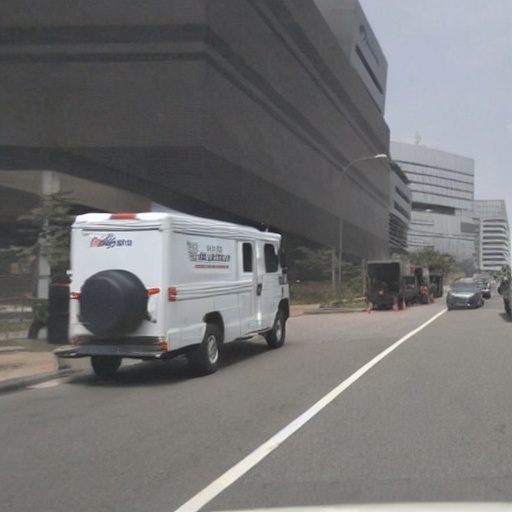} \\
\includegraphics[width=0.33\linewidth]{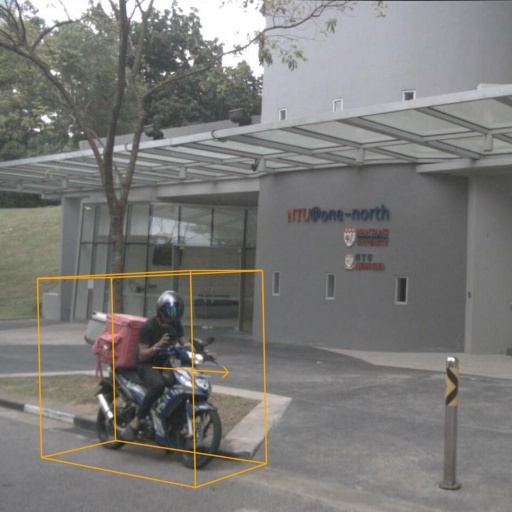} &
\includegraphics[width=0.33\linewidth]{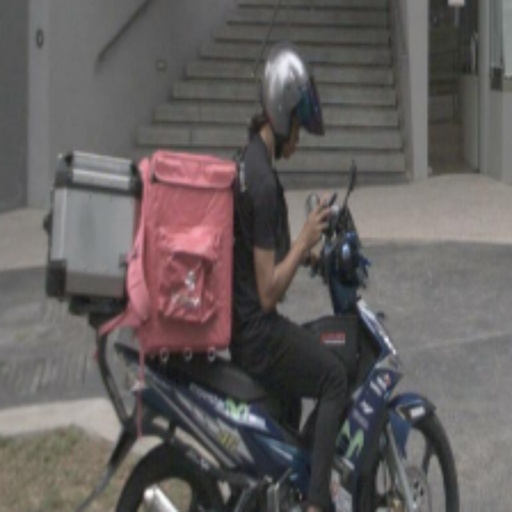} &
\includegraphics[width=0.33\linewidth]{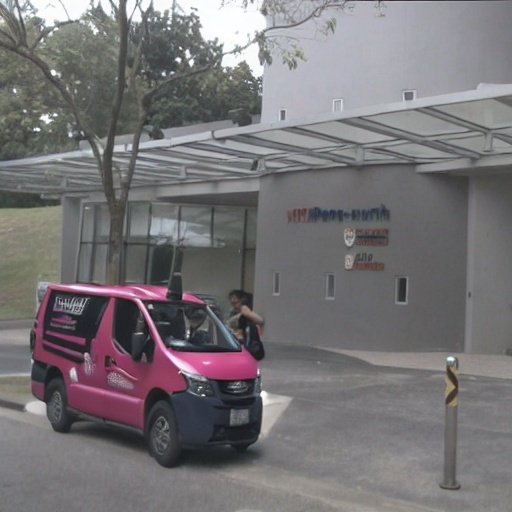}
\end{tabular}
\subcaption{}
\end{minipage}
\caption[Object insertion and replacement with out-of-domain and open-world references for MObI.]{Object insertion and replacement with out-of-domain and open-world references for MObI trained only on the pedestrian and car classes of nuScenes. (a) In the first two examples (top left), MObI inserts the correct object successfully but loses fine appearance details. In the last two examples (bottom left), MObI inserts a car instead of the object depicted by the reference. (b) In the first three examples (top right), MObI correctly replaces objects from classes outside of its training set, yet quality degrades. In the last example (bottom right), the model replaces the motorcycle with a small vehicle, reverting to a familiar class. Note that all examples have been correctly inserted in the target bounding box with the correct orientation.}
\label{fig:suppl:ood}
\end{figure*}

\subsection{MObI: future directions}
\label{mobi:future}

A promising avenue for future research lies in explicitly enforcing consistency across different viewpoints or time steps. This could be achieved by extending the cross-modal attention mechanism described in \autoref{sec:method:multimodal generation} to span multiple time steps, as explored in works such as~\cite{gao2023magicdrive, li2023drivingdiffusion, drivescape, wen2023panacea, lu2025wovogen}. Such an approach would maintain focus on a specific object throughout a sequence, ensuring temporal and geometric coherence in dynamic or multi-view scenes.

Another potential direction involves adapting the model to a broader, open-world setting. This could be accomplished by training on a diverse set of 3D object detection datasets, as demonstrated by~\cite{minderer2022simple}. Doing so would improve the model's capacity to handle a wider range of object appearances, placements, and environmental conditions.

Additionally, rather than conditioning solely on a single bounding box, the method could be extended to support full-scene context conditioning. This would involve incorporating information from all objects present in the scene, similar to strategies used in~\cite{gao2023magicdrive}. Such holistic conditioning could improve placement accuracy and reduce unintended interference with background elements.

Lastly, the development of evaluation metrics that measure cross-modal consistency and realism holistically remains an open challenge. Tailored metrics could better reflect human perception of multimodal scene plausibility and support more rigorous benchmarking of generative models used in safety-critical applications.

Despite current limitations, the approach presented here establishes a foundation for realistic and controllable multimodal scene editing. Such a capability is particularly valuable in autonomous driving, where synthetic data can help explore edge cases and improve the robustness of perception systems.

\subsection{AnydoorMed: future directions}

For \textbf{AnydoorMed}, one immediate direction involves extending current realism metrics to include downstream task performance, particularly in object detection and classification. Specifically, counterfactual anomalies sampled from underrepresented regions of the distribution could be used to augment training data and thereby improve the robustness of medical anomaly detectors.

Another promising avenue is applying the proposed method to other medical imaging modalities beyond mammography. Modalities such as magnetic resonance imaging (MRI), computed tomography (CT) or ultrasound scans present unique challenges regarding anatomy, resolution, and appearance. Testing the method across these domains would enable a more comprehensive assessment of its generalisability and adaptability.

Further research could extend the method to 3D volumetric inpainting, where entire slices or volumes of anatomical structures must be synthesised. This would require spatially consistent editing across multiple planes, using a similar mechanism for time consistency as described in~\autoref{mobi:future}. 3D inpainting would be particularly useful for longitudinal studies, surgical planning, and data augmentation in volumetric diagnostic tasks.

Improved anatomical priors and region-specific guidance mechanisms could also be incorporated to reduce the risk of generating implausible insertions. For example, organ-specific segmentation or landmark localisation could constrain the inpainting process to clinically valid regions.

Finally, interpretability and clinical usefulness remain underexplored. Collaborations with radiologists could develop human-in-the-loop editing and teaching workflows where the reference-guided generation is adapted in real-time, potentially aiding education, differential diagnosis, or adversarial testing of medical AI systems.

These future directions offer a path towards reliable and clinically relevant synthetic data generation tools for the medical domain.

\section{Concluding remarks}

This work introduces \textbf{MObI} and \textbf{AnydoorMed}, two novel methods that explore the potential of reference-guided inpainting to generate realistic counterfactuals across distinct domains. Despite certain limitations, both approaches demonstrate strong performance and adaptability, contributing a unique perspective on how foundation models can be steered for task-specific editing in safety-critical settings.

\textbf{MObI} enables controllable, semantically consistent object insertions across camera and lidar modalities, which is particularly valuable for generating diverse training or evaluation scenarios in autonomous driving. Meanwhile, \textbf{AnydoorMed} offers a practical solution for synthesising plausible anomalies in medical images, providing a valuable tool for developing and evaluating anomaly detection systems.

By adapting latent diffusion models to different perceptual domains with minimal supervision, this project proposes a flexible and scalable framework for counterfactual generation. It opens promising directions for future research in synthetic data generation, robustness testing, and designing AI systems better equipped to handle rare or out-of-distribution events.

%%%%%%%%%%%%%%%%%% REFERENCES %%%%%%%%%%%%%%%%%%
%\clearpage % uncomment to start on a new page if wanted
\printbibliography[title={References},heading=bibintoc] % a single list of references for the whole thesis

\begin{uomappendix}

\section{Reproducibility statement}
\label{sec:suppl:method}
To promote transparency and facilitate further research, all code, trained models, and instructions necessary to reproduce the experiments will be released at the time of publication. These resources include scripts for data preprocessing, model training, evaluation, and configuration files to replicate the results presented in this paper.

% \begin{figure*}[htpb]
%     \centering
%     \footnotesize
%     \includegraphics[width=0.95\linewidth]{MObI/images/method/range_masks/object_mask.jpg}
%     \rotatebox{90}{\quad \textbf{Object pixels} \quad \quad ~~ \textbf{Edit mask} \qquad \textbf{Range image}}
%     \caption[Object-centric range view with corresponding projected bounding box and object mask.]{From top to bottom: (i) object-centric range depth image, (ii) range depth context with an edit mask, generated by projecting the object bounding box onto the range view, and (iii) object mask highlighting pixels corresponding to points within the 3D bounding box.}
%     \label{fig:suppl:range_masks}
% \end{figure*}

The repositories will be made publicly available at:
\begin{itemize}
    \item \textbf{MObI:} \url{https://github.com/alexbuburuzan/MObI}
    \item \textbf{AnydoorMed:} \url{https://github.com/alexbuburuzan/AnydoorMed}
\end{itemize}

Comprehensive documentation will be provided to ensure the methods can be readily understood and applied by the broader research community.

\newpage
\section{Ethics statement and risk assessment}

The methods proposed in this work, {MObI} and {AnydoorMed}, are designed to advance the state of controllable counterfactual generation through reference-guided inpainting across diverse modalities, with particular applications in autonomous driving and medical imaging. While these technologies offer significant potential for improving robustness and safety in machine learning systems, they also raise important ethical considerations.

\textbf{Synthetic data and misuse.} The generation of synthetic content, if misused, can lead to the fabrication of misleading or harmful visual material. In the context of autonomous driving, unintended consequences during model training or evaluation could be caused by incorrect or manipulated data. Similarly, in medical imaging, the synthetic creation of anomalies must be handled with care to ensure that practitioners are not misled and that patient trust is not compromised. The use of the proposed methods in clinical decision-making workflows is explicitly cautioned against unless rigorous validation and expert oversight are provided.

\textbf{Bias and fairness.} As with any model trained on real-world data, the proposed methods may be affected by biases present in the underlying datasets. For example, imbalances in the nuScenes~\cite{caesar2020nuscenes} and VinDr-Mammo~\cite{nguyen2023vindr} datasets could impact the diversity of generated outputs. It is acknowledged that synthetic data may unintentionally reinforce biases unless appropriate mitigation strategies, such as dataset balancing or bias-aware training, are applied.

\textbf{Privacy and data use.} All datasets used in this work are publicly available and appropriately licensed for academic research. No personally identifiable information is included in the datasets, and the authors collected no data. For medical imaging data, care was taken to ensure the use of anonymised images where applicable.

\textbf{Responsible deployment.} The proposed techniques should be used to augment, not replace, existing methods of validation and evaluation in safety-critical systems. Responsible deployment requires collaboration with domain experts and adherence to regulatory standards, particularly in the healthcare and transport sectors.

It is hoped that, by ensuring transparency in the methodology and openly sharing the findings, a broader conversation will be supported regarding the ethical use of generative models in real-world applications. Continued research is encouraged to improve interpretability, fairness, and accountability in generating synthetic data.

\newpage
\section{Planning and achievements}

The technical work presented in Chapter 2 was conducted as part of a research internship at FiveAI, where I developed \textbf{MObI}, a multimodal diffusion-based framework for reference-guided object insertion in autonomous driving scenes. During the first ten weeks of Semester 1, I dedicated time to submitting the paper to CVPR.

\subsection*{Author Contributions}

\begin{itemize}
    \item \textbf{Alexandru Buburuzan:} Trained \textbf{MObI}, implemented the full training pipeline, data processing routines, and realism metrics; led the research on synthetic data generation and latent diffusion models; was the primary contributor to paper writing.
    \item \textbf{Anuj Sharma:} Contributed to downstream evaluations of \textbf{MObI} with an object detector and provided feedback on the manuscript.
    \item \textbf{John Redford:} Provided advisory support and feedback on the paper.
    \item \textbf{Puneet K. Dokania:} Advised during the ideation phase and contributed feedback throughout the project.
    \item \textbf{Romain Mueller:} Co-led the paper writing, assisted with downstream evaluations, and contributed to ideation; Main supervisor for the paper.
\end{itemize}

In addition, the first twelve weeks were used to revise the theory behind diffusion models and set up the foundational components for the second project, \textbf{AnydoorMed}. In collaboration with Prof. Tim Cootes, mammography was selected as the target domain. During this time, I conducted an in-depth literature review, initiated the AnydoorMed repository, and laid the groundwork for domain-specific model adaptation.

The topic of this dissertation was self-proposed and constitutes the foundation of my future PhD work.

\begin{table}[h]
\centering
\begin{tabular}{|p{0.22\textwidth}|p{0.37\textwidth}|p{0.37\textwidth}|}
\hline
\textbf{Week(s)} & \textbf{Planned activity} & \textbf{Actual outcome} \\
\hline
1–10 (Sem 1) & Polishing MObI paper & Paper submitted to CVPR \\
\hline
1–12 (Sem 1) & Theory revision, ideation for second project & Revised diffusion model theory, selected mammography domain, conducted literature review, initiated AnydoorMed repository \\
\hline
1 (Sem 2) & Rebuttal of MObI & Rebuttal prepared answering all of the reviewers' concerns \\
\hline
2–5  (Sem 2) & Finalise AnydoorMed pipeline and VAE fine-tuning & Pipeline completed; fine-tuned VAE and trained AnydoorMed on mammography scans \\
\hline
6 (Sem 2) & Conduct ablations and implement realism metrics & Ran ablations and finalised the realism evaluation table \\
\hline
6 (Sem 2) & Resubmit MObI in case of rejection & paper submitted to CVPR Workshop on Data-Driven Autonomous Driving Simulations and later accepted with very good reviews.\\
\hline
7–8 (Sem 2) & Figure generation & Created all visualisations and supplementary figure panels \\
\hline
9–11 (Sem 2) & Writing and consolidation & Integrated results, analysis, and narrative into final document \\
\hline
\end{tabular}
\caption*{Comparison of planned vs. actual progress over the course of the project.}
\end{table}

\subsection*{Summary of Achievements}
\begin{itemize}
    \item Successfully adapted foundation diffusion models for image inpainting to two distinct domains: autonomous driving and medical imaging.
    \item Developed \textbf{MObI}, a multimodal diffusion-based framework for reference-guided object insertion in driving scenes.
    \item Designed and implemented \textbf{AnydoorMed}, extending reference-guided inpainting methods to the medical domain, specifically to mammograms.
    \item Implemented a comprehensive suite of realism metrics to quantitatively evaluate the medical replacement and reinsertion.
    \item Extended the realism evaluation framework to the medical domain, demonstrating the cross-domain applicability of the proposed approach.
\end{itemize}

\subsection*{Additional Milestones}
\begin{itemize}
    \item Conducted a detailed realism evaluation for the medical insertion setting, which was more difficult than the reinsertion and replacement setting.
    \item Acceptance of \textbf{MObI} in the Proceedings of the CVPR Workshop on Data-Driven Autonomous Driving Simulations, following a successful submission and peer-review process.
\end{itemize}

\end{uomappendix}

\end{document}